\documentclass{lcs2}
\usepackage{amsthm}
\usepackage{xspace}
\usepackage{tabularx}
\usepackage{multirow}

\newtheorem{proposition}{Proposition}
\newtheorem{theorem}{Theorem}

\newtheorem{corollary}{Corollary}
\theoremstyle{definition}
\newtheorem{definition}{Definition}[section]

\DeclareRobustCommand{\HalluPrism}{\texttt{HalluPrism}\xspace}
\definecolor{promptheader}{HTML}{664E8E}
\definecolor{promptbody}{HTML}{F6F5F9}
\newtcolorbox{promptbox}[2][]{
  enhanced,
  breakable,
  colback=promptbody,
  colframe=promptheader,
  coltitle=white,
  fonttitle=\bfseries,
  title=#2,
  arc=3mm,
  boxrule=1pt,
  left=6pt,
  right=6pt,
  top=6pt,
  bottom=6pt,
  #1
}

\papertitle{\HalluPrism{}: When Multimodal Uncertainty Should Diagnose, Not Decide}
\papershorttitle{\HalluPrism}
\papershortauthors{A. Prakash, S. Dasgupta, and T. Chakraborty}
\paperauthors{%
  Aman Prakash \textsuperscript{\tiny 1}\quad
  Sourish Dasgupta \textsuperscript{\tiny 2}\quad
  Tanmoy Chakraborty \textsuperscript{\tiny 1}%
}
\paperaffil{%
  \textsuperscript{1} LCS2 Lab, Indian Institute of Technology Delhi, New Delhi, India\\
  \textsuperscript{2} KDM Lab, Dhirubhai Ambani University, Gujarat, India%
}
\paperemails{{\ttfamily amanprakash.connect@gmail.com, sourish\_dasgupta@dau.ac.in, tanchak@iitd.ac.in}}
\paperkeywords{Multimodal Large Language Models, Multimodal Uncertainty, Hallucination Diagnosis, Behavioral Probing, Cross-Modal Grounding, Selective Prediction}

\paperabstract{%
Multimodal Large Language Models (MLLMs) can assign similar confidence to answers that fail for different reasons. We propose \HalluPrism{}, a behavioral diagnostic that re-runs an answer after visual degradation, blank-image replacement, and grounding or relation checks. These targeted probes yield a signature over visual-perturbation sensitivity $(V)$, image-removal confidence retention $(L)$, and grounding/relation-probe instability $(A)$. Across 58K+ examples from four benchmarks and four MLLMs, image-removal confidence retention is most prevalent, while grounding/relation-probe instability better separates failure families. Only 18 of 48 source-target checks are diagonally aligned, so the coordinates should be interpreted jointly rather than as independent causal sources. With the dataset fixed, the joint signature improves failure-family AUROC from $0.634$ to $0.769$ on HallusionBench and from $0.707$ to $0.817$ on VizWiz, with smaller gains on POPE and VSR. In pooled XGBoost analysis, AUROC rises from $0.78$ with scalar confidence to $0.95$ with $(V,L,A)$ and $0.97$ when confidence is added. The same signature does not automatically improve correctness ranking. The three tested direct scalarizations can harm it. These results separate failure diagnosis from abstention scoring: multimodal uncertainty should characterize failure structure before it is used to decide whether to abstain or correct.%
}

\appendixtocon
\appendixtocname{Appendix Contents}

\begin{document}
\makelabtitle

\section{Introduction}
\label{sec:introduction}

A visually-grounded answer can be unsafe even when a Multimodal Large Language Model (MLLM) assigns high confidence to it. Consider a street-scene image where a traffic cone is to the right of a parked car, and the user asks, ``\textit{Is the traffic cone to the left of the car}?'' The model may answer ``\textit{yes}'' with high confidence. This single confident answer can hide different risks. The cone may be small, blurred, or partly occluded, suggesting weak visual evidence. The model may still answer ``\textit{yes}'' after the image is replaced with a blank canvas, showing answer persistence without the image, which may be consistent with language-prior reliance on the question alone. It may also recognize both the cone and the car while reversing the left-right relation, suggesting cross-modal binding failure. In such cases, while scalar uncertainty can help decide whether the model should abstain, it cannot explain the failure source.

Existing hallucination benchmarks already show that multimodal errors are heterogeneous, covering object hallucination, visual illusion, attribute error, spatial error, and answerability failure \citep{li2023evaluating,guan2024hallusionbench,wang2023amber,liu2023mmbench}. Confidence, calibration, and selective prediction methods usually rank likely correctness or decide whether the model should abstain \citep{chow1970optimum,kadavath2022language,kuhn2023semantic,geng2023survey}, while recent MLLM uncertainty work moves closer to hallucination detection through evidential conflict, hidden-state reliability, and uncertain visual tokens \citep{huang2025evidentialconflict,mushtaq2025harmony,seo2025visualtokenuncertainty}. However, scalar abstention is not an internal model guarantee. It is rather a deployment guardrail whose validity depends on designer-chosen scores, thresholds, validation sets, and risk assumptions that may fail for new open-world instances. We, therefore, argue that diagnosis should be the core object of multimodal uncertainty --- before answering or correcting, the system should characterize \textit{whether behavior is consistent with visual sensitivity, image-removal confidence retention, or cross-modal binding instability}.

\begin{figure*}[t]
  \centering
  \includegraphics[width=\textwidth]{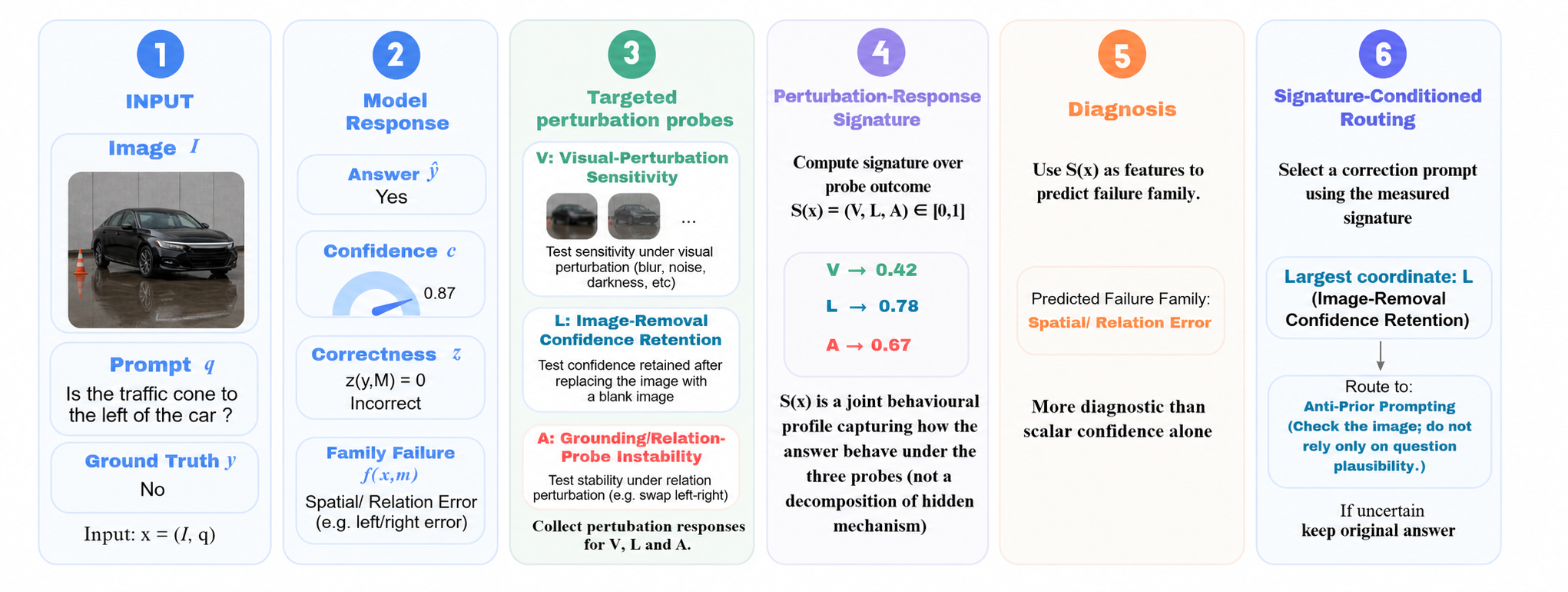}
    \caption{\textbf{An overview of the {\tt HalluPrism} perturbation-response protocol}. Here $L$ denotes image-removal confidence retention; language-prior reliance is only a possible interpretation of this response. The scores are behavioral probes, not a causal decomposition. Full prompt templates and probe parameters are given in Appendix~\ref{app:prompt_catalog}.}
    \label{fig:multimodal_probes}
\end{figure*}


We propose \HalluPrism{}, a perturbation-response diagnostic for MLLM uncertainty. Given an image, a prompt, and a generated answer, \HalluPrism{} computes scalar confidence $c$ and a perturbation-response signature $s(x)=(V,L,A)$ for visual-perturbation sensitivity $(V)$, image-removal confidence retention $(L)$, and grounding/relation-probe instability $(A)$. In the traffic-cone example, $V$ increases when blur, crop, darkness, or noise changes the answer or its confidence, $L$ increases when confidence is retained relative to the clean-image prediction after blank-image replacement, and $A$ increases when grounding or relation perturbations expose instability in the left-right binding. We additionally use a stricter image-removal control that requires the original answer itself to persist after image removal. The signature $s(x)$ is a perturbation-response profile, not a decomposition of hidden mechanisms. Hence, a high $L$ does not prove that the model never used the image, while high $V$ or $A$ records sensitivity rather than an isolated internal cause.

We use \HalluPrism{} as a \textit{joint} behavioral signature for diagnosis and correction-path selection. It first acts as a diagnostic feature for failure-family prediction. It then selects a correction path for risky predictions: high visual-perturbation sensitivity selects \textit{visual re-grounding} of the cone and car, high image-removal confidence retention selects \textit{anti-prior prompting} that asks the model to check the image rather than rely on question-only plausibility, and high grounding/relation-probe instability selects \textit{relation checking} of whether the cone is to the left or right of the car. We call this signature-conditioned choice of correction prompt \textit{intervention routing}.

Across $58{,}440$ examples from Gemma 4, LLaVA-Mistral/Vicuna, and Qwen3-VL on HallusionBench, POPE, VSR, and VizWiz-VQA, we find four patterns. First, raw prevalence does not have diagnostic value. Although $L$ dominates $99.7\%$ of samples, $A$ is seen to be the strongest tree-based feature for failure-family prediction. Second, the scores are not independent labels. If $V$, $L$, and $A$ were indeed source-specific, then visual, language, and alignment-ruined inputs would have mainly moved corresponding matching coordinates. However, across $4\times4\times3$ ruined-condition checks, only $18/48$ such events pass. Visual inputs pass in $11/16$ settings, while language and alignment inputs pass in $2/16$ and $5/16$. Thus, $(V,L,A)$ is a joint Perturbation-Response Signature, not independent causes. Third, this joint signature remains highly diagnostic. With the dataset fixed, $(V,L,A)$ improves AUROC over scalar confidence from $0.634$ to $0.769$ on HallusionBench and from $0.707$ to $0.817$ on VizWiz, with smaller gains on POPE and VSR. In the pooled XGBoost analysis, AUROC improves from $0.78$ with scalar confidence to $0.95$ with $(V,L,A)$, and to $0.97$ with confidence added.

Fourth, as a secondary application, the signature guides safer correction. Matched routing selects visual re-grounding, anti-prior prompting, or relation checking, instead of one generic prompt for all examples. On Qwen3-VL/VSR, generic prompting drops accuracy from $87.4\%$ to $76.2\%$ and breaks $15.5\%$ of correct answers. Matched routing maintains $87.6\%$ accuracy and lowers breakage to $2.7\%$.

\noindent\textbf{Contributions.} (i) We argue that perturbation-response diagnosis should be the core of multimodal uncertainty. (ii) We propose \HalluPrism{}, a perturbation-response signature over visual-perturbation sensitivity, image-removal confidence retention, and grounding/relation-probe instability. (iii) We show that controlled probes expose interventional entanglement rather than clean source separability. (iv) We show that this entangled signature improves failure-family diagnosis, while naive scalarization can hurt abstention ranking.\footnote{The source code is available at \url{https://github.com/gitgoap/HalluPrism}.}

\section{Related Work}
\label{sec:related_work}

\noindent \textbf{Hallucination Evaluation.} Vision-language hallucination work shows that multimodal errors are heterogeneous. Captioning and Visual Question Answering connect object hallucination and visual mis-reasoning to language priors and weak grounding \citep{rohrbach2018object,goyal2017making}. Benchmarks such as POPE, HallusionBench, AMBER, MME, and MMBench make this explicit across object presence, visual illusion, attributes, relations, perception, reasoning, and instruction following \citep{li2023evaluating,guan2024hallusionbench,wang2023amber,fu2023mme,liu2023mmbench}. \HalluPrism{} instead measures a behavioral Perturbation-Response Signature over visual-perturbation sensitivity, image-removal confidence retention, and grounding/relation-probe instability.

\paragraph{Uncertainty and Abstention.} Reject-option and selective-classification work treats uncertainty as a decision object. A model answers when risk is low and abstains otherwise \citep{chow1970optimum,geifman2017selective,geifman2019selectivenet}. Calibration asks whether confidence matches empirical correctness \citep{guo2017calibration}. Language-model uncertainty uses self-estimation, token probabilities, semantic equivalence, and sample inconsistency as reliability signals \citep{kadavath2022language,kuhn2023semantic,geng2023survey,manakul2023selfcheckgpt}. Multimodal uncertainty adds visual abstention, evidential conflict, hidden-state reliability, and visual-token uncertainty \citep{eisenschlos2024selectively,srinivasan2024selective,huang2025evidentialconflict,mushtaq2025harmony,seo2025visualtokenuncertainty}. Our critique is narrower. Abstention has statistical value under its assumptions. However, a thresholded score is warranted by validation transfer and need not diagnose risk in a new open-world instance.

\paragraph{Probing and Attribution.} Behavioral testing studies model behavior through controlled input changes. CheckList, contrast sets, and counterfactually augmented data expose local failures under small edits \citep{ribeiro2020checklist,gardner2020contrastsets,kaushik2020counterfactually}. Attribution methods instead use visual-region localization, attention flow, or image-token attention analysis \citep{selvaraju2017gradcam,jain2019attention,abnar2020quantifying,gong2024damro}. \HalluPrism{} follows the behavioral-testing view. It asks how the output changes under pressure on visual evidence, question-only priors, or cross-modal binding.

\paragraph{Mitigation and Routing.} Hallucination mitigation often applies a fixed correction rule during decoding or after generation. Visual Contrastive Decoding, OPERA, M3ID, Woodpecker, and Volcano reduce hallucination through visual contrast, decoding penalties, image-grounded amplification, or post-hoc revision \citep{leng2023mitigating,huang2023opera,favero2024multimodal,yin2023woodpecker,lee2023volcano}. Recent inference-time methods refine this direction through inter-modality calibration, attention calibration, and visual-aware sparsification \citep{li2025imccd,fazli2025caac,zhuang2025vasparse}. \HalluPrism{} is diagnosis-first. It profiles a risky prediction and then selects a matched prompt-based correction.

\section{Methodology}
\label{sec:halluprism}
We study multimodal uncertainty as a \textit{pre-decision} diagnostic object. The goal is not to build a better scalar abstention score. We instead test whether a prediction carries a behavioral Perturbation-Response Signature that can diagnose failure family and route correction without assuming clean causal separability. The design therefore separates three objects that are often conflated: correctness ranking, failure-source diagnosis, and correction-path choice. This separation is important because the same Perturbation-Response Signature can be useful for explaining a failure, can be harmful under the direct scalarizations tested here, and useful again when choosing a safer correction path. A diagnostic representation can therefore improve failure typing without necessarily improving an answer-vs.-abstain score.

All cited formal results appear in Appendix~\ref{app:formal_support}.

\paragraph{Prediction and Perturbation-Response Signatures.}
\label{subsec:prediction_source_signatures}
Consider an image $I$ where a traffic cone is to the right of a parked car, and a prompt $q$ asks whether the cone is to the left of the car. If an MLLM $M$ answers ``\textit{yes}'' to $x=(I,q)$ while the reference answer $y$ is ``\textit{no}'', the output is incorrect\footnote{We normalize answers before comparison, mapping surface variants such as ``yes'', ``Yes.'', or ``yes, it is'' to the same canonical task label.}. We write correctness as $z(x,M)=\mathbf{1}[\mathrm{norm}(\hat{y})=\mathrm{norm}(y)]$. Incorrect outputs receive a coarse failure-family label $f(x,M)\in\mathcal{F}$, such as \texttt{object-hallucination}, \texttt{spatial/relation-error}, \texttt{attribute-error}, or \texttt{answerability-failure}. A scalar confidence estimator maps the input-output instance to $c(x,M)\in[0,1]$. For a tokenized answer $\hat{y}=(t_1,\ldots,t_N)$, we use the length-normalized confidence probability:
\begin{equation}
\begin{split}
& c(\hat{y}\mid I,q)=\exp\left(\frac{1}{N}\sum_{i=1}^{N}\log p(t_i\mid t_{<i},I,q)\right).
\end{split}
\label{eq:scalar_confidence}
\end{equation}
For answerability-focused datasets, we use the dataset-specific normalized answer probability. This scalar score ranks likely correctness. However, it does not say whether the answer depends on weak visual evidence, survives without the image, or fails because the image-text relation is misbound. Hence, scalar confidence is a correctness-ranking surrogate, not a grounded reliability certificate. This non-identifiability is formalized in Proposition~\ref{prop:scalar_nonidentifiability}. We attach a $[0,1]$-bounded Perturbation-Response Signature to the same prediction:
\begin{equation}
\begin{split}
& s(x,M)=(V(x,M),L(x,M),A(x,M)).
\end{split}
\label{eq:source_signature}
\end{equation}
Here $V$ measures visual-perturbation sensitivity, $L$ measures image-removal confidence retention, and $A$ measures grounding/relation-probe instability. These scores are behavioral probe responses, not hidden causal mechanisms. They remain useful because failure families may occupy different regions of the joint $(V,L,A)$ space.

\paragraph{Perturbation Probes.}
\label{subsec:perturbation_probes}
We measure each score by asking how the same prediction behaves under controlled input pressure. The clean answer $\hat{y}$ is the baseline. A probe changes one part of the input and produces $\tilde{y}$. We compare answers with $\mathrm{flip}(\hat{y},\tilde{y})$, which equals $1$ when the answer changes, and confidence shifts with $\Delta c(\hat{y},\tilde{y})=|c(\hat{y})-c(\tilde{y})|$. Flips capture discrete instability. Confidence shifts capture softer changes. The \textbf{visual score} $V$ measures sensitivity to degraded visual evidence. We keep the prompt fixed and use Gaussian blur ($\sigma\!=\!15$), center crop ($50\%$ per dimension, resized back), reduce brightness to $30\%$ of original image, and Gaussian noise ($\mu\!=\!0,\sigma\!=\!25$). For the $k$-th visual perturbation answer $\tilde{y}^{v}_{k}$:
\begin{equation}
\begin{split}
& V=0.6\cdot\frac{1}{K}\sum_{k=1}^{K=4}\mathrm{flip}(\hat{y},\tilde{y}^{v}_{k}) \\
& \qquad +0.4\cdot\min\left(2\cdot\frac{1}{K}\sum_{k=1}^{K=4}\Delta c(\hat{y},\tilde{y}^{v}_{k}),1\right).
\end{split}
\label{eq:visual_score}
\end{equation}
The larger flip weight makes changed answers the primary signal of visual fragility. The \textbf{image-removal score} $L$ is operationalized as confidence retention after removing image evidence. We replace the image with a uniform blank image $I_{\mathrm{blank}}$ and keep the prompt fixed. Let $\tilde{y}^{\ell}$ be the blank-image answer. We define:

\begin{equation}
\begin{split}
& L \equiv L_{\mathrm{conf}}
=\min\left(\frac{c(\tilde{y}^{\ell})}{c(\hat{y})+\epsilon},1\right).
\end{split}
\label{eq:language_score}
\end{equation}

A high $L=L_{\mathrm{conf}}$ means that confidence is largely retained after image removal, even when the blank-image answer differs from the clean answer. It is a behavioral marker of confidence retention, not a causal claim that the model ignores the image. Since Eq.~\ref{eq:language_score} can saturate, we use the full $(V,L,A)$ vector for diagnosis. Appendix~\ref{app:probe_sensitivity} reports answer-persistence, match-only, and denominator-stable controls.

\begin{equation}
\begin{split}
& L_{\mathrm{persist}}
=\mathbf{1}[\mathrm{match}(\hat{y},\tilde{y}^{\ell})]
\cdot L_{\mathrm{conf}}.
\end{split}
\label{eq:language_persist}
\end{equation}

The stricter score is nonzero only when the original answer itself persists after image removal. Neither score alone is a causal claim that the model ignores the image. In the main results, $L$ denotes $L_{\mathrm{conf}}$ unless stated otherwise. Appendix~\ref{app:probe_sensitivity} reports the corresponding controls.

The \textbf{alignment score} $A$ measures sensitivity to image-text binding pressure. We use grounding re-checking for all datasets and relation swapping when an explicit relation is available. VSR supports relation swaps because its captions contain relations such as \textit{left}, \textit{right}, \textit{above}, and \textit{behind}. In the cone-car example, a left-of question is swapped into the corresponding right-of question. A relation-sensitive model should not preserve the same normalized answer under both prompts. Let $\tilde{y}^{a}_{\mathrm{ground}}$ be the grounding-prompt answer and $\tilde{y}^{a}_{\mathrm{rel}}$ the relation-swapped answer when available. We define:
\begin{equation}
\begin{split}
& P_g=0.5\cdot\mathrm{flip}(\hat{y},\tilde{y}^{a}_{\mathrm{ground}})\\
& \qquad +\min(\Delta c(\hat{y},\tilde{y}^{a}_{\mathrm{ground}}),0.5), \\
& P_r=\mathbf{1}[\mathrm{available}]\cdot
\mathbf{1}[\mathrm{norm}(\tilde{y}^{a}_{\mathrm{rel}})=\mathrm{norm}(\hat{y})], \\
& A=\begin{cases}
\min((1-\lambda)P_g+\lambda P_r,1), & \mathrm{if\ available},\\
\min(P_g,1), & \mathrm{otherwise},
\end{cases}\\
& \lambda=0.6.
\end{split}
\label{eq:alignment_score}
\end{equation}
The relation component receives larger weight because it gives a sharper binding stress. The expected answer should change under an explicit relation swap. The grounding component gives a weaker universal alignment probe. We therefore avoid comparing raw $A$ magnitudes across relation-bearing and non-relation-bearing datasets. Cross-dataset diagnosis uses dataset and model controls, while routing uses within-setting normalized dominance. Appendix~\ref{app:alignment_sensitivity} reports the no-VSR control.

\paragraph{Interpreting And Using Perturbation-Response Signatures.}
\label{subsec:diagnosis_abstention_routing}
Once the probes produce $(V,L,A)$, we separate three questions. Which coordinate is largest? Does a targeted intervention mainly move its intended coordinate? Does the full signature help diagnosis, abstention, or routing? A large coordinate is not a causal assignment. Raw dominance only counts the largest coordinate, $s^{*}_{\mathrm{raw}}(x,M)=\arg\max_{Q\in\{V,L,A\}} Q(x,M)$. Because one coordinate can saturate, we use within-setting normalized dominance for routing: $s^{*}_{\mathrm{norm}}(x,M)=\arg\max_{Q\in\{V,L,A\}} z_Q(x,M)$, where $z_Q(x,M)=(Q(x,M)-\mu_Q)/(\sigma_Q+\epsilon)$. This prevents a saturated coordinate, such as $L$, from dominating every routing decision. We next test whether the scores behave as separable labels. Let $\mathcal{T}_V$, $\mathcal{T}_L$, and $\mathcal{T}_A$ denote visual-ruined, language-ruined, and alignment-ruined conditions. These are stronger than the local probes used to compute the scores. For target $R$ and coordinate $Q\in\{V,L,A\}$:
\begin{equation}
\begin{split}
& \Delta_{R,Q}(x,M)=Q(\mathcal{T}_R(x),M)-Q(x,M).
\end{split}
\label{eq:intervention_response_main}
\end{equation}
A cleanly separable signature would satisfy magnitude-based diagonal dominance:
\begin{equation}
\begin{split}
& |\Delta_{R,R}(x,M)|>\max_{Q\neq R}|\Delta_{R,Q}(x,M)|.
\end{split}
\label{eq:diagonal_condition_main}
\end{equation}
We test this condition rather than assume it. Definition~\ref{def:diagonal_separability} states the condition. Proposition~\ref{prop:probe_entanglement} formalizes the key limitation. If a source-targeted intervention changes another coordinate at least as much as the intended coordinate, the intervention cannot be read as source-specific. We then evaluate three uses of Perturbation-Response Signatures. Failure-family diagnosis asks whether $s(x,M)$ predicts $f(x,M)$ better than scalar confidence. The role separation follows from Proposition~\ref{prop:diagnosis_not_abstention} and Theorem~\ref{thm:role_separation}: a signature can help explain the kind of error without improving correctness ranking. Abstention scalarization asks whether a source-aware score $\tilde{c}(x,M)=g(c,V,L,A)$ improves answer-versus-abstain ranking. Proposition~\ref{prop:pairwise_inversion} and Corollary~\ref{cor:saturation_ranking} state the ranking risk when source penalties are not aligned with correctness. Routing asks whether the signature helps choose a safer correction path. Proposition~\ref{prop:routing_decomposition} and Corollary~\ref{cor:break_rate_importance} motivate reporting fix rate and break rate with accuracy. For abstention, we test three scalarizations:
\begin{equation}
\begin{split}
& c_{\max}=c\cdot(1-\max(V,L,A)),\\
& c_w=c\cdot(1-0.3V-0.3L-0.4A),\\
& c_{\mathrm{geo}}=c\cdot((1-V)(1-L)(1-A))^{1/3}.
\end{split}
\label{eq:scalarizations}
\end{equation}
For intervention routing, the generic baseline applies the same reasoning prompt to every input. The matched policy maps $s^{*}_{\mathrm{norm}}(x,M)$ to visual re-grounding, anti-prior prompting, or relation checking. For example, an alignment-dominant cone-car prediction is routed to relation checking rather than a generic reasoning prompt. We report accuracy, fix rate, and break rate:
\begin{equation}
\begin{split}
& \mathrm{Fix}=\frac{\#\{\mathrm{incorrect}\rightarrow\mathrm{correct}\}}{\#\{\mathrm{incorrect}\}},\\
& \mathrm{Break}=\frac{\#\{\mathrm{correct}\rightarrow\mathrm{incorrect}\}}{\#\{\mathrm{correct}\}}.
\end{split}
\label{eq:fix_break_main}
\end{equation}
The protocol uses one forward pass, four visual perturbation passes, one blank-image pass, and 1-2 alignment passes, depending on whether relation perturbation is available. No model is retrained.

\begin{table*}[ht]
\centering
\setlength{\tabcolsep}{8pt} 
\small
\begin{tabular}{l l l l} 
\toprule
\textbf{Dataset} & \textbf{N Per} & \textbf{Primary Stress} & \textbf{Source Probe Role} \\ 
\midrule
HallusionBench & 951   & Visual illusion, attribute error   & Visual and alignment pressure \\
POPE           & 9,000 & Object-presence hallucination      & Blank-image and object-prior pressure \\
VSR            & 340   & Spatial and relation binding       & Relation-swap alignment pressure \\
VizWiz-VQA     & 4,319 & Visual insufficiency, answerability& Visual degradation and blank-image pressure \\
\midrule
\textbf{Models} & \multicolumn{3}{l}{Gemma 4, LLaVA-Mistral, LLaVA-Vicuna, and Qwen3-VL under deterministic decoding.} \\ 
\bottomrule
\end{tabular}
\caption{Evaluation suite and roles. Evaluation roles focus on mixed failure-family diagnosis for HallusionBench, object hallucination/abstention for POPE, relation error/routing breakdown for VSR, and answerability diagnosis for VizWiz-VQA. Each model-dataset output receives scalar confidence $c$ and Perturbation-Response Signature $(V, L, A)$. Full dataset roles, split details, probe availability, and failure-family mappings are given in Appendix~\ref{app:dataset_details}.}
\label{tab:evaluation_suite}
\end{table*}


\section{Experimental Setup}
\label{sec:experimental_setup}

We evaluate whether multimodal uncertainty should be studied as Perturbation-Response diagnosis and correction-path selection rather than only as answer-versus-abstain scoring. For each model output, we compute scalar confidence $c$ and Perturbation-Response Signature $s(x)=(V,L,A)$. We test failure diagnosis, abstention, and intervention routing under matched perturbation protocols.

\noindent \textbf{Datasets and Models.} We evaluate Gemma 4, Qwen3-VL, LLaVA-Mistral, and LLaVA-Vicuna on HallusionBench, POPE, Visual Spatial Reasoning (VSR), and VizWiz-VQA \citep{gemma4report,bai2025qwen3vl,liu2023llava15,guan2024hallusionbench,li2023evaluating,liu2023visual,gurari2018vizwiz}. The suite covers visually grounded hallucination, object-presence hallucination, spatial relation binding, and answerability. We use the image-paired HallusionBench subset, all POPE splits, the VSR zero-shot development set, and the VizWiz-VQA validation split. We use greedy decoding with temperature $0$, disabled sampling, and a $128$-token cap. Split, filtering, checkpoint, and license details appear in Appendices~\ref{app:dataset_details} and~\ref{app:artifact_license_data_statement}.

\noindent \textbf{Diagnostic and Decision Tasks.} Each prediction receives correctness and coarse failure-family labels. Correct predictions map to \textit{no failure}. Incorrect predictions follow benchmark structure. POPE maps to object-presence failure, VSR to spatial or relation error, VizWiz-VQA to answerability failure, and HallusionBench to visually grounded hallucination categories. This tests broad failure-family separation rather than an exhaustive taxonomy. We train logistic regression, XGBoost, and LightGBM with stratified $5$-fold cross-validation over $[c]$, $[V,L,A]$, and $[V,L,A,c]$ \citep{chen2016xgboost,ke2017lightgbm}. We report diagnosis AUROC and Macro-F1, source profiles, abstention AUROC and risk-coverage, and routing accuracy, fix rate, and break rate. These metrics keep failure typing, answer-versus-abstain ranking, and correction effects separate. Appendix~\ref{app:failure_family_mapping} gives the full mapping.

\noindent \textbf{Controls, Abstention, and Routing.} Since failure families and datasets are partly aligned, we add dataset-ID and within-dataset controls. We also test the two main score-design concerns. For $L$, Appendix~\ref{app:probe_sensitivity} compares the clipped-ratio score with match-only and denominator-stable variants. For $A$, Appendix~\ref{app:alignment_sensitivity} removes VSR, the only relation-swap-eligible dataset. Abstention compares scalar confidence with Eq.~\ref{eq:scalarizations}. Routing compares no intervention, one generic reasoning prompt, and matched intervention. Matched routing selects image enhancement, anti-prior instructions, or grounding/relation checking from normalized source dominance. Exact templates and enhancement parameters appear in Appendix~\ref{app:intervention_routing_policies}. We summarize benchmark roles in Table~\ref{tab:evaluation_suite}.

\noindent \textbf{Human Audit.} We run a blind audit over $200$ examples, with $50$ from each dataset. Three annotators assign visual evidence insufficiency, language-prior reliance, alignment error, mixed, unclear, or no failure. Annotators do not see $V$, $L$, $A$, scalar confidence, or predicted source. We adjudicate $36$ disagreement cases. Full-label Fleiss' $\kappa$ is $0.236$, and source versus non-source Fleiss' $\kappa$ is $0.48$. We use this audit only as a limited semantic plausibility check; it is not ground truth, training supervision, score calibration, or evidence of causal source recovery. Appendix~\ref{app:human_audit} gives the full protocol and tables.

\begin{table*}[t]
\centering
\scalebox{0.98}
{
\small
\begin{tabular}{lll}
\toprule
\textbf{Evidence Block} & \textbf{Setting} & \textbf{Main Value} \\
\midrule
Mean image-removal confidence retention & All model-dataset settings & $L$ ranges from $0.91$ to $1$ \\
Blank-image match & All model-dataset settings & LP match ranges from $37.6\%$ to $96.5\%$ \\
Highest visual-perturbation sensitivity & VizWiz settings & $V$ ranges from $0.19$ to $0.30$ \\
Highest alignment instability & Gemma 4/VSR, Qwen3-VL/VSR & $A=0.68$ and $A=0.61$ \\
Raw dominant source & $58{,}440$ pooled samples & $L/A/V=58{,}287/151/2$ \\
\midrule
Magnitude-based diagonal pass count & All intervention checks & $18/48$ pass \\
Visual-ruined target & Across $16$ settings & $11/16$ pass \\
Language-ruined target & Across $16$ settings & $2/16$ pass \\
Alignment-ruined target & Across $16$ settings & $5/16$ pass \\
Qwen3-VL/HallusionBench & Visual-ruined deltas & $\Delta V=-0.09$, $\Delta L=+0.01$, $\Delta A=-0.01$ \\
Qwen3-VL/HallusionBench & Language-ruined deltas & $\Delta V=-0.05$, $\Delta L=-0.09$, $\Delta A=+0.18$ \\
\bottomrule
\end{tabular}
}
\caption{\textbf{Source-profile and Entanglement Summary}. LP match is exact-answer agreement under blank-image replacement. Diagonal pass counts test whether the intended source has the largest absolute response under the corresponding ruined condition. Full profiles and intervention deltas are given in Appendix Tables~\ref{tab:app_source_profiles_full}-\ref{tab:app_identifiability_qwen_vl}.}
\label{tab:source_entanglement_main}
\end{table*}

\begin{table*}[t]
\centering
\setlength{\tabcolsep}{4pt}
\renewcommand{\arraystretch}{0.95}

\begin{minipage}[t]{0.43\textwidth}
\centering
\small
\begin{tabular}{lccc}
\toprule
\multicolumn{4}{c}{\textbf{Failure-Family Diagnosis}} \\
\midrule
\multirow{2}{*}{\textbf{Classifier}} & \multirow{2}{*}{\textbf{Scalar}} & \multirow{2}{*}{$\mathbf{V+L+A}$} & \multirow{2}{*}{$\mathbf{V+L+A+c}$} \\
 &  &  &  \\
\midrule
Logistic Regression & $0.75$ & $0.83$ & $\mathbf{0.87}$ \\
XGBoost & $0.78$ & $0.95$ & $\mathbf{0.97}$ \\
LightGBM & $0.78$ & $0.95$ & $\mathbf{0.97}$ \\
\bottomrule
\end{tabular}

\end{minipage}
\hfill
\begin{minipage}[t]{0.43\textwidth}
\centering
\small
\begin{tabular}{lccc}
\toprule
\multicolumn{4}{c}{\textbf{Tree-Based Feature Importance}} \\
\midrule
\textbf{Feature} & \textbf{XGB} & \textbf{LGBM} & \textbf{Role} \\
\midrule
$A$ & $\mathbf{0.50}$ & $\mathbf{0.36}$ & Alignment \\
$L$ & $0.20$ & $0.15$ & Language prior \\
$c$ & $0.15$ & $0.27$ & Confidence \\
$V$ & $0.15$ & $0.22$ & Visual fragility \\
\bottomrule
\end{tabular}
\end{minipage}

\vspace{3pt}

\small
\begin{tabular}{llccc}
\toprule
\multicolumn{5}{c}{\textbf{Dataset-ID Controls}} \\
\midrule
\textbf{Target} & \textbf{Baseline} & \textbf{Source Add-On} & \textbf{$\Delta$ AUROC} & \textbf{$\Delta$ Macro-F1} \\
\midrule
Raw family
& $D+M+c$
& $+\,V+L+A$
& $+0.005$
& $+0.078$ \\
Outcome family
& $D+M+c$
& $+\,V+L+A$
& $+0.006$
& $+0.040$ \\
\bottomrule
\end{tabular}

\caption{Failure-family diagnosis, feature importance, and dataset-ID controls. Classifier rows report pooled weighted one-vs-rest AUROC over $58{,}440$ examples. Feature rows report pooled tree-based feature importance. Dataset-ID controls use LightGBM and report gains after adding source features beyond dataset identity $D$, model identity $M$, and scalar confidence $c$. Full-precision diagnosis and control results are given in Appendix Tables~\ref{tab:app_failure_diagnosis_full}-\ref{tab:app_within_dataset_diagnosis}.}
\label{tab:diagnosis_importance_main}
\end{table*}

\begin{table}[t]
\centering
\small
\setlength{\tabcolsep}{4pt}
\begin{tabular}{lccc}
\toprule
\textbf{Control} & \textbf{Setting} & \textbf{AUROC} & \textbf{Macro-F1} \\
\midrule
$L_{\mathrm{conf}}$ & clipped ratio & $0.828$ & $0.333$ \\
$L_{\mathrm{match}}$ & match only & $0.823$ & $0.346$ \\
$L_{\mathrm{stable}}$ & floored denominator & $0.823$ & $0.349$ \\
\midrule
$c$ & excluding VSR & $0.761$ & $0.329$ \\
$(V,L,A)$ & excluding VSR & $0.839$ & $0.422$ \\
\bottomrule
\end{tabular}
\caption{Probe-design controls. The image-removal rows replace only the $L$ coordinate inside the pooled $(V,L,A)$ diagnosis feature set. The no-VSR control removes the only relation-swap-eligible dataset and compares scalar confidence with $(V,L,A)$ under the same non-VSR scope. Details are in Appendix Tables~\ref{tab:app_language_saturation_controls}-\ref{tab:app_alignment_availability_control}}
\label{tab:probe_design_controls_main}
\end{table}

\begin{table*}[t]
\centering
\setlength{\tabcolsep}{14pt}
\small
\begin{tabular}{llcc}
\toprule
\multicolumn{4}{l}{\textbf{Abstention: Scalar Confidence Vs. Best Source-Aware Scalarization}} \\
\midrule
\textbf{Model Family} & \textbf{Metric} & \textbf{Scalar Confidence} & \textbf{Best Source-Aware} \\
\midrule
Gemma 4 & AUROC & $0.81$ & $0.77\ (0.04\downarrow)$ \\
LLaVA-Mistral & AUROC & $0.84$ & $0.81\ (0.03\downarrow)$ \\
LLaVA-Vicuna & AUROC & $0.83$ & $0.78\ (0.05\downarrow)$ \\
Qwen3-VL & AUROC & $0.81$ & $0.73\ (0.08\downarrow)$ \\
\end{tabular}
\setlength{\tabcolsep}{10pt}
\small
\begin{tabular}{llccc}
\toprule
\multicolumn{5}{l}{\textbf{Routing: No Intervention Vs. Generic Prompting Vs. Matched Routing}} \\
\midrule
\textbf{Setting} & \textbf{Metric} & \textbf{No Intervention} & \textbf{Generic} & \textbf{Matched} \\
\midrule
Pooled & Accuracy & $76.6\%$ & $75.7\%$ & $\mathbf{75.9\%}$ \\
Highest accuracy & Count & $6/16$ & $3/16$ & $\mathbf{7/16}$ \\
Generic vs. matched & Accuracy wins & -- & $6/16$ & $\mathbf{10/16}$ \\
Generic vs. matched & Lower break rate & -- & $6/16$ & $\mathbf{10/16}$ \\
Qwen3-VL/VSR & Accuracy & $87.4\%$ & $76.2\%$ & $\mathbf{87.6\%}$ \\
Qwen3-VL/VSR & Break & -- & $15.5\%$ & $\mathbf{2.7\%}$ \\
\bottomrule
\end{tabular}
\caption{\textbf{Decision-facing behavior}. Abstention compares scalar confidence against the best source-aware scalarization for each model family. Routing compares no intervention, generic prompting, and matched routing. The \textit{Highest accuracy} row counts which policy obtains the best accuracy in each model-dataset setting. The pairwise rows compare generic prompting and matched routing across the same $16$ settings. Full-precision results appear in Appendix Tables~\ref{tab:app_abstention_scalarization_variants}-\ref{tab:app_intervention_win_counts}.}
\label{tab:decision_behavior_main}
\end{table*}

\section{Results}
\label{sec:results}

We evaluate $(V,L,A)$ for failure typing and matched correction, not as a scalar confidence replacement. The supporting appendix evidence follows the same order: source profiles in Tables~\ref{tab:app_source_profiles_full}-\ref{tab:app_dominant_source_counts}, entanglement in Tables~\ref{tab:app_identifiability_pass_counts}-\ref{tab:app_identifiability_qwen_vl}, diagnosis in Tables~\ref{tab:app_family_source_means}-\ref{tab:app_source_stress_slices}, human audit in Tables~\ref{tab:app_human_audit_template}-\ref{tab:app_human_audit_classifier}, abstention in Tables~\ref{tab:app_abstention_scalarization_variants}-\ref{tab:app_risk_coverage_pooled}, and routing in Tables~\ref{tab:app_intervention_full}-\ref{tab:app_intervention_win_counts}.

\paragraph{Prevalence is Not Diagnostic Value.}
\label{subsec:source_profiles}
We observe from Table~\ref{tab:source_entanglement_main} that $L$ is high across settings, with means from $0.91$ to $1$ and image-removal answer persistence from $37.6\%$ to $96.5\%$. VizWiz gives the highest visual scores ($0.19$ to $0.30$), while Gemma 4/VSR and Qwen3-VL/VSR give the highest alignment scores ($0.68$ and $0.61$). The raw largest-coordinate count is almost entirely $L$-dominant. Overall, $58{,}287/58{,}440$ samples are $L$-dominant, versus $151$ for $A$ and $2$ for $V$. This $99.7\%$ dominance refers to the confidence-retention score. Under the stricter answer-persistence score $L_{\mathrm{persist}}$, $L$ dominance falls to $51.6\%$, while $A$ and $V$ dominate $27.9\%$ and $20.4\%$ of examples, respectively. This is why we treat $L$ as a confidence-retention marker and use the full vector, not raw $L$ dominance, for diagnosis.
Appendix Tables~\ref{tab:app_source_profiles_full}-\ref{tab:app_dominant_source_counts} give the full source profiles and raw dominance counts. Appendix Table~\ref{tab:app_source_stress_slices} gives matched source-stress slices that support the same reading.

\paragraph{Interventions Expose Entanglement.}
\label{subsec:entanglement}
We test whether source-targeted ruined inputs isolate their intended coordinate under Definition~\ref{def:diagonal_separability}. We observe from Table~\ref{tab:source_entanglement_main} that only $18/48$ checks pass magnitude-based diagonal dominance: visual $11/16$, language $2/16$, and alignment $5/16$. Qwen3-VL/HallusionBench gives the main intuition. Visual ruin decreases $V$ by $0.09$ with small $L/A$ changes, while language ruin decreases $L$ by $0.09$ and increases $A$ by $0.18$. Appendix Table~\ref{tab:app_identifiability_pass_counts} gives the pass counts, and Appendix Tables~\ref{tab:app_identifiability_gemma_vl}-\ref{tab:app_identifiability_qwen_vl} give the model-specific signed deltas. These failures match Proposition~\ref{prop:probe_entanglement}. Hence, $(V,L,A)$ is a joint perturbation-response signature rather than three causal labels. This is, however, not a failure of the probe design. It is the diagnostic setting we want to expose: a visually fragile answer can also be image-removal sensitive or alignment-sensitive.

\paragraph{Signatures Diagnose Failure Families.}
Although the scores are entangled, their joint pattern remains diagnostic. Since failure-family labels and benchmark identity are partly aligned, we first test dataset and model controls. We observe from Table~\ref{tab:diagnosis_importance_main} that adding source features beyond these identifiers improves raw-family Macro-F1 from $0.71$ to $0.78$, and the stricter outcome-family setting from $0.31$ to $0.35$. Appendix Tables~\ref{tab:app_dataset_controlled_raw_family}-\ref{tab:app_dataset_controlled_outcome_family} give these controls. Fixed-dataset results are strongest on HallusionBench and VizWiz, weaker on POPE, and nearly neutral on VSR, as reported in Appendix Table~\ref{tab:app_within_dataset_diagnosis}. Specifically, $(V,L,A)$ improves AUROC over scalar confidence from $0.634$ to $0.769$ on HallusionBench, from $0.707$ to $0.817$ on VizWiz, from $0.823$ to $0.839$ on POPE, and from $0.722$ to $0.724$ on VSR. These results bound the claim: Perturbation-Response Signatures are most useful when multiple failure modes or answerability-sensitive distinctions remain, not as a uniformly transferable rule across all benchmarks. In an additional 10-split dataset/model-controlled analysis, adding $(V,L,A)$ beyond dataset ID, model ID, and scalar confidence improves benchmark-family Macro-F1 by $0.170$, with a positive gain in all $10/10$ splits. Leave-one-model-out Macro-F1 improves from $0.558$ to $0.716$.

We also test whether the result is carried by probe-design choices. We observe from Table~\ref{tab:probe_design_controls_main} that replacing the clipped-ratio $L$ with match-only or denominator-stable variants keeps diagnosis stable. Removing VSR, the only relation-swap-eligible dataset, also preserves the gain of $(V,L,A)$ over scalar confidence. Thus, the result is not explained by the clipped-ratio form of $L$ or by the VSR-only relation-swap branch. In the main pooled tree-based setting, XGBoost improves from $0.78$ AUROC with scalar confidence to $0.95$ with $(V,L,A)$, and to $0.97$ with confidence added. Appendix Table~\ref{tab:app_failure_diagnosis_full} gives the full diagnosis results. The gain is not merely prevalence-driven: the alignment coordinate contributes $0.50$ of pooled XGBoost importance, as detailed in Appendix Table~\ref{tab:app_feature_importance_full}. Appendix Table~\ref{tab:app_family_source_means} reports source means by failure family. A blind human audit gives bounded semantic support: clear-source labels are mostly alignment ($71/107$), and $(V,L,A)$ improves clear-source AUROC over scalar confidence from $0.43$ to $0.60$. Appendix Tables~\ref{tab:app_human_audit_template}-\ref{tab:app_human_audit_classifier} give the audit protocol, agreement, confusion matrix, and audit-side classifier.
\paragraph{Visual-Probe Robustness.}
We additionally test whether the diagnosis result depends on one visual perturbation or on the selected perturbation strength. Removing blur, crop, brightness, or noise individually keeps diagnosis AUROC between $0.822$ and $0.824$, with $89.8\%$--$97.8\%$ normalized source-assignment agreement. On the same $10{,}540$ outputs, increasing perturbation strength from weak to default to strong increases mean $V$ from $0.132$ to $0.196$ to $0.252$ and the visual answer-flip rate from $0.174$ to $0.263$ to $0.342$. Using $L_{\mathrm{persist}}$, diagnosis AUROC remains $0.785$, $0.786$, and $0.783$, compared with $0.632$ for scalar confidence. All paired-bootstrap intervals for these AUROC and Macro-F1 gains exclude zero. Full results are reported in Appendix~\ref{app:visual_robustness}.

\paragraph{Naive Scalarization Can Hurt.}
The diagnostic gain does not mean that the three direct source-aware scalarizations tested here should improve abstention. We observe from Table~\ref{tab:decision_behavior_main} that weighted, geometric, and max-penalty scalarizations underperform in AUROC. Qwen3-VL drops from $0.81$ scalar AUROC to $0.73$ for the best source-aware score. Appendix Table~\ref{tab:app_abstention_scalarization_variants} gives aggregate scalarization results, Appendix Table~\ref{tab:app_abstention_per_dataset} gives per-dataset comparisons, and Appendix Table~\ref{tab:app_risk_coverage_pooled} gives risk-coverage values. This does not certify scalar confidence as grounded. Rather, it shows that a diagnostic feature can be useful for explaining the kind of failure while still being misaligned with answer-versus-abstain ranking. In other words, source-aware uncertainty should not be naively scalarized for abstention. These experiments do not rule out learned, calibrated, or validation-trained scalarizers. As formalized in Appendix Proposition~\ref{prop:pairwise_inversion} and Corollary~\ref{cor:saturation_ranking}, source penalties and near-saturated coordinates can reverse useful confidence orderings.

\paragraph{Diagnosis-Triggered Routing.}
We treat routing as a secondary downstream application rather than as validation of the source coordinates. We test whether Perturbation-Response Signatures help choose safer correction paths. Generic prompting applies one reasoning prompt to every example, while matched routing selects visual re-grounding, anti-prior prompting, or relation checking from normalized source dominance. We observe from Table~\ref{tab:decision_behavior_main} that matched routing is not a universal accuracy booster. It beats generic prompting in $10/16$ settings, with pooled accuracy $75.9\%$ versus $75.7\%$ for generic prompting and $76.6\%$ without intervention. Its value is clearer in break-rate control. On Qwen3-VL/VSR, matched routing avoids the generic accuracy drop from $87.4\%$ to $76.2\%$ and lowers breakage from $15.5\%$ to $2.7\%$, while maintaining $87.6\%$ accuracy. It also lowers breakage on Gemma 4/HallusionBench and LLaVA-Mistral/VizWiz ($15.1\%\!\to\!4.9\%$ and $21.4\%\!\to\!17.4\%$). Appendix Table~\ref{tab:app_intervention_full} gives the full routing results, and Appendix Table~\ref{tab:app_intervention_win_counts} gives model-level win counts. This supports routing when generic correction risks damaging correct predictions. The practical use is therefore not universal correction, but avoiding avoidable damage when a generic correction prompt is too blunt. Since baseline accuracy exceeds $0.5$, Corollary~\ref{cor:break_rate_importance} gives larger weight to avoiding broken correct answers than to an equal-sized repair gain. As an additional control, we evaluate a probe-informed policy against generic CoT on $6{,}561$ model-output instances. Relative to generic CoT, the probe-informed policy improves pooled accuracy by $3.54$ percentage points (95\% CI $[2.82,\,4.25]$) and reduces break rate by $5.12$ points (95\% CI $[4.12,\,6.09]$). It improves accuracy in all four evaluated settings, but is not universally better than no intervention; VizWiz remains a boundary case. 


\section{Conclusion}
\label{sec:conclusion}

In this paper, we present \HalluPrism{} as a perturbation-response diagnostic for multimodal uncertainty in MLLMs. Across 58K+ examples, perturbation-response signatures improve failure-family diagnosis, while matched routing reduces harmful breakage in selected high-risk settings. The same evidence shows a boundary: better failure typing does not imply better abstention scoring. In our experiments, the three tested direct signature-based scalarizations do not improve correctness ranking and can perform worse than scalar confidence; learned or calibrated scalarizers remain an open direction. Perturbation-response diagnosis should therefore characterize whether behavior is consistent with visual sensitivity, image-removal confidence retention, or alignment instability, and help determine whether a generic correction is safe.


\clearpage
\section*{Limitations}
\label{sec:limitations}


\HalluPrism{} is a behavioral diagnostic layer, not a causal explanation of hidden mechanisms inside a Multimodal Large Language Model (MLLM). 
The scores are defined using a specific set of probes and perturbation strengths. Our ablation and severity analyses show that the main diagnostic trends remain stable across the settings tested here, while broader probe choices remain an avenue for future evaluation.
Because the evaluated benchmarks emphasize different failure families, we additionally report controlled and fixed-dataset analyses. The magnitude of the diagnostic gain varies across datasets, and cross-dataset generalization remains an important direction for future study.
 The evaluation covers four public benchmarks and four open-weight MLLM families, not video, multi-image reasoning, medical, or document-heavy expert settings, or latency-critical deployments. The implementation requires $7$ to $8$ forward passes per example. Routing uses lightweight prompts and gives conditional actionability, not universal mitigation. The human audit is intentionally small and checks semantic plausibility, not causal source recovery. 
Our abstention analysis focuses on three direct hand-designed scalarizations; learned or calibrated scalarization strategies are outside the scope of the present study.
The no-free-lunch argument does not reject selective prediction, calibration, risk control, or conformal prediction for their stated statistical targets. It states that these guarantees require validation-to-deployment assumptions and, by themselves, do not provide instance-grounded reliability.

\section*{Ethics Statement}
\HalluPrism{} is intended for model auditing and diagnostic analysis, not for standalone safety certification or deployment approval. Its scores are behavioral Perturbation-Response Signatures from controlled probes and should not be treated as causal explanations of hidden model mechanisms. The main ethical risk is over-trust, where a compact source label is mistaken for a complete reliability guarantee. We therefore keep diagnosis, scalarization, fix rate, and break rate separate. The experiments use public multimodal benchmarks and publicly accessible model families. We do not collect new user images or train a deployed model. Since VizWiz-VQA contains images and questions from blind or low-vision users, any release should avoid raw-image redistribution unless allowed by the original license. The human audit uses benchmark examples only, and annotators do not see model-derived source scores or routing decisions.

\section*{Acknowledgment}
Tanmoy Chakraborty acknowledges the support of the Rajiv Khemani Young Faculty Chair Professorship in Artificial Intelligence and the ICMR Extramural Research (Centre for Advanced Research) Project Grant (EM/DEV/CAR-2025/00725). We also thank Ayush Kumar, Kumari Shristee, and Aryan Raj for their valuable assistance with the data annotation process.

\small
\bibliographystyle{IEEEtranN}
\bibliography{refs}

@misc{gemma4report,
  author = {{Google DeepMind}},
  title  = {{Gemma 4 Model Card}},
  year   = {2026},
  month  = apr,
  url    = {https://ai.google.dev/gemma/docs/core/model_card_4},
  note   = {Accessed: 2026-05-20}
}

@inproceedings{rohrbach2018object,
    title = "Object Hallucination in Image Captioning",
    author = "Rohrbach, Anna  and
      Hendricks, Lisa Anne  and
      Burns, Kaylee  and
      Darrell, Trevor  and
      Saenko, Kate",
    editor = "Riloff, Ellen  and
      Chiang, David  and
      Hockenmaier, Julia  and
      Tsujii, Jun{'}ichi",
    booktitle = "Proceedings of the 2018 Conference on Empirical Methods in Natural Language Processing",
    month = oct # "-" # nov,
    year = "2018",
    address = "Brussels, Belgium",
    publisher = "Association for Computational Linguistics",
    url = "https://aclanthology.org/D18-1437/",
    doi = "10.18653/v1/D18-1437",
    pages = "4035--4045"
}

@INPROCEEDINGS{goyal2017making,
  author={Goyal, Yash and Khot, Tejas and Summers-Stay, Douglas and Batra, Dhruv and Parikh, Devi},
  booktitle={2017 IEEE Conference on Computer Vision and Pattern Recognition (CVPR)}, 
  title={Making the V in VQA Matter: Elevating the Role of Image Understanding in Visual Question Answering}, 
  year={2017},
  volume={},
  number={},
  pages={6325-6334},
  doi={10.1109/CVPR.2017.670}}

@inproceedings{li2023evaluating,
    title = "Evaluating Object Hallucination in Large Vision-Language Models",
    author = "Li, Yifan  and
      Du, Yifan  and
      Zhou, Kun  and
      Wang, Jinpeng  and
      Zhao, Xin  and
      Wen, Ji-Rong",
    editor = "Bouamor, Houda  and
      Pino, Juan  and
      Bali, Kalika",
    booktitle = "Proceedings of the 2023 Conference on Empirical Methods in Natural Language Processing",
    month = dec,
    year = "2023",
    address = "Singapore",
    publisher = "Association for Computational Linguistics",
    url = "https://aclanthology.org/2023.emnlp-main.20/",
    doi = "10.18653/v1/2023.emnlp-main.20",
    pages = "292--305"
}

@InProceedings{guan2024hallusionbench,
    author    = {Guan, Tianrui and Liu, Fuxiao and Wu, Xiyang and Xian, Ruiqi and Li, Zongxia and Liu, Xiaoyu and Wang, Xijun and Chen, Lichang and Huang, Furong and Yacoob, Yaser and Manocha, Dinesh and Zhou, Tianyi},
    title     = {HallusionBench: An Advanced Diagnostic Suite for Entangled Language Hallucination and Visual Illusion in Large Vision-Language Models},
    booktitle = {Proceedings of the IEEE/CVF Conference on Computer Vision and Pattern Recognition (CVPR)},
    month     = {June},
    year      = {2024},
    pages     = {14375-14385}
}

@misc{wang2023amber,
      title={AMBER: An LLM-free Multi-dimensional Benchmark for MLLMs Hallucination Evaluation}, 
      author={Junyang Wang and Yuhang Wang and Guohai Xu and Jing Zhang and Yukai Gu and Haitao Jia and Jiaqi Wang and Haiyang Xu and Ming Yan and Ji Zhang and Jitao Sang},
      year={2024},
      eprint={2311.07397},
      archivePrefix={arXiv},
      primaryClass={cs.CL},
      url={https://arxiv.org/abs/2311.07397}, 
}

@INPROCEEDINGS {gurari2018vizwiz,
author = { Gurari, Danna and Li, Qing and Stangl, Abigale J. and Guo, Anhong and Lin, Chi and Grauman, Kristen and Luo, Jiebo and Bigham, Jeffrey P. },
booktitle = { 2018 IEEE/CVF Conference on Computer Vision and Pattern Recognition (CVPR) },
title = {{ VizWiz Grand Challenge: Answering Visual Questions from Blind People }},
year = {2018},
volume = {},
ISSN = {},
pages = {3608-3617},
doi = {10.1109/CVPR.2018.00380},
url = {https://doi.ieeecomputersociety.org/10.1109/CVPR.2018.00380},
publisher = {IEEE Computer Society},
address = {Los Alamitos, CA, USA},
month =Jun}

@article{liu2023visual,
    title = "Visual Spatial Reasoning",
    author = "Liu, Fangyu  and
      Emerson, Guy  and
      Collier, Nigel",
    journal = "Transactions of the Association for Computational Linguistics",
    volume = "11",
    year = "2023",
    address = "Cambridge, MA",
    publisher = "MIT Press",
    url = "https://aclanthology.org/2023.tacl-1.37/",
    doi = "10.1162/tacl_a_00566",
    pages = "635--651"
}

@misc{kadavath2022language,
      title={Language Models (Mostly) Know What They Know}, 
      author={Saurav Kadavath and Tom Conerly and Amanda Askell and Tom Henighan and Dawn Drain and Ethan Perez and Nicholas Schiefer and Zac Hatfield-Dodds and Nova DasSarma and Eli Tran-Johnson and Scott Johnston and Sheer El-Showk and Andy Jones and Nelson Elhage and Tristan Hume and Anna Chen and Yuntao Bai and Sam Bowman and Stanislav Fort and Deep Ganguli and Danny Hernandez and Josh Jacobson and Jackson Kernion and Shauna Kravec and Liane Lovitt and Kamal Ndousse and Catherine Olsson and Sam Ringer and Dario Amodei and Tom Brown and Jack Clark and Nicholas Joseph and Ben Mann and Sam McCandlish and Chris Olah and Jared Kaplan},
      year={2022},
      eprint={2207.05221},
      archivePrefix={arXiv},
      primaryClass={cs.CL},
      url={https://arxiv.org/abs/2207.05221}, 
}

@misc{kuhn2023semantic,
      title={Semantic Uncertainty: Linguistic Invariances for Uncertainty Estimation in Natural Language Generation}, 
      author={Lorenz Kuhn and Yarin Gal and Sebastian Farquhar},
      year={2023},
      eprint={2302.09664},
      archivePrefix={arXiv},
      primaryClass={cs.CL},
      url={https://arxiv.org/abs/2302.09664}, 
}

@inproceedings{geng2023survey,
    title = "A Survey of Confidence Estimation and Calibration in Large Language Models",
    author = "Geng, Jiahui  and
      Cai, Fengyu  and
      Wang, Yuxia  and
      Koeppl, Heinz  and
      Nakov, Preslav  and
      Gurevych, Iryna",
    editor = "Duh, Kevin  and
      Gomez, Helena  and
      Bethard, Steven",
    booktitle = "Proceedings of the 2024 Conference of the North American Chapter of the Association for Computational Linguistics: Human Language Technologies (Volume 1: Long Papers)",
    month = jun,
    year = "2024",
    address = "Mexico City, Mexico",
    publisher = "Association for Computational Linguistics",
    url = "https://aclanthology.org/2024.naacl-long.366/",
    doi = "10.18653/v1/2024.naacl-long.366",
    pages = "6577--6595"
}

@inproceedings{manakul2023selfcheckgpt,
    title = "{S}elf{C}heck{GPT}: Zero-Resource Black-Box Hallucination Detection for Generative Large Language Models",
    author = "Manakul, Potsawee  and
      Liusie, Adian  and
      Gales, Mark",
    editor = "Bouamor, Houda  and
      Pino, Juan  and
      Bali, Kalika",
    booktitle = "Proceedings of the 2023 Conference on Empirical Methods in Natural Language Processing",
    month = dec,
    year = "2023",
    address = "Singapore",
    publisher = "Association for Computational Linguistics",
    url = "https://aclanthology.org/2023.emnlp-main.557/",
    doi = "10.18653/v1/2023.emnlp-main.557",
    pages = "9004--9017"
}

@inproceedings{eisenschlos2024selectively,
    title = "Selectively Answering Visual Questions",
    author = "Eisenschlos, Julian  and
      Maina, Hern{\'a}n  and
      Ivetta, Guido  and
      Benotti, Luciana",
    editor = "Ku, Lun-Wei  and
      Martins, Andre  and
      Srikumar, Vivek",
    booktitle = "Findings of the Association for Computational Linguistics: ACL 2024",
    month = aug,
    year = "2024",
    address = "Bangkok, Thailand",
    publisher = "Association for Computational Linguistics",
    url = "https://aclanthology.org/2024.findings-acl.250/",
    doi = "10.18653/v1/2024.findings-acl.250",
    pages = "4219--4229"
}

@inproceedings{srinivasan2024selective,
    title = "Selective ``Selective Prediction'': Reducing Unnecessary Abstention in Vision-Language Reasoning",
    author = "Srinivasan, Tejas  and
      Hessel, Jack  and
      Gupta, Tanmay  and
      Lin, Bill Yuchen  and
      Choi, Yejin  and
      Thomason, Jesse  and
      Chandu, Khyathi",
    editor = "Ku, Lun-Wei  and
      Martins, Andre  and
      Srikumar, Vivek",
    booktitle = "Findings of the Association for Computational Linguistics: ACL 2024",
    month = aug,
    year = "2024",
    address = "Bangkok, Thailand",
    publisher = "Association for Computational Linguistics",
    url = "https://aclanthology.org/2024.findings-acl.767/",
    doi = "10.18653/v1/2024.findings-acl.767",
    pages = "12935--12948"
}

@INPROCEEDINGS{selvaraju2017gradcam,
  author={Selvaraju, Ramprasaath R. and Cogswell, Michael and Das, Abhishek and Vedantam, Ramakrishna and Parikh, Devi and Batra, Dhruv},
  booktitle={2017 IEEE International Conference on Computer Vision (ICCV)}, 
  title={Grad-CAM: Visual Explanations from Deep Networks via Gradient-Based Localization}, 
  year={2017},
  volume={},
  number={},
  pages={618-626},
  doi={10.1109/ICCV.2017.74}}

@inproceedings{jain2019attention,
    title = "{A}ttention is not {E}xplanation",
    author = "Jain, Sarthak  and
      Wallace, Byron C.",
    editor = "Burstein, Jill  and
      Doran, Christy  and
      Solorio, Thamar",
    booktitle = "Proceedings of the 2019 Conference of the North {A}merican Chapter of the Association for Computational Linguistics: Human Language Technologies, Volume 1 (Long and Short Papers)",
    month = jun,
    year = "2019",
    address = "Minneapolis, Minnesota",
    publisher = "Association for Computational Linguistics",
    url = "https://aclanthology.org/N19-1357/",
    doi = "10.18653/v1/N19-1357",
    pages = "3543--3556"
}

@inproceedings{abnar2020quantifying,
    title = "Quantifying Attention Flow in Transformers",
    author = "Abnar, Samira  and
      Zuidema, Willem",
    editor = "Jurafsky, Dan  and
      Chai, Joyce  and
      Schluter, Natalie  and
      Tetreault, Joel",
    booktitle = "Proceedings of the 58th Annual Meeting of the Association for Computational Linguistics",
    month = jul,
    year = "2020",
    address = "Online",
    publisher = "Association for Computational Linguistics",
    url = "https://aclanthology.org/2020.acl-main.385/",
    doi = "10.18653/v1/2020.acl-main.385",
    pages = "4190--4197"
}

@inproceedings{gong2024damro,
    title = "{DAMRO}: Dive into the Attention Mechanism of {LVLM} to Reduce Object Hallucination",
    author = "Gong, Xuan  and
      Ming, Tianshi  and
      Wang, Xinpeng  and
      Wei, Zhihua",
    editor = "Al-Onaizan, Yaser  and
      Bansal, Mohit  and
      Chen, Yun-Nung",
    booktitle = "Proceedings of the 2024 Conference on Empirical Methods in Natural Language Processing",
    month = nov,
    year = "2024",
    address = "Miami, Florida, USA",
    publisher = "Association for Computational Linguistics",
    url = "https://aclanthology.org/2024.emnlp-main.439/",
    doi = "10.18653/v1/2024.emnlp-main.439",
    pages = "7696--7712"
}

@misc{leng2023mitigating,
      title={Mitigating Object Hallucinations in Large Vision-Language Models through Visual Contrastive Decoding}, 
      author={Sicong Leng and Hang Zhang and Guanzheng Chen and Xin Li and Shijian Lu and Chunyan Miao and Lidong Bing},
      year={2023},
      eprint={2311.16922},
      archivePrefix={arXiv},
      primaryClass={cs.CV},
      url={https://arxiv.org/abs/2311.16922}, 
}

@InProceedings{huang2023opera,
    author    = {Huang, Qidong and Dong, Xiaoyi and Zhang, Pan and Wang, Bin and He, Conghui and Wang, Jiaqi and Lin, Dahua and Zhang, Weiming and Yu, Nenghai},
    title     = {OPERA: Alleviating Hallucination in Multi-Modal Large Language Models via Over-Trust Penalty and Retrospection-Allocation},
    booktitle = {Proceedings of the IEEE/CVF Conference on Computer Vision and Pattern Recognition (CVPR)},
    month     = {June},
    year      = {2024},
    pages     = {13418-13427}
}

@INPROCEEDINGS{favero2024multimodal,
  author={Favero, Alessandro and Zancato, Luca and Trager, Matthew and Choudhary, Siddharth and Perera, Pramuditha and Achille, Alessandro and Swaminathan, Ashwin and Soatto, Stefano},
  booktitle={2024 IEEE/CVF Conference on Computer Vision and Pattern Recognition (CVPR)}, 
  title={Multi-Modal Hallucination Control by Visual Information Grounding}, 
  year={2024},
  volume={},
  number={},
  pages={14303-14312},
  doi={10.1109/CVPR52733.2024.01356}}

@article{yin2023woodpecker,
  title={Woodpecker: Hallucination correction for multimodal large language models},
  author={Yin, Shukang and Fu, Chaoyou and Zhao, Sirui and Xu, Tong and Wang, Hao and Sui, Dianbo and Shen, Yunhang and Li, Ke and Sun, Xing and Chen, Enhong},
  journal={Science China Information Sciences},
  volume={67},
  number={12},
  pages={220105},
  year={2024},
  publisher={Springer}
}

@inproceedings{lee2023volcano,
    title = "Volcano: Mitigating Multimodal Hallucination through Self-Feedback Guided Revision",
    author = "Lee, Seongyun  and
      Park, Sue Hyun  and
      Jo, Yongrae  and
      Seo, Minjoon",
    editor = "Duh, Kevin  and
      Gomez, Helena  and
      Bethard, Steven",
    booktitle = "Proceedings of the 2024 Conference of the North American Chapter of the Association for Computational Linguistics: Human Language Technologies (Volume 1: Long Papers)",
    month = jun,
    year = "2024",
    address = "Mexico City, Mexico",
    publisher = "Association for Computational Linguistics",
    url = "https://aclanthology.org/2024.naacl-long.23/",
    doi = "10.18653/v1/2024.naacl-long.23",
    pages = "391--404"
}

@article{huang2025evidentialconflict,
title = {Visual hallucination detection in large vision-language models via evidential conflict},
journal = {International Journal of Approximate Reasoning},
volume = {186},
pages = {109507},
year = {2025},
issn = {0888-613X},
doi = {https://doi.org/10.1016/j.ijar.2025.109507},
url = {https://www.sciencedirect.com/science/article/pii/S0888613X25001483},
author = {Tao Huang and Zhekun Liu and Rui Wang and Yang Zhang and Liping Jing}
}

@INPROCEEDINGS{mushtaq2025harmony,
  author={Mushtaq, Erum and Fabian, Zalan and Bakman, Yavuz Faruk and Ramakrishna, Anil and Soltanolkotabi, Mahdi and Avestimehr, Salman},
  booktitle={2025 IEEE/CVF Conference on Computer Vision and Pattern Recognition Workshops (CVPRW)}, 
  title={HARMONY: Hidden Activation Representations and Model Output-Aware Uncertainty Estimation for Vision-Language Models}, 
  year={2025},
  volume={},
  number={},
  pages={1654-1659},
  doi={10.1109/CVPRW67362.2025.00154}}

@misc{fazli2025caac,
      title={Mitigating Hallucination in Large Vision-Language Models via Adaptive Attention Calibration}, 
      author={Mehrdad Fazli and Bowen Wei and Ahmet Sari and Ziwei Zhu},
      year={2025},
      eprint={2505.21472},
      archivePrefix={arXiv},
      primaryClass={cs.CV},
      url={https://arxiv.org/abs/2505.21472}, 
}

@misc{seo2025visualtokenuncertainty,
      title={On Epistemic Uncertainty of Visual Tokens for Object Hallucinations in Large Vision-Language Models}, 
      author={Hoigi Seo and Dong Un Kang and Hyunjin Cho and Joohoon Lee and Se Young Chun},
      year={2025},
      eprint={2510.09008},
      archivePrefix={arXiv},
      primaryClass={cs.CV},
      url={https://arxiv.org/abs/2510.09008}, 
}

@INPROCEEDINGS{zhuang2025vasparse,
  author={Zhuang, Xianwei and Zhu, Zhihong and Xie, Yuxin and Liang, Liming and Zou, Yuexian},
  booktitle={2025 IEEE/CVF Conference on Computer Vision and Pattern Recognition (CVPR)}, 
  title={VASparse: Towards Efficient Visual Hallucination Mitigation via Visual-Aware Token Sparsification}, 
  year={2025},
  volume={},
  number={},
  pages={4189-4199},
  doi={10.1109/CVPR52734.2025.00396}}

@misc{bai2025qwen3vl,
      title={Qwen3-VL Technical Report}, 
      author={Shuai Bai and Yuxuan Cai and Ruizhe Chen and Keqin Chen and Xionghui Chen and Zesen Cheng and Lianghao Deng and Wei Ding and Chang Gao and Chunjiang Ge and Wenbin Ge and Zhifang Guo and Qidong Huang and Jie Huang and Fei Huang and Binyuan Hui and Shutong Jiang and Zhaohai Li and Mingsheng Li and Mei Li and Kaixin Li and Zicheng Lin and Junyang Lin and Xuejing Liu and Jiawei Liu and Chenglong Liu and Yang Liu and Dayiheng Liu and Shixuan Liu and Dunjie Lu and Ruilin Luo and Chenxu Lv and Rui Men and Lingchen Meng and Xuancheng Ren and Xingzhang Ren and Sibo Song and Yuchong Sun and Jun Tang and Jianhong Tu and Jianqiang Wan and Peng Wang and Pengfei Wang and Qiuyue Wang and Yuxuan Wang and Tianbao Xie and Yiheng Xu and Haiyang Xu and Jin Xu and Zhibo Yang and Mingkun Yang and Jianxin Yang and An Yang and Bowen Yu and Fei Zhang and Hang Zhang and Xi Zhang and Bo Zheng and Humen Zhong and Jingren Zhou and Fan Zhou and Jing Zhou and Yuanzhi Zhu and Ke Zhu},
      year={2025},
      eprint={2511.21631},
      archivePrefix={arXiv},
      primaryClass={cs.CV},
      url={https://arxiv.org/abs/2511.21631}, 
}

@INPROCEEDINGS{liu2023llava15,
  author={Liu, Haotian and Li, Chunyuan and Li, Yuheng and Lee, Yong Jae},
  booktitle={2024 IEEE/CVF Conference on Computer Vision and Pattern Recognition (CVPR)}, 
  title={Improved Baselines with Visual Instruction Tuning}, 
  year={2024},
  volume={},
  number={},
  pages={26286-26296},
  doi={10.1109/CVPR52733.2024.02484}}

@misc{fu2023mme,
      title={MME: A Comprehensive Evaluation Benchmark for Multimodal Large Language Models}, 
      author={Chaoyou Fu and Peixian Chen and Yunhang Shen and Yulei Qin and Mengdan Zhang and Xu Lin and Jinrui Yang and Xiawu Zheng and Ke Li and Xing Sun and Yunsheng Wu and Rongrong Ji and Caifeng Shan and Ran He},
      year={2025},
      eprint={2306.13394},
      archivePrefix={arXiv},
      primaryClass={cs.CV},
      url={https://arxiv.org/abs/2306.13394}, 
}

@misc{liu2023mmbench,
      title={MMBench: Is Your Multi-modal Model an All-around Player?}, 
      author={Yuan Liu and Haodong Duan and Yuanhan Zhang and Bo Li and Songyang Zhang and Wangbo Zhao and Yike Yuan and Jiaqi Wang and Conghui He and Ziwei Liu and Kai Chen and Dahua Lin},
      year={2024},
      eprint={2307.06281},
      archivePrefix={arXiv},
      primaryClass={cs.CV},
      url={https://arxiv.org/abs/2307.06281}, 
}

@inproceedings{chen2016xgboost,
author = {Chen, Tianqi and Guestrin, Carlos},
title = {XGBoost: A Scalable Tree Boosting System},
year = {2016},
isbn = {9781450342322},
publisher = {Association for Computing Machinery},
address = {New York, NY, USA},
url = {https://doi.org/10.1145/2939672.2939785},
doi = {10.1145/2939672.2939785},
booktitle = {Proceedings of the 22nd ACM SIGKDD International Conference on Knowledge Discovery and Data Mining},
pages = {785–794},
numpages = {10},
location = {San Francisco, California, USA},
series = {KDD '16}
}

@misc{li2025imccd,
      title={Cross-Modal Attention Calibration for LVLM Hallucination Mitigation}, 
      author={Jiaming Li and Jiacheng Zhang and Zequn Jie and Lin Ma and Guanbin Li},
      year={2026},
      eprint={2501.01926},
      archivePrefix={arXiv},
      primaryClass={cs.CV},
      url={https://arxiv.org/abs/2501.01926}, 
}

@inproceedings{ke2017lightgbm,
 author = {Ke, Guolin and Meng, Qi and Finley, Thomas and Wang, Taifeng and Chen, Wei and Ma, Weidong and Ye, Qiwei and Liu, Tie-Yan},
 booktitle = {Advances in Neural Information Processing Systems},
 editor = {I. Guyon and U. Von Luxburg and S. Bengio and H. Wallach and R. Fergus and S. Vishwanathan and R. Garnett},
 pages = {},
 publisher = {Curran Associates, Inc.},
 title = {LightGBM: A Highly Efficient Gradient Boosting Decision Tree},
 url = {https://proceedings.neurips.cc/paper_files/paper/2017/file/6449f44a102fde848669bdd9eb6b76fa-Paper.pdf},
 volume = {30},
 year = {2017}
}

@InProceedings{guo2017calibration,
  title = 	 {On Calibration of Modern Neural Networks},
  author =       {Chuan Guo and Geoff Pleiss and Yu Sun and Kilian Q. Weinberger},
  booktitle = 	 {Proceedings of the 34th International Conference on Machine Learning},
  pages = 	 {1321--1330},
  year = 	 {2017},
  editor = 	 {Precup, Doina and Teh, Yee Whye},
  volume = 	 {70},
  series = 	 {Proceedings of Machine Learning Research},
  month = 	 {06--11 Aug},
  publisher =    {PMLR},
  url = 	 {https://proceedings.mlr.press/v70/guo17a.html}
}

@ARTICLE{chow1970optimum,
  author={Chow, C.},
  journal={IEEE Transactions on Information Theory}, 
  title={On optimum recognition error and reject tradeoff}, 
  year={1970},
  volume={16},
  number={1},
  pages={41-46},
  doi={10.1109/TIT.1970.1054406}}

@misc{geifman2017selective,
      title={Selective Classification for Deep Neural Networks}, 
      author={Yonatan Geifman and Ran El-Yaniv},
      year={2017},
      eprint={1705.08500},
      archivePrefix={arXiv},
      primaryClass={cs.LG},
      url={https://arxiv.org/abs/1705.08500}, 
}

@InProceedings{geifman2019selectivenet,
  title = 	 {{S}elective{N}et: A Deep Neural Network with an Integrated Reject Option},
  author =       {Geifman, Yonatan and El-Yaniv, Ran},
  booktitle = 	 {Proceedings of the 36th International Conference on Machine Learning},
  pages = 	 {2151--2159},
  year = 	 {2019},
  editor = 	 {Chaudhuri, Kamalika and Salakhutdinov, Ruslan},
  volume = 	 {97},
  series = 	 {Proceedings of Machine Learning Research},
  month = 	 {09--15 Jun},
  publisher =    {PMLR},
  url = 	 {https://proceedings.mlr.press/v97/geifman19a.html}
}

@inproceedings{ribeiro2020checklist,
    title = "Beyond Accuracy: Behavioral Testing of {NLP} Models with {C}heck{L}ist",
    author = "Ribeiro, Marco Tulio  and
      Wu, Tongshuang  and
      Guestrin, Carlos  and
      Singh, Sameer",
    editor = "Jurafsky, Dan  and
      Chai, Joyce  and
      Schluter, Natalie  and
      Tetreault, Joel",
    booktitle = "Proceedings of the 58th Annual Meeting of the Association for Computational Linguistics",
    month = jul,
    year = "2020",
    address = "Online",
    publisher = "Association for Computational Linguistics",
    url = "https://aclanthology.org/2020.acl-main.442/",
    doi = "10.18653/v1/2020.acl-main.442",
    pages = "4902--4912"
}

@inproceedings{gardner2020contrastsets,
    title = "Evaluating Models' Local Decision Boundaries via Contrast Sets",
    author = "Gardner, Matt  and
      Artzi, Yoav  and
      Basmov, Victoria  and
      Berant, Jonathan  and
      Bogin, Ben  and
      Chen, Sihao  and
      Dasigi, Pradeep  and
      Dua, Dheeru  and
      Elazar, Yanai  and
      Gottumukkala, Ananth  and
      Gupta, Nitish  and
      Hajishirzi, Hannaneh  and
      Ilharco, Gabriel  and
      Khashabi, Daniel  and
      Lin, Kevin  and
      Liu, Jiangming  and
      Liu, Nelson F.  and
      Mulcaire, Phoebe  and
      Ning, Qiang  and
      Singh, Sameer  and
      Smith, Noah A.  and
      Subramanian, Sanjay  and
      Tsarfaty, Reut  and
      Wallace, Eric  and
      Zhang, Ally  and
      Zhou, Ben",
    editor = "Cohn, Trevor  and
      He, Yulan  and
      Liu, Yang",
    booktitle = "Findings of the Association for Computational Linguistics: EMNLP 2020",
    month = nov,
    year = "2020",
    address = "Online",
    publisher = "Association for Computational Linguistics",
    url = "https://aclanthology.org/2020.findings-emnlp.117/",
    doi = "10.18653/v1/2020.findings-emnlp.117",
    pages = "1307--1323"
}

@inproceedings{
kaushik2020counterfactually,
title={Learning The Difference That Makes A Difference With Counterfactually-Augmented Data},
author={Divyansh Kaushik and Eduard Hovy and Zachary Lipton},
booktitle={International Conference on Learning Representations},
year={2020},
url={https://openreview.net/forum?id=Sklgs0NFvr}
}

\beginappendix
\section{Theoretical Analysis}
\label{app:formal_support}

We provide formal results that support the Perturbation-Response interpretation used in the main paper. The results do not claim that $V$, $L$, and $A$ recover hidden causal mechanisms inside a Multimodal Large Language Model (MLLM). They instead separate four roles: scalar confidence ranks correctness, Perturbation-Response Signatures support failure-family diagnosis, scalarization can distort abstention rankings, and routing must be evaluated through both fix and break rates. The setup is related to classical reject-option classification and selective prediction \citep{chow1970optimum,geifman2017selective,geifman2019selectivenet}.

\subsection{Scalar Confidence Does Not Identify Failure Source}
\label{app:scalar_nonidentifiability}

Let $X$ denote the input-output instance, $C=c(X)\in[0,1]$ denote scalar confidence, $Z\in\{0,1\}$ denote correctness, and $F\in\mathcal{F}$ denote the failure-family label, with $F=\mathrm{None}$ when $Z=1$. A source-diagnosis rule based only on scalar confidence is any measurable map $h:[0,1]\rightarrow\mathcal{F}\cup\{\mathrm{None}\}$.

\begin{proposition}[Scalar Non-Identifiability of Failure Source]
\label{prop:scalar_nonidentifiability}
A zero-error diagnosis rule $h(C)$ exists if and only if $F$ is measurable with respect to $C$. Equivalently, for almost every confidence value $c$, the conditional distribution $P(F\mid C=c)$ must place all mass on one label. If there exists a set $B\subseteq[0,1]$ with $P(C\in B)>0$ such that
\begin{equation}
\begin{split}
& \mathbb{E}\left[1-\max_{f\in\mathcal{F}\cup\{\mathrm{None}\}}P(F=f\mid C)\mid C\in B\right] > 0,
\end{split}
\end{equation}
then every scalar-confidence-only diagnosis rule has positive error on $B$.
\end{proposition}

\begin{proof}
If $F$ is measurable with respect to $C$, then there exists a measurable function $h$ such that $F=h(C)$ almost surely. This gives zero diagnosis error. Conversely, if a zero-error rule $h(C)$ exists, then $F=h(C)$ almost surely, and hence $F$ is measurable with respect to $C$. For the positive-error statement, condition on each confidence value. The best scalar rule at confidence value $C=c$ selects the label with largest conditional probability, so its conditional error is $1-\max_f P(F=f\mid C=c)$. Averaging this quantity over $C\in B$ gives the displayed expression. If this value is positive, every scalar-confidence-only diagnosis rule has positive error on $B$.
\end{proof}

\noindent\textbf{Implications.} Proposition~\ref{prop:scalar_nonidentifiability} explains why scalar confidence is an incomplete object for multimodal hallucination analysis. Even if scalar confidence is useful for abstention, it cannot diagnose whether an error is an object hallucination, a spatial error, an attribute error, or an answerability failure unless those labels are already determined by confidence. This motivates the failure-family diagnosis experiment, where $(V,L,A)$ is evaluated against scalar confidence for predicting the type of failure.

\subsection{Interventional Separability and Entanglement}
\label{app:interventional_separability}

Let $S(X)=(V(X),L(X),A(X))$ denote the Perturbation-Response Signature. Let $\mathcal{T}_V$, $\mathcal{T}_L$, and $\mathcal{T}_A$ denote visual, image-removal, and grounding/relation interventions. For $R,S\in\{V,L,A\}$, define the intervention response:
\begin{equation}
\begin{split}
& \Delta_{R,S}(X)=S(\mathcal{T}_R(X))-S(X).
\end{split}
\label{eq:app_delta_response}
\end{equation}
The expected response matrix is $D\in\mathbb{R}^{3\times 3}$, with entries:
\begin{equation}
\begin{split}
& D_{R,S}=\mathbb{E}[\Delta_{R,S}(X)].
\end{split}
\label{eq:app_expected_response}
\end{equation}

\begin{definition}[Diagonal Separability Under A Probe Family]
\label{def:diagonal_separability}
The Perturbation-Response Signature is diagonally separable under the probe family $\{\mathcal{T}_V,\mathcal{T}_L,\mathcal{T}_A\}$ if, for every $R\in\{V,L,A\}$,
\begin{equation}
\begin{split}
& D_{R,R}>\max_{S\neq R}D_{R,S}.
\end{split}
\label{eq:app_diagonal_separability}
\end{equation}
\end{definition}

\begin{proposition}[Off-Diagonal Magnitude Response Implies Probe-Level Entanglement]
\label{prop:probe_entanglement}
If Definition~\ref{def:diagonal_separability} fails for some source-targeting ruined condition $R$, then the probe family does not support a one-to-one interpretation in which condition $R$ primarily affects only source $R$. In particular, if $|D_{R,S}|\geq |D_{R,R}|$ for some $S\neq R$, then the observed response to $\mathcal{T}_R$ is at least as large in magnitude on an off-target source as on its intended source.
\end{proposition}

\begin{proof}
The condition for a one-to-one source-selective probe interpretation is magnitude-based diagonal dominance under Definition~\ref{def:diagonal_separability}. If $|D_{R,S}|\geq |D_{R,R}|$ for some $S\neq R$, then the off-target response has magnitude no smaller than the intended response. Hence the intervention response cannot be assigned uniquely to the intended source under this probe family.
\end{proof}

\noindent\textbf{Implications.} Proposition~\ref{prop:probe_entanglement} gives the formal basis for the interventional entanglement matrix. The empirical question is not whether every ruined condition validates a clean source label. The question is whether each source-targeting condition produces the largest absolute response on its intended coordinate or whether comparable off-target responses appear. When a visual-ruined condition changes alignment or behavior consistent with language-prior reliance, that response is informative rather than a nuisance. It supports the interpretation of $(V,L,A)$ as a joint Perturbation-Response Signature.

\subsection{Diagnosis And Abstention Are Different Prediction Problems}
\label{app:diagnosis_abstention_distinction}

Let $T$ denote the failure type conditional on being wrong, so that $T=F$ when $Z=0$. Let $C$ denote scalar confidence and let $S=(V,L,A)$ denote the Perturbation-Response Signature. Abstention requires ranking examples by $P(Z=1\mid\cdot)$. Failure diagnosis requires predicting $T$ or $F$.

\begin{proposition}[Diagnostic Information Need Not Improve Correctness Ranking]
\label{prop:diagnosis_not_abstention}
There exist distributions over $(C,\mathbf{S},Z,T)$ such that the Perturbation-Response Signature $\mathbf{S}$ is strictly more informative than scalar confidence $C$ for failure-type prediction among incorrect examples, while $C$ is at least as informative as $\mathbf{S}$ for correctness ranking. Here $C\in[0,1]$ denotes scalar confidence, $\mathbf{S}=(V,L,A)$ denotes the Perturbation-Response Signature, $Z\in\{0,1\}$ denotes correctness, and $T$ denotes the failure type conditional on $Z=0$. One sufficient condition is:
\begin{equation}
\begin{split}
& T \not\!\perp\!\!\!\perp \mathbf{S} \mid Z=0,\\
& T \perp\!\!\!\perp C \mid Z=0,\\
& P(Z=1\mid C,\mathbf{S})=P(Z=1\mid C).
\end{split}
\label{eq:app_diagnosis_condition}
\end{equation}
Under these conditions, $\mathbf{S}$ can improve failure-type diagnosis over $C$ among incorrect examples, although adding $\mathbf{S}$ cannot improve the Bayes-optimal correctness score beyond $C$.
\end{proposition}

\begin{proof}
The first two conditions state that $S$ carries information about the failure type among incorrect examples, while $C$ does not. Therefore a classifier with access to $S$ can improve failure-type prediction over a classifier using $C$ alone whenever the conditional distributions differ across failure types. The third condition states that, after conditioning on $C$, the Perturbation-Response Signature provides no additional information about correctness. Hence the Bayes-optimal correctness ranking is a function of $C$ alone. Adding $S$ cannot improve that ranking.
\end{proof}

\noindent\textbf{Implications.} Proposition~\ref{prop:diagnosis_not_abstention} explains why the failure-diagnosis and abstention results can move in opposite directions. The Perturbation-Response Signature can substantially improve failure-family AUROC while still failing to improve correctness ranking. This supports the paper’s central distinction. Perturbation-Response Signatures are diagnostic representations. They are not guaranteed scalar confidence improvements.

\begin{theorem}[Role Separation Between Perturbation-Response Signatures And Scalar Confidence]
\label{thm:role_separation}
Suppose the conditions in Proposition~\ref{prop:diagnosis_not_abstention} hold. Then any Bayes-optimal abstention policy based on $(C,S)$ can be matched by a policy based on $C$ alone, while a Bayes-optimal failure-family classifier based on $(C,S)$ can strictly outperform every classifier based on $C$ alone.
\end{theorem}

\begin{proof}
For abstention, the Bayes-optimal score is $P(Z=1\mid C,S)$. By the third condition in Eq.~\ref{eq:app_diagnosis_condition}, this equals $P(Z=1\mid C)$. Therefore a policy that uses only $C$ can match the optimal policy that uses $(C,S)$. For failure-family prediction, the first two conditions imply that $S$ changes the conditional distribution of $T$ among incorrect examples, while $C$ does not. Therefore the Bayes-optimal classifier with access to $S$ can achieve lower expected classification error than any classifier restricted to $C$ when the conditional failure-type distributions differ across Perturbation-Response regions.
\end{proof}

\noindent\textbf{Implications.} Theorem~\ref{thm:role_separation} states that a Perturbation-Response Signature may be valuable for explaining what kind of failure occurred without improving answer-versus-abstain decisions. This justifies evaluating failure-family diagnosis and abstention as separate tasks, rather than treating abstention as the only test of uncertainty usefulness.

\subsection{Why Naive Source-Aware Scalarization Can Hurt Abstention}
\label{app:scalarization_trap}

Many source-aware confidence rules have the multiplicative form:
\begin{equation}
\begin{split}
& \tilde{C}=C\cdot \phi(S),
\end{split}
\label{eq:app_multiplicative_scalarization}
\end{equation}
where $\phi(S)\in[0,1]$ is a source penalty. The weighted, geometric, and max-penalty scores in Section~\ref{subsec:diagnosis_abstention_routing} are instances of this pattern.

\begin{proposition}[Pairwise Inversion Under Source Penalties]
\label{prop:pairwise_inversion}
Consider two examples $i$ and $j$ with scalar confidences $C_i>C_j$. A multiplicative source-aware score $\tilde{C}=C\phi(S)$ reverses their ranking if and only if:
\begin{equation}
\begin{split}
& \frac{\phi(S_j)}{\phi(S_i)}>\frac{C_i}{C_j}.
\end{split}
\label{eq:app_pairwise_inversion}
\end{equation}
Therefore, source penalties can hurt abstention whenever penalty variation is not aligned with correctness and is large enough to overcome scalar-confidence margins.
\end{proposition}

\begin{proof}
The source-aware ranking reverses the scalar-confidence ranking when $\tilde{C}_j>\tilde{C}_i$. This is equivalent to $C_j\phi(S_j)>C_i\phi(S_i)$. Dividing both sides by $C_j\phi(S_i)$ gives Eq.~\ref{eq:app_pairwise_inversion}.
\end{proof}

\noindent\textbf{Implications.} Proposition~\ref{prop:pairwise_inversion} explains how source-aware scalarization can reduce abstention quality. A penalty based on $V$, $L$, and $A$ can reorder examples that scalar confidence ranks correctly. This matters especially when the penalty is weakly related to correctness. The result supports the source-aware confidence-trap analysis in Section~\ref{sec:results}.

\begin{corollary}[Saturated Source Scores Can Compress Useful Rankings]
\label{cor:saturation_ranking}
Suppose a source penalty $\phi(S)$ is strongly affected by a nearly saturated coordinate such as $L\approx 1$. If $\phi(S)$ varies across examples due to small changes in saturated source scores and those variations are weakly correlated with correctness, then the multiplicative score $\tilde{C}=C\phi(S)$ can introduce pairwise inversions relative to $C$ without adding useful correctness information.
\end{corollary}

\begin{proof}
By Proposition~\ref{prop:pairwise_inversion}, ranking inversions occur whenever the penalty ratio exceeds the scalar-confidence ratio. Near saturation, small absolute variations in a penalty function can still produce relative penalty differences across examples, especially for nonlinear or max-based penalties. If those penalty differences are weakly correlated with correctness, they can reverse confidence orderings without improving correctness ranking. Therefore the scalarized score can compress or distort useful confidence rankings.
\end{proof}

\noindent\textbf{Implications.} Corollary~\ref{cor:saturation_ranking} connects the formal ranking argument to the empirical source profiles. The image-removal confidence-retention score is near saturation in many model-dataset settings. Multiplicative source penalties therefore risk flattening or distorting the confidence ranking. This explains why Perturbation-Response Signatures can improve failure diagnosis while source-aware scalar confidence can underperform scalar confidence in selective prediction.

\subsection{Routing Depends On Fix And Break Rates}
\label{app:routing_fix_break}

Let $a$ be the baseline accuracy before intervention. For a policy $\pi$, let $\mathrm{Fix}_{\pi}$ be the fraction of originally incorrect examples corrected by the intervention and let $\mathrm{Break}_{\pi}$ be the fraction of originally correct examples made incorrect.

\begin{proposition}[Accuracy Decomposition For Intervention Routing]
\label{prop:routing_decomposition}
The post-intervention accuracy of policy $\pi$ is:
\begin{equation}
\begin{split}
& \mathrm{Acc}_{\pi}=a+(1-a)\mathrm{Fix}_{\pi}-a\mathrm{Break}_{\pi}.
\end{split}
\label{eq:app_routing_accuracy}
\end{equation}
For two policies $\pi_1$ and $\pi_2$, policy $\pi_1$ has higher post-intervention accuracy than $\pi_2$ if and only if:
\begin{equation}
\begin{split}
& (1-a)(\mathrm{Fix}_{\pi_1}-\mathrm{Fix}_{\pi_2}) \\
& \qquad > a(\mathrm{Break}_{\pi_1}-\mathrm{Break}_{\pi_2}).
\end{split}
\label{eq:app_routing_condition}
\end{equation}
\end{proposition}

\begin{proof}
Among the originally correct fraction $a$, the intervention preserves a fraction $1-\mathrm{Break}_{\pi}$. Among the originally incorrect fraction $1-a$, it repairs a fraction $\mathrm{Fix}_{\pi}$. Thus:
\begin{equation}
\begin{split}
& \mathrm{Acc}_{\pi}=a(1-\mathrm{Break}_{\pi})+(1-a)\mathrm{Fix}_{\pi},\\
& \mathrm{Acc}_{\pi}=a+(1-a)\mathrm{Fix}_{\pi}-a\mathrm{Break}_{\pi}.
\end{split}
\end{equation}
Subtracting the expression for $\pi_2$ from the expression for $\pi_1$ gives Equation.~\ref{eq:app_routing_condition}.
\end{proof}

\noindent\textbf{Implications.} Proposition~\ref{prop:routing_decomposition} explains why intervention routing should be evaluated with both fix rate and break rate. A generic intervention can repair some failures and still reduce accuracy if it breaks many originally correct examples. This is especially important in high-accuracy settings, where the $a\mathrm{Break}_{\pi}$ term is large. The result supports the paper’s routing analysis, where matched routing is most meaningful when it avoids harmful generic overcorrection.

\begin{corollary}[Break-Rate Reduction Can Matter More Than Fix-Rate Gain]
\label{cor:break_rate_importance}
If baseline accuracy $a>0.5$, then a one-point reduction in break rate contributes more to post-intervention accuracy than a one-point increase in fix rate. Formally, the coefficient on $\mathrm{Break}_{\pi}$ in Eq.~\ref{eq:app_routing_accuracy} has magnitude $a$, while the coefficient on $\mathrm{Fix}_{\pi}$ is $1-a$.
\end{corollary}

\begin{proof}
The claim follows directly from Eq.~\ref{eq:app_routing_accuracy}. The contribution of fix rate is weighted by $1-a$, while the contribution of break rate is weighted by $a$. If $a>0.5$, then $a>1-a$, and break-rate changes have larger absolute effect on final accuracy than equal-sized fix-rate changes.
\end{proof}

\noindent\textbf{Implications.} Corollary~\ref{cor:break_rate_importance} clarifies why the break rate is emphasized alongside accuracy. In settings such as Qwen3-VL/VSR, the generic intervention harms many originally correct answers. Reducing breakage can be more safety-relevant than maximizing repair count alone. This supports the conditional routing claim: Perturbation-Response Signatures are useful when they prevent harmful interventions, not only when they increase pooled accuracy.

\subsection{Limits Of Scalar Abstention Under Open-World Deployment}
\label{app:limits_scalar_abstention}

The preceding results imply a broader limitation of score-based abstention. Abstention is not an intrinsic inference-time capability of a model. It is a designer-imposed guardrail over a score, threshold, prediction set, verifier, judge, or learned rejection rule. Such a guardrail can support selective answering under validation assumptions. However, it is not automatically an instance-grounded verdict for a new input. We separate two requirements. First, the scalar score must be sufficient for the grounded safe-answer label. Second, the validation-to-deployment score-risk relation must transfer to the inference distribution. Without these requirements, a score-based abstention rule can be executed, but it has no unconditional inference-time guarantee.

\begin{theorem}[No Grounded Abstention From A Non-Sufficient Scalar]
\label{thm:no_grounded_scalar_abstention}
Let $X=(I,q,\hat{y})$ be a prediction instance. Let $Y^\star\in\{0,1\}$ denote the grounded safe-answer label, where $Y^\star=1$ means that the model should answer under the target reliability criterion, and $Y^\star=0$ means that it should not answer. Let $R=r(X)\in\mathbb{R}$ be any scalar abstention score, and let an abstention policy be any measurable function $h(R)\in\{0,1\}$.

A perfectly reliable scalar abstention policy exists if and only if $Y^\star$ is measurable with respect to $R$. Equivalently, for almost every score value or score region used by the policy, all examples assigned the same scalar information must have the same grounded safe-answer label. If there exists a score region $B$ with $P(R\in B)>0$ such that
\begin{equation}
\begin{split}
& P(Y^\star=1\mid R\in B)>0,\\
& P(Y^\star=0\mid R\in B)>0,
\end{split}
\label{eq:grounded_scalar_region}
\end{equation}
then every abstention policy based only on $R$ has positive error on $B$. Therefore, unless the scalar score is sufficient for the grounded safe-answer label, the question ``should the model answer?'' cannot be reliably answered from that scalar alone.
\end{theorem}

\begin{proof}
If $Y^\star$ is measurable with respect to $R$, then there exists a measurable function $h$ such that $h(R)=Y^\star$ almost surely. This gives a perfectly reliable scalar abstention policy. Conversely, suppose a perfectly reliable scalar abstention policy $h(R)$ exists. Then $Y^\star=h(R)$ almost surely, which means $Y^\star$ is measurable with respect to $R$. For the positive-error statement, condition on $R\in B$. Any policy based only on $R$ must assign its decision using only the scalar information available in $B$. Since both $Y^\star=1$ and $Y^\star=0$ have positive conditional mass in $B$, the policy must be wrong on the mass assigned to at least one of these labels. Hence, the error on $B$ is positive.
\end{proof}

\paragraph{Implication.}
The theorem is score-agnostic. It applies to confidence, entropy, calibrated probability, semantic uncertainty, verifier scores, judge scores, conformal nonconformity scores, and learned rejection scores. A scalar score can support reliable abstention only when it already contains the grounded safe-answer state required by the deployment criterion.

\begin{proposition}[No Unconditional Inference-Time Abstention Guarantee]
\label{prop:no_unconditional_abstention}
Let $R=r(X)$ be any scalar score used for abstention, and let $\tau$ be a threshold chosen from validation data. The inference-time decision
\begin{equation}
\begin{split}
& \mathrm{answer}(X)=\mathbf{1}[R(X)\geq \tau]
\end{split}
\label{eq:threshold_answer_rule}
\end{equation}
is guaranteed to satisfy a target deployment risk only under a transfer condition connecting the validation and deployment distributions, for example
\begin{equation}
\begin{split}
& P_{\mathrm{deploy}}(Y^\star=0\mid R\geq \tau)\\
& \qquad = P_{\mathrm{val}}(Y^\star=0\mid R\geq \tau).
\end{split}
\label{eq:score_risk_transfer}
\end{equation}
Without such a condition, there exists a deployment distribution with the same score range and the same abstention rule for which the accepted region has arbitrarily high error. Hence, score-based abstention has no distribution-free inference-time guarantee over open-world inputs.
\end{proposition}

\begin{proof}
The threshold $\tau$ is fixed after validation. If no assumption links $P_{\mathrm{deploy}}$ to $P_{\mathrm{val}}$, we can construct a deployment distribution that places positive, or even all, mass on examples with $R\geq\tau$ and $Y^\star=0$. The abstention rule accepts these examples because it observes only $R$. Therefore, the accepted risk on deployment can be arbitrarily larger than the validation risk. Any guarantee must therefore come from an external transfer assumption, not from the score-threshold rule itself.
\end{proof}

\paragraph{Implication.}
The proposition locates the open-world weakness of score-threshold abstention. A validation risk curve, calibration curve, or prediction-set guarantee depends on the validation support and the assumed transfer relation. The accepted region is therefore warranted by the designer's validation assumptions, not by an unconditional inference-time guarantee for each new instance.

\begin{corollary}[No Source-Complete Scalar Reliability]
\label{cor:no_source_complete_scalar_reliability}
Let $G$ denote a decision-relevant grounding or failure-source state, such as whether the answer is visually supported, image-removal sensitive, or cross-modally unstable. Let $A^\star$ denote the safe correction action, with $A^\star=\textsc{none}$ when no correction is needed. If either $G$ or $A^\star$ is not measurable with respect to a scalar abstention score $R$, then no policy using only $R$ as an answer-versus-abstain interface is source-complete. Such a score may support selective answering under its assumptions, but it cannot identify the failure source or decide whether a generic correction is safe.
\end{corollary}

\begin{proof}
A policy using only $R$ can condition only on the information contained in $R$. If $G$ is not measurable with respect to $R$, then no function of $R$ can recover the grounding or failure-source state. If $A^\star$ is not measurable with respect to $R$, then no function of $R$ can determine the safe correction action. Therefore, a scalar answer-versus-abstain interface can be selective without being source-complete.
\end{proof}

\paragraph{Implication.}
The corollary connects scalar abstention to source incompleteness. A scalar score may support selective answering under its assumptions, although it cannot by itself identify the failure source, establish grounding, or decide whether a generic correction is safe.

\begin{table*}[t]
\centering
\small
\begin{tabular}{p{0.14\textwidth}p{0.12\textwidth}p{0.12\textwidth}p{0.22\textwidth}p{0.24\textwidth}}
\toprule
\textbf{Dataset} & \textbf{Split / Subset} & \textbf{N Per Model} & \textbf{Primary Diagnostic Stress} & \textbf{Role In \HalluPrism{}} \\
\midrule
HallusionBench & Image-paired subset & $951$ & Visual illusion, attribute error, and grounded hallucination & Mixed failure-family diagnosis and routing analysis under visually grounded hallucination cases. \\
POPE & Random, popular, and adversarial splits & $9{,}000$ & Object-presence hallucination and object priors & Object-hallucination diagnosis, question-only prior pressure, and abstention boundary check. \\
VSR & Zero-shot development examples & $340$ & Spatial and relation binding & Relation-swap alignment probing and routing break-rate analysis. \\
VizWiz-VQA & Validation split & $4{,}319$ & Visual insufficiency and answerability & Answerability diagnosis, visual-fragility analysis, and abstention boundary check. \\
\bottomrule
\end{tabular}
\caption{Dataset details for the \HalluPrism{} evaluation suite. Counts are per model before pooling across the four evaluated model families.}
\label{tab:dataset_details_summary}
\end{table*}

\begin{table*}[t]
\centering
\small
\begin{tabular}{p{0.14\textwidth}p{0.15\textwidth}p{0.25\textwidth}p{0.25\textwidth}}
\toprule
\textbf{Dataset} & \textbf{Output Space} & \textbf{Incorrect-Prediction Label} & \textbf{Diagnostic Interpretation} \\
\midrule
HallusionBench & yes/no & Attribute error or object hallucination & Mixed visually grounded hallucination and illusion failure. \\
POPE & yes/no & Object hallucination & Object-presence error under random, popular, or adversarial object-prior pressure. \\
VSR & true/false & Spatial or relation error & Cross-modal binding failure over explicit spatial relations. \\
VizWiz-VQA & Short answer or \textit{unanswerable} & Answerability failure & Visual insufficiency or failure to respect answerability constraints. \\
\bottomrule
\end{tabular}
\caption{Failure-family mapping used for diagnostic classification. Correct predictions are mapped to \textit{no failure}. Incorrect predictions are mapped using each benchmark's primary stress factor and available labels.}
\label{tab:failure_family_mapping}
\end{table*}

\begin{table*}[t]
\centering
\small
\begin{tabular}{lrrl}
\toprule
\textbf{Dataset} & \textbf{Per-Model N} & \textbf{Pooled N} & \textbf{Failure-Family Counts Per Model} \\
\midrule
HallusionBench & 951 & 3,804 & attribute = 360, object hallucination = 591 \\
POPE & 9,000 & 36,000 & object hallucination = 9,000 \\
VizWiz-VQA & 4,319 & 17,276 & answerability = 1,385, none = 2,934 \\
VSR & 340 & 1,360 & spatial = 340 \\
\bottomrule
\end{tabular}
\caption{Dataset composition and failure-family coverage. Counts are reported per model and after pooling across the four evaluated model families.}
\label{tab:dataset_composition_failure_coverage}
\end{table*}

\begin{table*}[t]
\centering
\small
\begin{tabular}{lllll}
\toprule
\textbf{Dataset} & \textbf{Visual Degradation} & \textbf{Blank Image} & \textbf{Grounding Re-Check} & \textbf{Relation Swap} \\
\midrule
HallusionBench & Yes & Yes & Yes & No \\
POPE & Yes & Yes & Yes & No \\
VSR & Yes & Yes & Yes & Yes \\
VizWiz-VQA & Yes & Yes & Yes & No \\
\bottomrule
\end{tabular}
\caption{Probe availability across datasets. Relation-swapped probing is used only when the input contains an explicit perturbable relation.}
\label{tab:dataset_probe_availability}
\end{table*}


\section{Dataset Details}
\label{app:dataset_details}

We use four public multimodal benchmarks that stress different sources of prediction risk. The goal is not to build a single pooled hallucination leaderboard. We instead need a dataset suite in which visual fragility, image-removal confidence retention, and grounding/relation-probe instability can be observed under different task conditions. HallusionBench contributes visual illusion and grounded hallucination cases, POPE contributes object-presence hallucination probes, Visual Spatial Reasoning (VSR) contributes explicit relation-binding examples, and VizWiz-VQA contributes real-user visual insufficiency and answerability cases. This mixture supports the Perturbation-Response view because the same scalar confidence value can occur in examples with different underlying diagnostic stresses.

\subsection{Dataset Roles}
\label{app:dataset_roles}

Table~\ref{tab:dataset_details_summary} summarizes the datasets, splits, sample counts, and their roles in the paper. The counts are reported per model before pooling across the four evaluated model families. HallusionBench contains $1{,}129$ datapoints, out of which $178$ are text-only. Since \HalluPrism{} requires visual perturbation, we evaluate the $951$ image-paired examples. POPE contributes $9{,}000$ object-presence examples across random, popular, and adversarial splits. VSR contributes $340$ zero-shot development examples with explicit spatial relations. VizWiz-VQA contributes $4{,}319$ validation examples and is used for visual insufficiency and answerability analysis.

\subsection{HallusionBench}
\label{app:dataset_hallusionbench}

HallusionBench evaluates visually grounded hallucination and illusion cases. We use the image-paired subset because the Perturbation-Response protocol requires image perturbations. Text-only examples are excluded since visual degradation, blank-image replacement, and image-grounded alignment probes would not be meaningful without an associated image. The evaluated subset contains yes/no questions, which allows answer normalization into canonical \texttt{yes} and \texttt{no} labels. We map the benchmark categories into attribute-level and object-level hallucination labels for failure-family diagnosis. HallusionBench is useful for the main story because it contains mixed visual and semantic pressures: a model may answer from a plausible prior, from partial visual evidence, or from unstable image-text binding.

\subsection{POPE}
\label{app:dataset_pope}

POPE provides object-presence hallucination probes. We use all $9{,}000$ examples per model across the random, popular, and adversarial splits. The task is binary, and outputs are normalized to \texttt{yes} and \texttt{no}. POPE is primarily used to measure object-hallucination behavior under strong object-prior pressure. It is also useful for the abstention boundary check because object-presence examples often produce high scalar confidence even when image-removal confidence retention is high. Since POPE does not contain explicit spatial relations, the relation-swap alignment probe is not used. Alignment is measured through the grounding re-check probe.

\subsection{Visual Spatial Reasoning}
\label{app:dataset_vsr}

VSR evaluates whether an image satisfies a spatial caption. We use $340$ zero-shot development examples. The base prompt asks the model whether the caption is true or false, and outputs are normalized to \texttt{true} and \texttt{false}. VSR is the clearest dataset for relation-sensitive alignment probing because captions contain explicit spatial relations such as \textit{left}, \textit{right}, \textit{above}, \textit{below}, \textit{in front of}, and \textit{behind}. For VSR, we use both alignment probes: (i) grounding re-checking and (ii) relation-swapped caption verification. The relation-swapped probe is central to testing whether the model preserves the same answer when the expected relation should change. VSR therefore anchors the cross-modal binding part of the Perturbation-Response Signature.

\subsection{VizWiz-VQA}
\label{app:dataset_vizwiz}

VizWiz-VQA contains visual questions from blind or low-vision users. We use the validation split with $4{,}319$ examples. The dataset is important for two reasons. First, many examples contain weak, blurry, poorly framed, or incomplete visual evidence, which directly stresses visual fragility. Second, some questions are not answerable from the image, making the dataset useful for answerability diagnosis. We instruct the model to return \texttt{unanswerable} when the image is too unclear or the question cannot be answered from visual evidence. We normalize answer variants before exact comparison. Since VizWiz questions are open-ended and do not usually contain controlled spatial relations, the relation-swap probe is not used. Alignment is measured through grounding re-checking.

\subsection{Failure-Family Mapping}
\label{app:failure_family_mapping}

Each model output receives a correctness label and a failure-family label. Correct predictions are mapped to \textit{no failure}. Incorrect predictions are mapped according to the benchmark stress factor and the available metadata. This mapping is intentionally coarse (see Table \ref{tab:failure_family_mapping}). We use it to test whether Perturbation-Response Signatures diagnose broad failure families, not to claim exhaustive error-taxonomy coverage.

\subsection{Probe Availability By Dataset}
\label{app:probe_availability_by_dataset}

We use visual degradation, blank-image replacement, and grounding re-checking for all image-paired datasets. We use relation-swapped probing only when the input contains an explicit perturbable relation, which applies to VSR (see Table \ref{tab:dataset_probe_availability}).

\subsection{Pooled Evaluation Units}
\label{app:pooled_evaluation_units}

The basic evaluation unit is a model-output instance $(x,M)$, where $x$ is a dataset example and $M$ is an evaluated model family. We use this pooled unit for source-profile summaries, dominant-source counts, failure-family diagnosis, abstention scalarization, and routing analysis. The per-dataset and pooled counts are reported in Table~\ref{tab:dataset_composition_failure_coverage}. We also report dataset-identity controls and within-dataset diagnosis because benchmark identity and failure family are partially aligned by design.

\section{Complete Prompt Catalog}
\label{app:prompt_catalog}

We provide the exact templates used for clean prediction, Perturbation-Response probing, probe-level entanglement tests, intervention routing, and auxiliary confidence reporting. These templates define the reproducible prompting protocol used in the experiments.

\subsection{Base Dataset Prompts}
\label{app:base_dataset_prompts}

Base prompts define the clean prediction $\hat{y}$ for each dataset.

\begin{promptbox}{Yes/No Visual Reasoning: POPE and HallusionBench}
\textbf{Required Input:}
\begin{itemize}[itemsep=2pt,parsep=0pt]
    \item \textbf{Question}: \{question\}
\end{itemize}

\textbf{Prompt Template:}

\{question\}

Answer with only `Yes' or `No'.
\end{promptbox}

\begin{promptbox}{Spatial Caption Verification: VSR}
\textbf{Required Input:}
\begin{itemize}[itemsep=2pt,parsep=0pt]
    \item \textbf{Caption}: \{caption\}
\end{itemize}

\textbf{Prompt Template:}

Look at the image carefully. Is the following statement true or false?

Statement: ``\{caption\}''

Answer with only `True' or `False'.
\end{promptbox}

\begin{promptbox}{Open-Ended VQA With Answerability: VizWiz}
\textbf{Required Input:}
\begin{itemize}[itemsep=2pt,parsep=0pt]
    \item \textbf{Question}: \{question\}
\end{itemize}

\textbf{Prompt Template:}

\{question\}

If the image is too unclear to answer, or the question cannot be answered from the image, say `unanswerable'. Otherwise, give a short answer (one or two words).
\end{promptbox}

\subsection{Diagnostic Perturbation-Response Probes}
\label{app:diagnostic_source_probes}

These probes produce the perturbed answers used to compute $V$, $L$, and $A$.

\subsubsection{Visual Fragility Probe}
\label{app:visual_probe_prompt}

The visual probe keeps the base prompt fixed and applies image degradations.

\begin{promptbox}{Visual Fragility Probe}
\textbf{Required Input:}
\begin{itemize}[itemsep=2pt,parsep=0pt]
    \item \textbf{Image}: \{image\}
    \item \textbf{Base Prompt}: \{base\_prompt\}
\end{itemize}

\textbf{Image Perturbations:}
\begin{itemize}[itemsep=2pt,parsep=0pt]
    \item Gaussian blur with $\sigma=15$.
    \item Center crop covering $50\%$ of each image dimension, resized back to the original size.
    \item Brightness reduction to $30\%$ of the original image.
    \item Gaussian noise with $\mu=0$ and $\sigma=25$.
\end{itemize}

\textbf{Prompt Template:}

\{base\_prompt\}
\end{promptbox}

\subsubsection{Image-Removal Confidence-Retention Probe}
\label{app:language_prior_probe}

The image-removal probe replaces the original image with a blank canvas while keeping the base prompt fixed.

\begin{promptbox}{Image-Removal Confidence-Retention Probe}
\textbf{Required Input:}
\begin{itemize}[itemsep=2pt,parsep=0pt]
    \item \textbf{Blank Image}: \{blank\_image\}
    \item \textbf{Base Prompt}: \{base\_prompt\}
\end{itemize}

\textbf{Image Replacement:}

The original image is replaced with a uniform blank canvas of the same input size.

\textbf{Prompt Template:}

\{base\_prompt\}
\end{promptbox}

\subsubsection{Cross-Modal Alignment Probe}
\label{app:alignment_probe_prompt}

The alignment probe uses grounding re-checking for all datasets and relation-swapped prompting only for VSR.

\begin{promptbox}{Alignment Probe: Grounding Re-Check}
\textbf{Required Input:}
\begin{itemize}[itemsep=2pt,parsep=0pt]
    \item \textbf{Original Prompt}: \{original\_prompt\}
\end{itemize}

\textbf{Prompt Template:}

Look at the image carefully. First, describe what you see in the image in one sentence. Then answer the following question.

\{original\_prompt\}
\end{promptbox}

\begin{promptbox}{Alignment Probe: Relation-Swapped Prompt for VSR}
\textbf{Required Input:}
\begin{itemize}[itemsep=2pt,parsep=0pt]
    \item \textbf{Caption}: \{caption\}
    \item \textbf{Relation-Swapped Caption}: \{relation\_swapped\_caption\}
\end{itemize}

\textbf{Prompt Template:}

Look at the image carefully. Is the following statement true or false?

Statement: ``\{relation\_swapped\_caption\}''

Answer with only `True' or `False'.
\end{promptbox}

\paragraph{Relation-swap construction.}
For VSR, we construct relation-swapped captions using dataset relations such as \textit{left}, \textit{right}, \textit{above}, \textit{below}, \textit{in front of}, and \textit{behind}.

\subsection{Probe-Level Entanglement Stress Tests}
\label{app:probe_entanglement_prompts}

These probes test whether targeted perturbations produce diagonal or off-diagonal responses in the measured Perturbation-Response Signature.

\begin{promptbox}{Language-Ruined Probe}
\textbf{Required Input:}
\begin{itemize}[itemsep=2pt,parsep=0pt]
    \item \textbf{Decoy Question}: \{decoy\_question\}
    \item \textbf{Answer Constraint}: \{constraint\}
\end{itemize}

\textbf{Example Decoy Question:}

What is the capital city of France?

\textbf{Prompt Template:}

\{decoy\_question\}

\{constraint\}
\end{promptbox}

\begin{promptbox}{Alignment-Ruined Probe: Fallback Opposite-Answer Prompt}
\textbf{Use Case:}

This fallback is used when deterministic relation inversion or polarity inversion is unavailable.

\textbf{Required Input:}
\begin{itemize}[itemsep=2pt,parsep=0pt]
    \item \textbf{Original Prompt}: \{original\_prompt\}
\end{itemize}

\textbf{Prompt Template:}

Answer the OPPOSITE of what you would normally answer to the following question.

\{original\_prompt\}
\end{promptbox}

\subsection{Intervention Routing Policies}
\label{app:intervention_routing_policies}

We provide the full prompt and parameter catalog for intervention routing. After computing the normalized dominant source $s^{*}_{\mathrm{norm}}(x,M)$, we choose one of three matched correction paths: visual enhancement for visual-dominant cases, anti-prior prompting for L-Dominant Cases, and grounding or relation checking for alignment-dominant cases. These templates and parameters define the intervention implementation used in the routing results. Appendix~\ref{app:intervention_routing} reports only the empirical outcomes.

\subsubsection{Generic Intervention Baseline}
\label{app:generic_intervention_prompt}

The generic baseline applies the same correction prompt to every example.

\begin{promptbox}{Generic Reasoning Intervention}
\textbf{Required Input:}
\begin{itemize}[itemsep=2pt,parsep=0pt]
    \item \textbf{Original Prompt}: \{original\_prompt\}
\end{itemize}

\textbf{Prompt Template:}

Think step by step about the image and question before answering.

\{original\_prompt\}
\end{promptbox}

\subsubsection{Matched Intervention for Visual-Dominant Cases}
\label{app:visual_intervention}

For visual-dominant cases, we apply image enhancement and keep the original prompt fixed.

\begin{promptbox}{Visual-Dominant Intervention: Image Enhancement}
\textbf{Required Input:}
\begin{itemize}[itemsep=2pt,parsep=0pt]
    \item \textbf{Image}: \{image\}
    \item \textbf{Original Prompt}: \{original\_prompt\}
\end{itemize}

\textbf{Image Enhancement:}
\begin{itemize}[itemsep=2pt,parsep=0pt]
    \item Upscale the image by $1.5\times$ using LANCZOS interpolation.
    \item Increase color saturation by $1.1\times$.
    \item Increase sharpness by $2.5\times$.
\end{itemize}

\textbf{Prompt Template:}

\{original\_prompt\}
\end{promptbox}

\subsubsection{Matched Intervention for L-Dominant Cases}
\label{app:anti_prior_intervention}

For L-Dominant Cases, we add an anti-prior instruction before the original prompt.

\begin{promptbox}{L-Dominant Intervention: Anti-Prior Prompt}
\textbf{Required Input:}
\begin{itemize}[itemsep=2pt,parsep=0pt]
    \item \textbf{Original Prompt}: \{original\_prompt\}
\end{itemize}

\textbf{Prompt Template:}

Look at the image very carefully before answering. Do NOT guess based on common sense or what is usually true. Base your answer ONLY on what you can actually see in this specific image.

\{original\_prompt\}
\end{promptbox}

\subsubsection{Matched Intervention for Alignment-Dominant Cases}
\label{app:alignment_intervention}

For alignment-dominant cases, we use grounded answering for general inputs and relation checking for VSR.

\begin{promptbox}{Alignment-Dominant Intervention: Grounded Answering}
\textbf{Required Input:}
\begin{itemize}[itemsep=2pt,parsep=0pt]
    \item \textbf{Original Prompt}: \{original\_prompt\}
\end{itemize}

\textbf{Prompt Template:}

Look at the image carefully. First, describe what you see in the image in one sentence. Then answer the following question based on your description.

\{original\_prompt\}
\end{promptbox}

\begin{promptbox}{Alignment-Dominant Intervention: VSR Relation Check}
\textbf{Required Input:}
\begin{itemize}[itemsep=2pt,parsep=0pt]
    \item \textbf{Caption}: \{caption\}
\end{itemize}

\textbf{Prompt Template:}

Follow these steps:

1. First, list the objects you can see in the image.

2. Then, describe the spatial relationship between them.

3. Finally, is this statement true or false: ``\{caption\}''

Answer with only `True' or `False' on the last line.
\end{promptbox}

\subsection{Answer Normalization}
\label{app:answer_normalization}

We apply answer normalization before comparing clean, perturbed, and intervened outputs. For binary yes/no datasets, surface variants such as `yes', `Yes.', and `yes, it is' are mapped to the canonical label \texttt{yes}; analogous variants are mapped to \texttt{no}. For VSR, \texttt{true} and \texttt{false} are the canonical labels. For VizWiz, \texttt{unanswerable} variants are mapped to a single answerability label, and short open-ended answers are lowercased and stripped of punctuation before exact matching. These normalization rules are applied consistently to clean predictions, diagnostic probes, scalarization comparisons, and intervention-routing evaluations.

\section{Additional Results}
\label{app:additional_results}

We provide additional result that support the Perturbation-Response view. We first report probe-design controls for image-removal confidence-retention saturation and alignment availability, and then specify dataset composition and source-profile statistics. We then examine controlled interventions and their off-target effects. After that, we analyze failure-family diagnosis, dataset-identity controls, human validation, abstention, and intervention routing. The additional tables form a single empirical trace from source prevalence to diagnostic utility. The main pattern is consistent across the appendix. Multimodal source signals are entangled under intervention, yet their joint signature remains informative for failure diagnosis. The same signature should not be treated as a direct scalar confidence correction.

\subsection{Image-Removal Saturation Controls}
\label{app:probe_sensitivity}

The default image-removal score $L=L_{\mathrm{conf}}$
 in Eq.~\ref{eq:language_score}
 measures confidence retention after removing image evidence and does not require the blank-image answer to match the clean answer.
Because the clipped confidence ratio can saturate, we run implementation-level controls to check whether the main conclusions depend on this formulation. We also evaluate the stricter answer-persistence score $L_{\mathrm{persist}}$
 from Eq.~\ref{eq:language_persist}
, which additionally requires the original answer to survive image removal.
These controls do not replace the default score used in the main paper. They test whether image-removal saturation alone explains the diagnostic results.


We compare four variants. First, $L_{\mathrm{conf}}$ is the default bounded confidence-retention score from Eq.~\ref{eq:language_score}. Second, $L_{\mathrm{persist}}$ additionally requires the blank-image answer to match the clean answer. Third, $L_{\mathrm{match}}$ removes the confidence ratio and keeps only blank-image answer survival:
\begin{equation}
\begin{split}
& L_{\mathrm{match}}
=\mathbf{1}[\mathrm{match}(\hat{y},\tilde{y}^{\ell})].
\end{split}
\label{eq:l_match_control}
\end{equation}
Fourth, $L_{\mathrm{stable}}$ keeps the confidence-ratio interpretation but suppresses low-denominator amplification by replacing the clean-confidence denominator with a fixed floor $\tau_c$:
\begin{equation}
\begin{aligned}
L_{\mathrm{stable}}
&=\mathbf{1}[\mathrm{match}(\hat{y},\tilde{y}^{\ell})] \\
&\quad \cdot \min\left(
\frac{c(\tilde{y}^{\ell})}
{\max(c(\hat{y}),\tau_c)+\epsilon},
1
\right).
\end{aligned}
\label{eq:l_stable_control}
\end{equation}
We set $\tau_c=0.5$ in all experiments.

\begin{table*}[t]
\centering
\small
\setlength{\tabcolsep}{4pt}
\scalebox{1}{
\begin{tabular}{lccccc}
\toprule
\textbf{Variant} & \textbf{Removal Input} & \textbf{Ratio Treatment} & \textbf{$L$ Dom.} & \textbf{AUROC} & \textbf{Macro-F1} \\
\midrule
$L_{\mathrm{conf}}$ & Blank & Clipped ratio & $99.7\%$ & $0.828$ & $0.333$ \\
$L_{\mathrm{persist}}$ & Blank & Match-gated ratio & $51.6\%$ & $0.823$ & $0.349$ \\
$L_{\mathrm{match}}$ & Blank & None & $51.7\%$ & $0.823$ & $0.346$ \\
$L_{\mathrm{stable}}$ & Blank & Floored denominator & $51.6\%$ & $0.823$ & $0.349$ \\
\bottomrule
\end{tabular}}
\caption{Image-removal saturation controls. $L$ Dom. denotes the fraction of examples for which $L$ is the raw dominant coordinate. AUROC and Macro-F1 use the same pooled failure-family diagnosis setting with a stratified 5-fold logistic-regression classifier, replacing only the image-removal coordinate.}
\label{tab:app_language_saturation_controls}
\end{table*}

For $L_{\mathrm{persist}}$, observed clean/blank answer agreement is $51.7\%$, compared with $28.8\%$ agreement expected from the answer marginals, giving chance-adjusted agreement $0.321$. XGBoost AUROC remains $0.952$ with $(V,L_{\mathrm{persist}},A)$ and $0.967$ when scalar confidence is added.


We observe from Table~\ref{tab:app_language_saturation_controls} that the raw dominance rate of $L$ drops sharply when answer matching is required, from $99.7\%$ for $L_{\mathrm{conf}}$ to about $51.6$--$51.7\%$ for the stricter variants. However, the diagnostic result remains stable. Replacing $L_{\mathrm{conf}}$ with $L_{\mathrm{persist}}$, $L_{\mathrm{match}}$, or $L_{\mathrm{stable}}$ changes AUROC only slightly, while Macro-F1 remains comparable or improves. Thus, the main diagnostic pattern does not depend on the broader confidence-retention formulation. The default $L$ should be read as confidence retention under image removal, not as answer persistence or a causal source label.

\subsection{Alignment Availability Controls}
\label{app:alignment_sensitivity}

The alignment score in Eq.~\ref{eq:alignment_score} is availability-aware. The grounding re-check component is available for all datasets, while relation-swap probing is available only when the input contains an explicit perturbable relation. In our evaluation suite, this applies to VSR. Therefore, raw $A$ magnitudes should not be interpreted as directly comparable source intensities across all datasets. We instead use $A$ as a behavioral alignment-stress coordinate and rely on dataset/model controls and within-setting normalization when the score is used across heterogeneous benchmarks.

We run an availability control to test whether the diagnostic value of $A$ is driven by the VSR-only relation-swap branch or by the $0.4/0.6$ grounding-relation weighting. This control removes VSR entirely from the pooled diagnosis setting. Since VSR is the only relation-eligible dataset, the no-VSR setting contains no relation-swap examples and no example for which the relation weight can affect $A$. The comparison therefore asks whether the Perturbation-Response Signature remains diagnostic when only the non-relation benchmark suite is used.

\begin{table*}[t]
\centering
\small
\setlength{\tabcolsep}{5pt}
\scalebox{1}{
\begin{tabular}{llccc}
\toprule
\textbf{Feature Set} & \textbf{Evaluation Scope} & \textbf{Relation-Swap Eligible} & \textbf{AUROC} & \textbf{Macro-F1} \\
\midrule
$c$ & All datasets & Yes & $0.752$ & $0.260$ \\
$(V,L,A)$ & All datasets & Yes & $0.828$ & $0.333$ \\
$c$ & Excluding VSR & No & $0.761$ & $0.329$ \\
$(V,L,A)$ & Excluding VSR & No & $0.839$ & $0.422$ \\
\bottomrule
\end{tabular}
}
\caption{Alignment availability control. We compare scalar confidence $c$ with the Perturbation-Response Signature $(V,L,A)$ under the same stratified 5-fold logistic-regression failure-family diagnosis setting. The no-VSR setting removes the only relation-swap-eligible dataset, so any remaining gain cannot be attributed to the VSR-only relation branch or to the grounding-relation weight in Eq.~\ref{eq:alignment_score}.}
\label{tab:app_alignment_availability_control}
\end{table*}

We observe from Table~\ref{tab:app_alignment_availability_control} that removing VSR does not reduce the diagnostic value of the Perturbation-Response Signature. With all datasets, $(V,L,A)$ improves AUROC over scalar confidence from $0.752$ to $0.828$ and Macro-F1 from $0.260$ to $0.333$. After excluding VSR, the gain remains nearly the same in AUROC and becomes larger in Macro-F1: AUROC improves from $0.761$ to $0.839$, while Macro-F1 improves from $0.329$ to $0.422$. This indicates that the main diagnosis result is not an artifact of the relation-swap branch or the particular $0.4/0.6$ weighting used when relation-swap probing is available. The relation-swap branch should be read as a stronger task-valid alignment stress for relation-bearing inputs, not as the source of the pooled diagnostic effect.

\subsection{Visual-Perturbation Robustness}
\label{app:visual_robustness}

We test whether the diagnosis result depends on the choice or strength of the visual perturbations. Removing blur, crop, brightness, or noise individually keeps diagnosis AUROC between $0.822$ and $0.824$, while normalized source assignments agree with the complete four-probe version on $89.8\%$--$97.8\%$ of examples.

We further evaluate weak, default, and strong perturbations on the same $10{,}540$ model-output instances.

\begin{table}[t]
\centering
\small
\setlength{\tabcolsep}{3pt}
\begin{tabular}{lcccccc}
\toprule
\textbf{Setting} & \textbf{Blur} & \textbf{Crop} & \textbf{Bright.} & \textbf{Noise} & $\mathbf{\bar V}$ & \textbf{Flip} \\
\midrule
Weak    & $8$  & $75\%$ & $50\%$ & $10$ & $0.132$ & $0.174$ \\
Default & $15$ & $50\%$ & $30\%$ & $25$ & $0.196$ & $0.263$ \\
Strong  & $22$ & $25\%$ & $15\%$ & $50$ & $0.252$ & $0.342$ \\
\bottomrule
\end{tabular}
\caption{Visual-perturbation severity analysis. Stronger perturbations increase both mean visual sensitivity and answer-flip rate.}
\label{tab:visual_perturbation_sensitivity}
\end{table}

Instance rankings remain stable relative to the default setting, with Spearman correlation $0.892$ for weak and $0.924$ for strong perturbations. With $L_{\mathrm{persist}}$, diagnosis AUROC is $0.785$, $0.786$, and $0.783$ under weak, default, and strong perturbations, respectively, compared with $0.632$ for scalar confidence. Macro-F1 remains $0.405$--$0.406$, compared with $0.179$ for scalar confidence. All paired-bootstrap intervals for the AUROC and Macro-F1 gains exclude zero.


\begin{table*}[t]
\centering
\small
\begin{tabular}{llrrrrrrrrr}
\toprule
\textbf{Model} & \textbf{Dataset} & \textbf{$N$} & \textbf{Acc} & \textbf{$\bar{V}$} & \textbf{$\bar{L}$} & \textbf{$\bar{A}$} & \textbf{LP Match} & \textbf{Dom-$V$} & \textbf{Dom-$L$} & \textbf{Dom-$A$} \\
\midrule
Gemma 4 & HallusionBench & 951 & 68.2\% & 0.122 & 0.958 & 0.501 & 56.9\% & 0 & 949 & 2 \\
Gemma 4 & POPE & 9000 & 86.9\% & 0.114 & 0.999 & 0.336 & 53.3\% & 0 & 9000 & 0 \\
Gemma 4 & VizWiz & 4319 & 54.0\% & 0.303 & 0.996 & 0.638 & 37.5\% & 2 & 4308 & 9 \\
Gemma 4 & VSR & 340 & 71.8\% & 0.143 & 0.999 & 0.675 & 64.4\% & 0 & 340 & 0 \\
LLaVA-Mistral & HallusionBench & 951 & 52.6\% & 0.090 & 0.979 & 0.241 & 58.4\% & 0 & 951 & 0 \\
LLaVA-Mistral & POPE & 9000 & 89.1\% & 0.104 & 0.999 & 0.061 & 55.8\% & 0 & 9000 & 0 \\
LLaVA-Mistral & VizWiz & 4319 & 54.9\% & 0.192 & 0.994 & 0.479 & 44.0\% & 0 & 4302 & 17 \\
LLaVA-Mistral & VSR & 340 & 70.3\% & 0.158 & 1.000 & 0.099 & 50.6\% & 0 & 340 & 0 \\
LLaVA-Vicuna & HallusionBench & 951 & 47.7\% & 0.071 & 0.945 & 0.208 & 71.2\% & 0 & 951 & 0 \\
LLaVA-Vicuna & POPE & 9000 & 88.6\% & 0.116 & 0.992 & 0.204 & 56.5\% & 0 & 9000 & 0 \\
LLaVA-Vicuna & VizWiz & 4319 & 53.9\% & 0.215 & 0.977 & 0.604 & 39.6\% & 0 & 4198 & 121 \\
LLaVA-Vicuna & VSR & 340 & 51.5\% & 0.037 & 0.914 & 0.105 & 96.5\% & 0 & 340 & 0 \\
Qwen3-VL & HallusionBench & 951 & 74.4\% & 0.142 & 0.962 & 0.318 & 58.6\% & 0 & 951 & 0 \\
Qwen3-VL & POPE & 9000 & 88.9\% & 0.094 & 1.000 & 0.265 & 56.9\% & 0 & 9000 & 0 \\
Qwen3-VL & VizWiz & 4319 & 61.3\% & 0.234 & 0.999 & 0.488 & 39.7\% & 0 & 4318 & 1 \\
Qwen3-VL & VSR & 340 & 87.4\% & 0.142 & 0.995 & 0.613 & 51.8\% & 0 & 339 & 1 \\
\bottomrule
\end{tabular}
\caption{Full source-profile with accuracy and raw dominant-source (\textbf{Dom}) counts. LP Match denotes exact-answer agreement between the original image and blank-image condition.}
\label{tab:app_source_profiles_full}
\end{table*}
\begin{table*}[t]
\centering
\small
\begin{tabular}{lrrrr}
\toprule
\textbf{Model} & \textbf{$N$} & \textbf{Dominant $V$} & \textbf{Dominant $L$} & \textbf{Dominant $A$} \\
\midrule
Gemma 4 & $14{,}610$ & $2$ ($0.01\%$) & $14{,}597$ ($99.91\%$) & $11$ ($0.08\%$) \\
LLaVA-Mistral & $14{,}610$ & $0$ ($0.00\%$) & $14{,}593$ ($99.88\%$) & $17$ ($0.12\%$) \\
LLaVA-Vicuna & $14{,}610$ & $0$ ($0.00\%$) & $14{,}489$ ($99.17\%$) & $121$ ($0.83\%$) \\
Qwen3-VL & $14{,}610$ & $0$ ($0.00\%$) & $14{,}608$ ($99.99\%$) & $2$ ($0.01\%$) \\
Pooled & $58{,}440$ & $2$ ($<0.01\%$) & $58{,}287$ ($99.74\%$) & $151$ ($0.26\%$) \\
\bottomrule
\end{tabular}
\caption{\textbf{Largest-Coordinate Counts}. We compute the largest coordinate using $\arg\max(V,L,A)$. The largest coordinate is overwhelmingly $L$ because $L$ is near saturation in most settings; this count should therefore be read as prevalence, not diagnostic value.}
\label{tab:app_dominant_source_counts}
\end{table*}

\subsection{Full Source Profiles}
\label{app:full_source_profiles}

We use the full source-profile tables to separate raw prevalence from diagnostic value. We observe from Table~\ref{tab:app_source_profiles_full} that $L$ is high across nearly all model-dataset settings, with means from $0.914$ to $1.000$ and blank-image match rates from $37.5\%$ to $96.5\%$. In contrast, $V$ and $A$ show stronger task dependence: VizWiz gives the largest visual-fragility values, while VSR gives the largest alignment values for Gemma 4 and Qwen3-VL. We observe from Table~\ref{tab:app_dominant_source_counts} that raw $\arg\max(V,L,A)$ selects $L$ for $58{,}287/58{,}440$ pooled instances. This saturation makes raw dominance a poor proxy for diagnostic value, motivating our use of the full $(V,L,A)$ vector for failure-family diagnosis and normalized dominance for routing.



\begin{table*}[t]
\centering
\small
\begin{tabular}{lrrrr}
\toprule
\textbf{Model} & \textbf{Visual-Ruined} & \textbf{Language-Ruined} & \textbf{Alignment-Ruined} & \textbf{All Checks} \\
\midrule
Gemma 4 & 3/4 & 1/4 & 1/4 & 5/12 \\
LLaVA-Mistral & 3/4 & 0/4 & 1/4 & 4/12 \\
LLaVA-Vicuna & 1/4 & 0/4 & 1/4 & 2/12 \\
Qwen3-VL & 4/4 & 1/4 & 2/4 & 7/12 \\
\midrule
\textbf{Pooled} & \textbf{11/16} & \textbf{2/16} & \textbf{5/16} & \textbf{18/48} \\
\bottomrule
\end{tabular}
\caption{\textbf{Magnitude-based diagonal-dominance pass counts for the interventional entanglement test}. A check passes when the intended source coordinate has the largest absolute change under the corresponding ruined condition.}
\label{tab:app_identifiability_pass_counts}
\end{table*}


\begin{table*}[t]
\centering
\small
\begin{tabular}{llrrrl}
\toprule
\textbf{Dataset} & \textbf{Intervention} & \textbf{$\Delta V$} & \textbf{$\Delta L$} & \textbf{$\Delta A$} & \textbf{Outcome} \\
\midrule
HallusionBench & Visual-ruined & -0.075 & +0.016 & +0.126 & FAIL \\
HallusionBench & Language-ruined & -0.061 & -0.025 & +0.088 & FAIL \\
HallusionBench & Alignment-ruined & -0.001 & -0.000 & -0.001 & PASS \\
POPE & Visual-ruined & -0.086 & +0.000 & +0.068 & PASS \\
POPE & Language-ruined & -0.031 & -0.052 & +0.247 & FAIL \\
POPE & Alignment-ruined & +0.000 & +0.000 & +0.000 & FAIL \\
VizWiz & Visual-ruined & -0.198 & +0.000 & +0.049 & PASS \\
VizWiz & Language-ruined & -0.228 & -0.002 & -0.009 & FAIL \\
VizWiz & Alignment-ruined & -0.000 & +0.000 & +0.000 & FAIL \\
VSR & Visual-ruined & -0.143 & -0.001 & +0.044 & PASS \\
VSR & Language-ruined & -0.023 & -0.038 & -0.023 & PASS \\
VSR & Alignment-ruined & -0.014 & +0.000 & +0.005 & FAIL \\
\bottomrule
\end{tabular}
\caption{\textbf{Interventional Entanglement $\Delta$ for Gemma 4}. A check passes when the intended source coordinate has the largest absolute change under the corresponding ruined condition. Resultsare computed using unrounded $\Delta$.}
\label{tab:app_identifiability_gemma_vl}
\end{table*}


\begin{table*}[t]
\centering
\small
\begin{tabular}{llrrrl}
\toprule
\textbf{Dataset} & \textbf{Intervention} & \textbf{$\Delta V$} & \textbf{$\Delta L$} & \textbf{$\Delta A$} & \textbf{Outcome} \\
\midrule
HallusionBench & Visual-ruined & -0.026 & +0.011 & +0.012 & PASS \\
HallusionBench & Language-ruined & -0.051 & +0.007 & -0.054 & FAIL \\
HallusionBench & Alignment-ruined & +0.001 & +0.000 & +0.003 & PASS \\
POPE & Visual-ruined & -0.050 & +0.001 & -0.006 & PASS \\
POPE & Language-ruined & -0.076 & -0.014 & +0.097 & FAIL \\
POPE & Alignment-ruined & +0.000 & +0.000 & +0.000 & FAIL \\
VizWiz & Visual-ruined & -0.061 & +0.002 & -0.175 & FAIL \\
VizWiz & Language-ruined & -0.072 & -0.020 & +0.255 & FAIL \\
VizWiz & Alignment-ruined & -0.000 & +0.000 & -0.000 & FAIL \\
VSR & Visual-ruined & -0.153 & -0.002 & -0.093 & PASS \\
VSR & Language-ruined & -0.146 & -0.001 & +0.004 & FAIL \\
VSR & Alignment-ruined & -0.023 & -0.000 & +0.008 & FAIL \\
\bottomrule
\end{tabular}
\caption{\textbf{Interventional Entanglement $\Delta$ for LLaVA-Mistral}. A check passes when the intended source coordinate has the largest absolute change under the corresponding ruined condition. Results are computed using unrounded $\Delta$.}
\label{tab:app_identifiability_llava_mistral}
\end{table*}

\begin{table*}[t]
\centering
\small
\begin{tabular}{llrrrl}
\toprule
\textbf{Dataset} & \textbf{Intervention} & \textbf{$\Delta V$} & \textbf{$\Delta L$} & \textbf{$\Delta A$} & \textbf{Outcome} \\
\midrule
HallusionBench & Visual-ruined & -0.020 & +0.038 & +0.022 & FAIL \\
HallusionBench & Language-ruined & -0.053 & +0.037 & -0.024 & FAIL \\
HallusionBench & Alignment-ruined & +0.002 & -0.001 & -0.002 & PASS \\
POPE & Visual-ruined & -0.029 & +0.008 & +0.239 & FAIL \\
POPE & Language-ruined & -0.101 & -0.004 & -0.018 & FAIL \\
POPE & Alignment-ruined & +0.000 & +0.000 & +0.000 & FAIL \\
VizWiz & Visual-ruined & -0.121 & +0.016 & -0.025 & PASS \\
VizWiz & Language-ruined & -0.061 & +0.019 & +0.097 & FAIL \\
VizWiz & Alignment-ruined & +0.000 & -0.000 & -0.000 & FAIL \\
VSR & Visual-ruined & -0.025 & +0.074 & -0.062 & FAIL \\
VSR & Language-ruined & -0.024 & +0.074 & +0.197 & FAIL \\
VSR & Alignment-ruined & +0.001 & +0.007 & -0.001 & FAIL \\
\bottomrule
\end{tabular}
\caption{\textbf{Interventional Entanglement $\Delta$ for LLaVA-Vicuna}. A check passes when the intended source coordinate has the largest absolute change under the corresponding ruined condition. Results are computed using unrounded $\Delta$.}
\label{tab:app_identifiability_llava_vicuna}
\end{table*}

\begin{table*}[t]
\centering
\small
\begin{tabular}{llrrrl}
\toprule
\textbf{Dataset} & \textbf{Intervention} & \textbf{$\Delta V$} & \textbf{$\Delta L$} & \textbf{$\Delta A$} & \textbf{Outcome} \\
\midrule
HallusionBench & Visual-ruined & -0.085 & +0.011 & -0.009 & PASS \\
HallusionBench & Language-ruined & -0.047 & -0.086 & +0.175 & FAIL \\
HallusionBench & Alignment-ruined & -0.001 & +0.000 & +0.002 & PASS \\
POPE & Visual-ruined & -0.049 & +0.000 & +0.024 & PASS \\
POPE & Language-ruined & +0.003 & -0.149 & +0.293 & FAIL \\
POPE & Alignment-ruined & +0.000 & +0.000 & +0.000 & FAIL \\
VizWiz & Visual-ruined & -0.114 & -0.000 & -0.035 & PASS \\
VizWiz & Language-ruined & -0.114 & -0.006 & +0.052 & FAIL \\
VizWiz & Alignment-ruined & -0.000 & +0.000 & +0.000 & PASS \\
VSR & Visual-ruined & -0.126 & -0.000 & +0.049 & PASS \\
VSR & Language-ruined & +0.081 & -0.128 & -0.000 & PASS \\
VSR & Alignment-ruined & -0.020 & +0.000 & -0.003 & FAIL \\
\bottomrule
\end{tabular}
\caption{\textbf{Interventional Entanglement $\Delta$ for Qwen3-VL}. A check passes when the intended source coordinate has the largest absolute change under the corresponding ruined condition. Results are computed using unrounded $\Delta$.}
\label{tab:app_identifiability_qwen_vl}
\end{table*}

\subsection{Interventional Entanglement Across Models}
\label{app:interventional_entanglement}

We next test whether targeted interventions isolate the intended source score. This test is stricter than source profiling. If the three scores behaved as independent source components, visual-ruined inputs would produce the largest absolute change in $V$, language-ruined inputs would produce the largest absolute change in $L$, and alignment-ruined inputs would produce the largest absolute change in $A$. We observe from Table~\ref{tab:app_identifiability_pass_counts} that this magnitude-based diagonal-dominance pattern holds only partially. Across all model-dataset-intervention checks, $18$ of $48$ pass. Visual-ruined inputs pass in $11$ of $16$ settings, while language-ruined inputs pass in $2$ of $16$ and alignment-ruined inputs pass in $5$ of $16$. This supports a more precise interpretation of $(V,L,A)$. Visual ruin is often source-selective in magnitude, although the direction is not always positive. Language and alignment interventions remain much less source-selective. Therefore, the vector should be read as a joint behavioral signature rather than as a clean causal decomposition.

The model-specific tables show how off-target responses arise across backbones. We observe from Table~\ref{tab:app_identifiability_gemma_vl} that Gemma 4 satisfies magnitude-based dominance for visual-ruined inputs in $3$ of $4$ settings, but the signed response is not always intuitive. On VizWiz, visual ruin reduces $V$ by $0.198$ rather than increasing it. On HallusionBench, visual ruin also increases $A$ by $0.126$, which shows that visual corruption can change alignment behavior. These rows separate two effects: the intended coordinate may change most in magnitude, while the sign and off-target movement still reveal non-monotone source behavior.

For LLaVA-Mistral, we observe from Table~\ref{tab:app_identifiability_llava_mistral} that visual-ruined inputs pass in $3$ of $4$ settings, while language-ruined inputs pass in none. VizWiz language ruin increases $A$ by $0.255$, and VSR visual ruin changes both $V$ and $A$ in the negative direction. These responses indicate that entanglement is not limited to one model family. They also show why magnitude and signed direction should be interpreted separately.

For LLaVA-Vicuna, we observe from Table~\ref{tab:app_identifiability_llava_vicuna} that only $2$ of $12$ checks pass overall. HallusionBench alignment ruin passes, while most other checks fail the magnitude-based criterion. POPE visual ruin increases $A$ by $0.239$, even though the intended target is visual. VSR language ruin also increases $A$ by $0.197$. These rows reinforce the joint-signature interpretation because off-target alignment responses can dominate or substantially accompany the intended response.

For Qwen3-VL, we observe from Table~\ref{tab:app_identifiability_qwen_vl} that $7$ of $12$ checks pass, including all four visual-ruined settings. However, the signed deltas still show non-monotone behavior. On HallusionBench, visual ruin reduces $V$ by $0.085$ rather than increasing it, while language ruin reduces $L$ by $0.086$ and increases $A$ by $0.175$. On POPE, language ruin reduces $L$ by $0.149$ and increases $A$ by $0.293$. These cases show why independent source attribution is too strong. The vector records how the answer changes under controlled pressure, and a ruined condition can shift the model toward a different answering route.

\subsection{Perturbation-Response Signatures By Failure Family}
\label{app:source_signatures_by_failure}

We connect Perturbation-Response Signatures to failure families before turning to predictive diagnosis. Table~\ref{tab:app_family_source_means} reports pooled source means by raw failure-family label. Object hallucination is associated with very high image-removal confidence retention, with $\bar{L}=0.996$, and a high scalar score of $0.967$. Answerability cases also have near-saturated image-removal confidence-retention scores, with $\bar{L}=0.997$, although their accuracy is lower at $77.4\%$. The no-annotated-family row has the highest mean visual and alignment values, with $\bar{V}=0.289$ and $\bar{A}=0.587$, because this row is heavily shaped by answerable VizWiz instances. This row should not be read as the correctness-conditioned no-failure target used in the outcome-family controls. These values explain why single-coordinate interpretations are insufficient. Failure families do not map cleanly onto one isolated source. Instead, they occupy different regions of the joint Perturbation-Response space.

\begin{table*}[t]

\centering
\small
\begin{tabular}{lrrrrrr}
\toprule
\textbf{Failure Family} & \textbf{$N$} & \textbf{$\bar{V}$} & \textbf{$\bar{L}$} & \textbf{$\bar{A}$} & \textbf{Scalar} & \textbf{Acc} \\
\midrule
Answerability & 5540 & 0.124 & 0.997 & 0.479 & 0.880 & 77.4\% \\
Attribute & 1440 & 0.099 & 0.946 & 0.302 & 0.922 & 64.3\% \\
None & 11736 & 0.289 & 0.988 & 0.587 & 0.836 & 45.9\% \\
Object Hallucination & 38364 & 0.107 & 0.996 & 0.223 & 0.967 & 86.5\% \\
Spatial & 1360 & 0.120 & 0.977 & 0.373 & 0.919 & 70.2\% \\
\bottomrule
\end{tabular}
\caption{Pooled source-signature means by failure family.}
\label{tab:app_family_source_means}
\end{table*}

\subsection{Full Failure-Family Diagnosis}
\label{app:full_failure_family_diagnosis}

We report model-wise failure-family AUROC in Table~\ref{tab:app_failure_diagnosis_full}. The table compares scalar confidence, the Perturbation-Response Signature $(V,L,A)$, and the combined representation $(V,L,A,c)$ across classifier families. We report tree-based feature importance in Table~\ref{tab:app_feature_importance_full}. These results provide the full supporting evidence for the main diagnosis claim summarized in Section~\ref{sec:results}.

\begin{table*}[t]
\centering
\small
\begin{tabular}{llrrrr}
\toprule
\textbf{Model} & \textbf{Classifier} & \textbf{Scalar} & \textbf{V+L+A} & \textbf{V+L+A+Scalar} & \textbf{$\Delta_{VLA-S}$} \\
\midrule
Gemma 4 & Logistic Reg. & 0.754$\pm$0.009 & 0.921$\pm$0.006 & 0.932$\pm$0.006 & +0.167 \\
Gemma 4 & XGBoost & 0.834$\pm$0.013 & 0.965$\pm$0.001 & 0.974$\pm$0.002 & +0.131 \\
Gemma 4 & LightGBM & 0.836$\pm$0.012 & 0.965$\pm$0.001 & 0.973$\pm$0.002 & +0.130 \\
LLaVA-Mistral & Logistic Reg. & 0.847$\pm$0.003 & 0.892$\pm$0.003 & 0.907$\pm$0.003 & +0.045 \\
LLaVA-Mistral & XGBoost & 0.860$\pm$0.003 & 0.973$\pm$0.002 & 0.979$\pm$0.002 & +0.113 \\
LLaVA-Mistral & LightGBM & 0.859$\pm$0.004 & 0.973$\pm$0.001 & 0.980$\pm$0.002 & +0.115 \\
LLaVA-Vicuna & Logistic Reg. & 0.868$\pm$0.005 & 0.876$\pm$0.005 & 0.941$\pm$0.004 & +0.009 \\
LLaVA-Vicuna & XGBoost & 0.873$\pm$0.004 & 0.960$\pm$0.001 & 0.981$\pm$0.001 & +0.087 \\
LLaVA-Vicuna & LightGBM & 0.872$\pm$0.004 & 0.960$\pm$0.002 & 0.981$\pm$0.000 & +0.087 \\
Qwen3-VL & Logistic Reg. & 0.730$\pm$0.006 & 0.756$\pm$0.005 & 0.757$\pm$0.006 & +0.027 \\
Qwen3-VL & XGBoost & 0.806$\pm$0.007 & 0.972$\pm$0.002 & 0.976$\pm$0.002 & +0.166 \\
Qwen3-VL & LightGBM & 0.836$\pm$0.005 & 0.972$\pm$0.002 & 0.976$\pm$0.002 & +0.137 \\
Pooled & Logistic Reg. & 0.752$\pm$0.002 & 0.828$\pm$0.003 & 0.873$\pm$0.002 & +0.077 \\
Pooled & XGBoost & 0.781$\pm$0.002 & 0.954$\pm$0.002 & 0.971$\pm$0.001 & +0.173 \\
Pooled & LightGBM & 0.780$\pm$0.002 & 0.953$\pm$0.002 & 0.968$\pm$0.002 & +0.172 \\
\bottomrule
\end{tabular}
\caption{Failure-family prediction AUROC by model and classifier. $\Delta_{VLA-S}$ denotes the gain of $V+L+A$ over scalar confidence.}
\label{tab:app_failure_diagnosis_full}
\end{table*}

\begin{table*}[t]

\centering
\small
\begin{tabular}{llrrrc}
\toprule
\textbf{Model} & \textbf{Classifier} & \textbf{$V$} & \textbf{$L$} & \textbf{$A$} & \textbf{Scalar Confidence} \\
\midrule
Gemma 4 & XGBoost & 0.105 & 0.305 & \textbf{0.505} & 0.085 \\
Gemma 4 & LightGBM & 0.269 & 0.153 & \textbf{0.292} & 0.286 \\
LLaVA-Mistral & XGBoost & 0.084 & 0.102 & \textbf{0.688} & 0.125 \\
LLaVA-Mistral & LightGBM & 0.294 & 0.099 & \textbf{0.331} & 0.276 \\
LLaVA-Vicuna & XGBoost & 0.107 & 0.178 & \textbf{0.503} & 0.212 \\
LLaVA-Vicuna & LightGBM & 0.211 & 0.173 & \textbf{0.317} & 0.299 \\
Qwen3-VL & XGBoost & 0.093 & 0.274 & \textbf{0.580} & 0.054 \\
Qwen3-VL & LightGBM & 0.257 & 0.153 & \textbf{0.329} & 0.260 \\
Pooled & XGBoost & 0.147 & 0.200 & \textbf{0.504} & 0.150 \\
Pooled & LightGBM & 0.222 & 0.146 & \textbf{0.360} & 0.271 \\
\bottomrule
\end{tabular}
\caption{Feature importance for tree-based failure-family diagnosis.}
\label{tab:app_feature_importance_full}
\end{table*}

\subsection{Dataset-Identity And Within-Dataset Controls}
\label{app:dataset_identity_controls}

The evaluation suite is deliberately heterogeneous: POPE primarily stresses object hallucination, VSR stresses spatial and relation errors, VizWiz stresses answerability, and HallusionBench contributes mixed hallucination cases. This design supports the diagnosis-first claim, although it also means that pooled failure-family prediction can partly reflect benchmark identity. We therefore use dataset-identity, within-dataset, leave-dataset-out, and matched source-stress controls to test whether Perturbation-Response Signatures add information beyond coarse dataset and model identifiers.

We observe from Table~\ref{tab:app_dataset_controlled_raw_family} that dataset and model identifiers already give a strong raw-family baseline. Dataset ID alone reaches AUROC $0.976$ and Macro-F1 $0.558$, while adding scalar confidence raises the result to AUROC $0.986$ and Macro-F1 $0.705$. Adding $V+L+A$ to dataset and model identifiers improves AUROC to $0.991$ and Macro-F1 to $0.783$. The AUROC gain is small because the controlled baseline is already high, while the Macro-F1 gain indicates better handling of less dominant failure labels.

\begin{table*}[t]

\centering
\small
\begin{tabular}{lcc}
\toprule
\textbf{Feature Set} & \textbf{AUROC} & \textbf{Macro-F1} \\
\midrule
Dataset ID & 0.976 & 0.558 \\
Dataset ID + Model ID & 0.975 & 0.558 \\
Dataset ID + Model ID + Scalar & 0.986 & 0.705 \\
Dataset ID + Model ID + $V+L+A$ & \textbf{0.991} & \textbf{0.783} \\
Dataset ID + Model ID + $V+L+A+\mathrm{Scalar}$ & \textbf{0.991} & \textit{0.782} \\
\bottomrule
\end{tabular}
\caption{Dataset-ID controlled failure-family diagnosis using a stratified held-out split. The target is the benchmark-provided failure-family label used in the pooled diagnosis experiment.}
\label{tab:app_dataset_controlled_raw_family}
\end{table*}

We observe from Table~\ref{tab:app_dataset_controlled_outcome_family} that the stricter outcome-family control gives a smaller but positive gain. Here correct predictions are mapped to \textit{no failure}, so the classifier must separate both correctness and error family. Dataset ID and model ID remain weak for Macro-F1, both at $0.145$. Adding scalar confidence raises Macro-F1 to $0.308$, and the combined representation reaches the best result with AUROC $0.853$ and Macro-F1 $0.348$.

\begin{table}
\centering
\setlength{\tabcolsep}{3pt}
\scalebox{1}{
\small
\begin{tabular}{p{0.25\textwidth}cc}
\toprule
\textbf{Feature Set} & \textbf{AUROC} & \textbf{Macro-F1} \\
\midrule
Dataset ID & 0.735 & 0.145 \\
Dataset ID + Model ID & 0.743 & 0.145 \\
Dataset ID + Model ID + Scalar & 0.847 & 0.308 \\
Dataset ID + Model ID + $V+L+A$ & 0.838 & 0.308 \\
Dataset ID + Model ID + $V+L+A+\mathrm{Scalar}$ & \textbf{0.853} & \textbf{0.348} \\
\bottomrule
\end{tabular}
}
\caption{Dataset-ID controlled outcome-family diagnosis using a stratified held-out split. Correct predictions are mapped to no-failure and incorrect predictions to their observed failure family.}
\label{tab:app_dataset_controlled_outcome_family}
\end{table}
\begin{table*}[t]
\centering
\small
\begin{tabular}{llclcc}
\toprule
\textbf{Dataset} & \textbf{Target} & \textbf{N} & \textbf{Feature Set} & \textbf{AUROC} & \textbf{F1} \\
\midrule
\multirow{3}{*}{HallusionBench} & \multirow{3}{*}{Attribute vs object} &  \multirow{3}{*}{3804} & Scalar & 0.634 & 0.621 \\
 &  &  & $V+L+A$ & 0.769 & 0.718 \\
 &  &  & $V+L+A+\mathrm{Scalar}$ & \textbf{0.801} & \textbf{0.739} \\
\midrule
\multirow{3}{*}{VizWiz} & \multirow{3}{*}{Answerability vs answerable} & \multirow{3}{*}{17276} & Scalar & 0.707 & 0.693 \\
 &  &  & $V+L+A$ & 0.817 & 0.782 \\
 &  &  & $V+L+A+\mathrm{Scalar}$ & \textbf{0.825} & \textbf{0.784} \\
\midrule
\multirow{3}{*}{POPE} & \multirow{3}{*}{Object error vs no failure} & \multirow{3}{*}{36000} & Scalar & 0.823 & 0.839 \\
 &  &  & $V+L+A$ & 0.839 & 0.853 \\
 &  &  & $V+L+A+\mathrm{Scalar}$ & \textbf{0.860} & \textbf{0.857} \\
\midrule
\multirow{3}{*}{VSR} & \multirow{3}{*}{Spatial error vs no failure} & \multirow{3}{*}{1360} & Scalar & 0.722 & 0.688 \\
 &  &  & $V+L+A$ & 0.724 & 0.701 \\
 &  &  & $V+L+A+\mathrm{Scalar}$ & \textbf{0.735} & \textbf{0.709} \\
\bottomrule
\end{tabular}
\caption{Within-dataset diagnostic controls. Each row fixes the dataset and includes model identity in all feature sets. Improvements from $V+L+A$ therefore cannot be attributed to dataset identity alone.}
\label{tab:app_within_dataset_diagnosis}
\end{table*}


We see in Table~\ref{tab:app_within_dataset_diagnosis} that the clearest fixed-dataset gains appear on HallusionBench and VizWiz. On HallusionBench, scalar confidence reaches AUROC $0.634$ and F1 $0.621$, while $(V,L,A)$ reaches AUROC $0.769$ and F1 $0.718$. The combined representation reaches AUROC $0.801$ and F1 $0.739$. On VizWiz, scalar confidence reaches AUROC $0.707$ and F1 $0.693$, while $(V,L,A)$ reaches AUROC $0.817$ and F1 $0.782$. The combined representation reaches AUROC $0.825$ and F1 $0.784$.

POPE and VSR give more conservative fixed-dataset results. On POPE, scalar confidence reaches AUROC $0.823$, while $(V,L,A)$ reaches $0.839$; the combined representation reaches AUROC $0.860$. On VSR, scalar confidence and $(V,L,A)$ are nearly matched, with AUROC $0.722$ and $0.724$, respectively; the combined representation reaches AUROC $0.735$. These rows bound the diagnosis claim: Perturbation-Response Signatures help most when the target contains multiple failure modes or answerability-sensitive distinctions.

\begin{table*}[t]
\centering
\small
\begin{tabular}{lccc}
\toprule
\textbf{Held-Out Dataset} & \textbf{Model ID + Scalar} & $\mathbf{Model ID + V+L+A}$ & $\mathbf{Model ID + V+L+A+Scalar}$ \\
\midrule
HallusionBench & 0.339 & \textbf{0.492} & 0.371 \\
POPE & 0.201 & \textbf{0.375} & 0.229 \\
VSR & 0.290 & \textbf{0.416} & 0.274 \\
VizWiz & 0.300 & 0.328 & \textbf{0.331} \\
\bottomrule
\end{tabular}
\caption{Leave-dataset-out binary failure detection. This is a stress test only, because multiclass failure-family support differs across datasets.}
\label{tab:app_leave_dataset_out_binary}
\end{table*}
\begin{table*}[t]
\centering
\setlength{\tabcolsep}{3pt}
\small
\begin{tabular}{lllrccccc}
\toprule
\textbf{Dataset} & \textbf{Slice} & \textbf{Expected Source} & \textbf{N} & \textbf{Acc} & $\bar{V}$ & $\bar{L}$ & $\bar{A}$ & \textbf{LP Match} \\
\midrule
POPE & adversarial & $L$ & 12000 & 86.5\% & 0.119 & 0.998 & 0.219 & 53.8\% \\
POPE & popular & $L$ & 12000 & 88.3\% & 0.108 & 0.998 & 0.209 & 55.6\% \\
POPE & random & $L$ & 12000 & 90.3\% & 0.094 & 0.997 & 0.222 & 57.5\% \\
VizWiz & unanswerable & $V$/answerability & 5540 & 77.4\% & 0.124 & 0.997 & 0.479 & 77.9\% \\
VizWiz & answerable & $V$/answerability & 11736 & 45.9\% & 0.289 & 0.988 & 0.587 & 22.4\% \\
HallusionBench & not visual-required & $V/A$ & 2016 & 56.3\% & 0.110 & 0.961 & 0.318 & 61.3\% \\
HallusionBench & visual-required & $V/A$ & 1788 & 65.8\% & 0.102 & 0.961 & 0.317 & 61.2\% \\
VSR & all relations & $A$ & 1360 & 70.2\% & 0.120 & 0.977 & 0.373 & 65.8\% \\
VSR & relation=at the right side of & $A$ & 76 & 65.8\% & 0.143 & 0.982 & 0.345 & 59.2\% \\
VSR & relation=behind & $A$ & 120 & 71.7\% & 0.110 & 0.980 & 0.368 & 66.7\% \\
VSR & relation=facing & $A$ & 72 & 56.9\% & 0.133 & 0.984 & 0.369 & 70.8\% \\
VSR & relation=in front of & $A$ & 104 & 61.5\% & 0.130 & 0.979 & 0.382 & 65.4\% \\
VSR & relation=on & $A$ & 96 & 83.3\% & 0.118 & 0.967 & 0.380 & 58.3\% \\
VSR & relation=on top of & $A$ & 72 & 73.6\% & 0.131 & 0.978 & 0.369 & 62.5\% \\
VSR & relation=touching & $A$ & 80 & 71.2\% & 0.116 & 0.982 & 0.382 & 68.8\% \\
VSR & relation=under & $A$ & 128 & 71.9\% & 0.099 & 0.972 & 0.385 & 75.0\% \\
\bottomrule
\end{tabular}
\caption{Matched source-stress slices from existing metadata. These slices check whether source signatures align with dataset-specific stress factors under fixed task conditions.}
\label{tab:app_source_stress_slices}
\end{table*}

We include leave-dataset-out binary transfer as a stress test in Table~\ref{tab:app_leave_dataset_out_binary}. The results are weak overall, which is expected because the datasets differ in task structure and failure-label support. Holding out POPE gives AUROC $0.201$ with model ID and scalar confidence, $0.375$ with model ID and $V+L+A$, and $0.229$ with the combined representation. Holding out HallusionBench gives $0.339$, $0.492$, and $0.371$, respectively. We therefore use this table as a transfer stress test, not as the main diagnosis evidence.
We find from Table~\ref{tab:app_source_stress_slices} that matched source-stress slices align with the same interpretation. In POPE, image-removal confidence-retention scores remain near saturation across adversarial, popular, and random splits, with $\bar{L}$ between $0.997$ and $0.998$, while accuracy drops from $90.3\%$ on random to $86.5\%$ on adversarial. In VizWiz, answerable examples have higher visual and alignment scores than unanswerable examples, with $\bar{V}=0.289$ and $\bar{A}=0.587$ versus $\bar{V}=0.124$ and $\bar{A}=0.479$. VSR relation slices show stable high image-removal confidence-retention scores and moderate alignment scores across relations. These slices provide interpretive support for reading the full $(V,L,A)$ vector rather than a single coordinate.





\subsection{Human Audit Results}
\label{app:human_audit_results}

We report the completed human-audit protocol in Table~\ref{tab:app_human_audit_template}. The audit follows the protocol in Appendix~\ref{app:human_audit}. We use $200$ examples, with $50$ examples each from HallusionBench, POPE, VSR, and VizWiz. Three annotators assign one of six labels: visual evidence insufficiency, language-prior reliance, alignment error, mixed, unclear, or no failure. Annotators see the image, question or caption, reference answer, model answer, and correctness flag. They do not see $V$, $L$, $A$, scalar confidence, raw dominant source, normalized dominant source, or routing decision. We manually adjudicate $36$ disagreement cases before computing the final audit labels.

\begin{table*}[t]
\centering
\setlength{\tabcolsep}{3pt}
\scalebox{1}{
\small
\begin{tabular}{lrrrr}
\toprule
\textbf{Dataset} & \textbf{$N$} & \textbf{Raw Agreement} & \textbf{Fleiss's $\kappa$} & \textbf{Model-Human Agreement} \\
\midrule
HallusionBench & 50 & 84.0\% & 0.290 & 45.5\% \\
POPE & 50 & 84.0\% & 0.201 & 22.2\% \\
VizWiz & 50 & 62.0\% & 0.086 & 37.5\% \\
VSR & 50 & 98.0\% & 0.221 & 37.1\% \\
\midrule
All & 200 & 82.0\% & 0.236 & 36.3\% \\
\bottomrule
\end{tabular}
}
\caption{Human-audit reporting template. Annotators label the dominant failure source as visual, language-prior, alignment, mixed, unclear, or no-failure. We report agreement (Fleiss's $\kappa$) before adjudication and model-human agreement after adjudication.}
\label{tab:app_human_audit_template}
\end{table*}

We report annotator agreement and adjudicated label counts in Table~\ref{tab:app_human_audit_agreement}. We read the audit as bounded semantic support for the Perturbation-Response view. Full-label Fleiss' $\kappa$ is $0.236$, and source versus non-source Fleiss' $\kappa$ is $0.480$. Exact all-annotator agreement is $20.5\%$, while at least two annotators agree in $82.0\%$ of examples. These values indicate fair full-label agreement and stronger agreement when the task is reduced to source versus non-source attribution. This behavior is expected under the paper's interpretation, since visual evidence, language priors, and cross-modal binding can interact in the same example.

\begin{table}[t]
\centering
\small
\begin{tabular}{lc}
\toprule
\textbf{Quantity} & \textbf{Value} \\
\midrule
Examples & 200 \\
Annotators & 3 \\
Adjudicated disagreement cases & 36 \\
Exact all-annotator agreement & 20.5\% \\
At least two annotators agree & 82.0\% \\
Full-label Fleiss' $\kappa$ & 0.236 \\
Source vs non-source Fleiss' $\kappa$ & 0.480 \\
\midrule
Alignment & 71 \\
No failure & 4 \\
Unclear & 87 \\
Language-prior & 20 \\
Visual & 16 \\
Mixed & 2 \\
\bottomrule
\end{tabular}
\caption{Human-audit agreement and adjudicated label distribution. The audit uses 200 examples, with 50 examples per dataset.}
\label{tab:app_human_audit_agreement}
\end{table}

After adjudication, $107$ examples receive a clear source label: $71$ alignment, $20$ language-prior, and $16$ visual. The remaining cases are $87$ unclear and $4$ no-failure, and $2$ mixed cases after adjudication. Among the $140$ originally incorrect examples, the adjudicated labels are $71$ alignment, $32$ unclear, $19$ language-prior, $16$ visual, and $2$ mixed. This distribution differs sharply from the raw largest-coordinate count, which is almost always $L$-dominant. It instead matches the main diagnostic pattern in the paper: image-removal confidence retention is the most prevalent measured signal, while alignment is more useful for separating failure behavior.

We compare the adjudicated human label with raw and normalized dominant sources only on clear-source cases. Mixed, unclear, and no-failure cases are reported separately rather than forced into one of the three source coordinates. We observe from Table~\ref{tab:app_human_audit_confusion} that raw dominant source has low agreement with human clear-source labels because it collapses to language-prior. Raw dominance reaches clear-source accuracy $0.165$, Macro-F1 $0.094$, and Cohen's $\kappa=0.000$. Normalized dominance improves over raw dominance, with clear-source accuracy $0.363$, Macro-F1 $0.307$, and Cohen's $\kappa=0.053$, although the agreement remains limited. This supports an important distinction. Raw $\arg\max(V,L,A)$ is useful for measuring prevalence, not for assigning a human-like source label.

\begin{table}[t]
\centering
\setlength{\tabcolsep}{3pt}
\small
\begin{tabular}{lrrr}
\toprule
\textbf{Human Label} & \textbf{Predicted $V$} & \textbf{Predicted $L$} & \textbf{Predicted $A$} \\
\midrule
Visual & 6 & 7 & 3 \\
Language-prior & 10 & 6 & 4 \\
Alignment & 18 & 26 & 27 \\
\bottomrule
\end{tabular}
\caption{Human-audit confusion-matrix template for clear-source examples.}
\label{tab:app_human_audit_confusion}
\end{table}

We also evaluate a small audit-side diagnostic classifier on the $107$ clear-source examples. We report this comparison in Table~\ref{tab:app_human_audit_classifier}. Scalar confidence reaches weighted AUROC $0.430$ for clear-source label prediction, while raw $(V,L,A)$ reaches $0.597$. Normalized $(z_V,z_L,z_A)$ reaches $0.556$, and raw $(V,L,A)$ with scalar confidence reaches $0.587$. The result is exploratory because the audit sample is small, although it is aligned with the main failure-family diagnosis result. The joint Perturbation-Response Signature carries more source-attribution signal than scalar confidence. At the same time, the audit does not validate $(V,L,A)$ as independent causal sources. It supports the restrained claim used throughout the paper: Perturbation-Response Signatures are behavioral diagnostic objects, many examples remain genuinely ambiguous, and argmax collapse loses information that is present in the full vector.

\begin{table*}[t]
\centering
\small
\begin{tabular}{lccc}
\toprule
\textbf{Feature Set} & \textbf{Macro-F1} & \textbf{Balanced Acc.} & \textbf{Weighted AUROC} \\
\midrule
Scalar confidence & 0.146 & 0.229 & 0.430 \\
Raw $V,L,A$ & \textbf{0.328} & \textbf{0.330} & \textbf{0.597} \\
Normalized $z_V,z_L,z_A$ & 0.269 & 0.310 & 0.556 \\
Raw $V,L,A$ + scalar & \textit{0.322} & 0.313 & 0.587 \\
\bottomrule
\end{tabular}
\caption{Human audit-side clear-source prediction. The classifier is evaluated on the 107 examples adjudicated as visual, language-prior, or alignment.}
\label{tab:app_human_audit_classifier}
\end{table*}

\subsection{Abstention And Source-Aware Scalarization}
\label{app:abstention_additional}

We examine whether Perturbation-Response Signatures can be collapsed into scalar confidence scores. We observe from Table~\ref{tab:app_abstention_scalarization_variants} that scalar confidence remains the strongest correctness-ranking signal under the tested max-penalty, weighted, and geometric scalarizations. The pooled AUROC is $0.783$ for scalar confidence and $0.767$ for the best source-aware scalarization. The same pattern holds across model families, with the largest gap for Qwen3-VL, where AUROC drops from $0.809$ to $0.732$. This supports the scalarization boundary: the Perturbation-Response Signature can help diagnose failure family while still weakening abstention when compressed into a penalty-adjusted confidence score.

\begin{table*}[t]
\centering
\small
\begin{tabular}{lrrrrr}
\toprule
\textbf{Model} & \textbf{Scalar} & \textbf{Penalty} & \textbf{Weighted} & \textbf{Geometric} & \textbf{Best SA} \\
\midrule
Gemma 4 & 0.809 & 0.738 & \textit{0.773} & 0.741 & 0.773 \\
LLaVA-Mistral & 0.836 & 0.718 & 0.\textit{806} & 0.719 & 0.806 \\
LLaVA-Vicuna & 0.828 & 0.665 & \textit{0.778} & 0.678 & 0.778 \\
Qwen3-VL & 0.809 & 0.721 & \textit{0.732} & 0.721 & 0.732 \\
Pooled & 0.783 & 0.638 & \textit{0.767} & 0.640 & 0.767 \\
\bottomrule
\end{tabular}
\caption{Aggregate AUROC for all scalarization variants.}
\label{tab:app_abstention_scalarization_variants}
\end{table*}

We observe from Table~\ref{tab:app_abstention_per_dataset} that this pattern is not only a pooled artifact. Source-aware scalarization is usually worse than scalar confidence across model-dataset slices. The largest drop appears for Qwen3-VL on POPE, from AUROC $0.826$ to $0.702$. HallusionBench also shows clear drops for Gemma 4 and LLaVA-Mistral, from $0.690$ to $0.605$ and from $0.609$ to $0.552$, respectively. The main exception is LLaVA-Vicuna on VSR, where AUROC improves from $0.685$ to $0.702$. These exceptions are useful boundary cases, although the dominant pattern remains negative for naive scalarization.

\begin{table*}[t]
\centering
\setlength{\tabcolsep}{3pt}
\small
\begin{tabular}{llrrrrr}
\toprule
\textbf{Model} & \textbf{Dataset} & \textbf{Scalar AUROC} & \textbf{Best SA AUROC} & \textbf{$\Delta_{SA-S}$} & \textbf{Scalar Sel@60} & \textbf{Best SA Sel@60} \\
\midrule
Gemma 4 & HallusionBench & 0.690 & 0.605 (weighted) & -0.085 & 77.5\% & 73.3\% \\
Gemma 4 & POPE & 0.779 & 0.741 (penalty) & -0.038 & 94.7\% & 94.7\% \\
Gemma 4 & VizWiz & 0.734 & 0.718 (weighted) & -0.016 & 67.4\% & 67.7\% \\
Gemma 4 & VSR & 0.712 & 0.654 (penalty) & -0.058 & 81.4\% & 81.4\% \\
LLaVA-Mistral & HallusionBench & 0.609 & 0.552 (penalty) & -0.057 & 58.2\% & 54.9\% \\
LLaVA-Mistral & POPE & 0.835 & 0.807 (weighted) & -0.028 & 97.5\% & 95.9\% \\
LLaVA-Mistral & VizWiz & 0.720 & 0.666 (weighted) & -0.054 & 67.3\% & 64.5\% \\
LLaVA-Mistral & VSR & 0.685 & 0.680 (weighted) & -0.005 & 79.9\% & 80.4\% \\
LLaVA-Vicuna & HallusionBench & 0.578 & 0.552 (penalty) & -0.026 & 50.7\% & 48.9\% \\
LLaVA-Vicuna & POPE & 0.816 & 0.810 (weighted) & -0.006 & 97.1\% & 96.0\% \\
LLaVA-Vicuna & VizWiz & 0.722 & 0.683 (weighted) & -0.039 & 66.6\% & 64.2\% \\
LLaVA-Vicuna & VSR & 0.685 & 0.702 (penalty) & +0.017 & 60.8\% & 63.2\% \\
Qwen3-VL & HallusionBench & 0.675 & 0.610 (weighted) & -0.065 & 81.9\% & 80.0\% \\
Qwen3-VL & POPE & 0.826 & 0.702 (penalty) & -0.124 & 97.3\% & 93.7\% \\
Qwen3-VL & VizWiz & 0.706 & 0.682 (weighted) & -0.024 & 71.8\% & 71.2\% \\
Qwen3-VL & VSR & 0.802 & 0.751 (penalty) & -0.051 & 95.6\% & 95.1\% \\
\bottomrule
\end{tabular}
\caption{Per-dataset abstention comparison between scalar confidence and the best source-aware scalarization.}
\label{tab:app_abstention_per_dataset}
\end{table*}

We observe from Table~\ref{tab:app_risk_coverage_pooled} that scalar confidence is especially stronger in the low-coverage regime. At $5\%$ coverage, scalar confidence has risk $0.018$, compared with $0.023$ for weighted scalarization, $0.290$ for max-penalty, and $0.225$ for geometric scalarization. At $10\%$ coverage, scalar confidence has risk $0.033$, while max-penalty and geometric scalarization remain much higher at $0.194$ and $0.171$. These results locate the proper role of $(V,L,A)$: Perturbation-Response Signatures are better suited for failure typing and routing than for naive scalar confidence replacement.

\begin{table*}[t]
\centering
\small
\begin{tabular}{lrrrrr}
\toprule
\textbf{Coverage} & \textbf{Scalar Risk} & \textbf{Penalty Risk} & \textbf{Weighted Risk} & \textbf{Geometric Risk} & \textbf{Best SA Risk} \\
\midrule
5\% & 0.018 & 0.290 & 0.023 & 0.225 & 0.023 (Weighted) \\
10\% & 0.033 & 0.194 & 0.034 & 0.171 & 0.034 (Weighted) \\
20\% & 0.050 & 0.134 & 0.070 & 0.128 & 0.070 (Weighted) \\
40\% & 0.081 & 0.094 & 0.084 & 0.094 & 0.084 (Weighted) \\
60\% & 0.099 & 0.157 & 0.112 & 0.157 & 0.112 (Weighted) \\
80\% & 0.157 & 0.224 & 0.159 & 0.224 & 0.159 (Weighted) \\
100\% & 0.234 & 0.234 & 0.234 & 0.234 & 0.234 (Penalty) \\
\bottomrule
\end{tabular}
\caption{Pooled risk-coverage values for scalar confidence and source-aware scalarizations.}
\label{tab:app_risk_coverage_pooled}
\end{table*}

\subsection{Intervention Routing}
\label{app:intervention_routing}

We report the full intervention-routing results in Table~\ref{tab:app_intervention_full}. The comparison uses three conditions: (i) no intervention, (ii) a generic reasoning intervention applied to every example, and (iii) matched routing based on the normalized dominant Perturbation-Response Signature. The prompt templates and image-enhancement parameters are given in Appendix~\ref{app:intervention_routing_policies}. Here we focus only on routing outcomes. 
We observe from Table~\ref{tab:app_intervention_full} that matched routing is conditionally useful rather than a universal accuracy booster. In the pooled result, no intervention reaches $76.6\%$ accuracy, generic intervention reaches $75.7\%$, and matched routing reaches $75.9\%$. The aggregate accuracy difference is small. The more relevant pattern is break-rate control: generic intervention has a pooled break rate of $4.1\%$, while matched routing lowers it to $3.9\%$.
The clearest case appears for Qwen3-VL on VSR. Generic intervention reduces accuracy from $87.4\%$ to $76.2\%$ and breaks $15.5\%$ of originally correct predictions. Matched routing reaches $87.6\%$ accuracy and lowers break rate to $2.7\%$. Gemma 4 on HallusionBench shows the same direction, with matched routing reducing break rate from $15.1\%$ to $4.9\%$ while recovering the no-intervention accuracy level. On LLaVA-Mistral/VizWiz, both interventions reduce accuracy relative to no intervention, although matched routing still lowers break rate from $21.4\%$ to $17.4\%$. These cases show the intended use of routing: reducing harmful overcorrection when a generic prompt is disruptive.

We observe from Table~\ref{tab:app_intervention_win_counts} that matched routing beats generic intervention in $10/16$ model-dataset settings and reduces break rate in $10/16$ settings. It beats no intervention in $8/16$ settings. Thus, \HalluPrism{} should not be read as a general mitigation method. Its decision-facing value is conditional actionability: Perturbation-Response Signatures help decide when one-size-fits-all correction is risky and when a source-matched correction is safer.

\begin{table*}[t]
\centering
\setlength{\tabcolsep}{3pt} 
\small
\begin{tabular}{llrrrrrrrr}
\toprule
\textbf{Model} & \textbf{Dataset} & \textbf{$N$} & \textbf{No Interv.} & \textbf{Generic} & \textbf{Matched} & \textbf{Fix$_G$} & \textbf{Fix$_M$} & \textbf{Break$_G$} & \textbf{Break$_M$} \\
\midrule
Gemma 4 & HallusionBench & 951 & 68.2\% & 60.9\% & 68.9\% & 9.3\% & 12.6\% & 15.1\% & 4.9\% \\
Gemma 4 & POPE & 9000 & 86.9\% & 86.8\% & 86.3\% & 11.9\% & 10.2\% & 1.9\% & 2.3\% \\
Gemma 4 & VizWiz & 4319 & 54.0\% & 51.8\% & 52.2\% & 5.9\% & 7.8\% & 9.2\% & 9.9\% \\
Gemma 4 & VSR & 340 & 71.8\% & 72.4\% & 71.5\% & 22.9\% & 8.3\% & 8.2\% & 3.7\% \\
LLaVA-Mistral & HallusionBench & 951 & 52.6\% & 50.9\% & 53.8\% & 8.2\% & 12.0\% & 10.6\% & 8.4\% \\
LLaVA-Mistral & POPE & 9000 & 89.1\% & 88.8\% & 89.2\% & 6.8\% & 8.1\% & 1.2\% & 0.8\% \\
LLaVA-Mistral & VizWiz & 4319 & 54.9\% & 49.7\% & 51.8\% & 14.5\% & 14.3\% & 21.4\% & 17.4\% \\
LLaVA-Mistral & VSR & 340 & 70.3\% & 68.8\% & 70.6\% & 5.0\% & 5.9\% & 4.2\% & 2.1\% \\
LLaVA-Vicuna & HallusionBench & 951 & 47.7\% & 47.9\% & 50.4\% & 6.4\% & 13.3\% & 6.6\% & 9.0\% \\
LLaVA-Vicuna & POPE & 9000 & 88.6\% & 89.1\% & 89.0\% & 17.7\% & 16.0\% & 1.7\% & 1.6\% \\
LLaVA-Vicuna & VizWiz & 4319 & 53.9\% & 54.4\% & 53.3\% & 11.5\% & 9.1\% & 8.9\% & 8.8\% \\
LLaVA-Vicuna & VSR & 340 & 51.5\% & 52.1\% & 52.6\% & 1.2\% & 3.0\% & 0.0\% & 0.6\% \\
Qwen3-VL & HallusionBench & 951 & 74.4\% & 71.7\% & 74.3\% & 8.6\% & 16.5\% & 6.6\% & 5.8\% \\
Qwen3-VL & POPE & 9000 & 88.9\% & 88.5\% & 87.9\% & 6.7\% & 6.7\% & 1.3\% & 1.9\% \\
Qwen3-VL & VizWiz & 4319 & 61.3\% & 60.5\% & 59.2\% & 5.2\% & 6.0\% & 4.5\% & 7.3\% \\
Qwen3-VL & VSR & 340 & 87.4\% & 76.2\% & 87.6\% & 18.6\% & 20.9\% & 15.5\% & 2.7\% \\
Pooled & All & 58440 & 76.6\% & 75.7\% & 75.9\% & 9.7\% & 10.0\% & 4.1\% & 3.9\% \\
\bottomrule
\end{tabular}
\caption{Full intervention-routing results. Generic applies one prompt to all samples. Matched chooses an intervention from the dominant source signature.}
\label{tab:app_intervention_full}
\end{table*}

\begin{table*}[t]

\centering
\small
\begin{tabular}{lrrr}
\toprule
\textbf{Model} & \textbf{Matched $>$ Generic} & \textbf{Matched $>$ No Interv.} & \textbf{Lower Break Rate} \\
\midrule
Gemma 4 & 2/4 & 1/4 & 2/4 \\
LLaVA-Mistral & 4/4 & 3/4 & 4/4 \\
LLaVA-Vicuna & 2/4 & 3/4 & 2/4 \\
Qwen3-VL & 2/4 & 1/4 & 2/4 \\
Pooled & 10/16 & 8/16 & 10/16 \\
\bottomrule
\end{tabular}
\caption{Model-level routing win counts.}
\label{tab:app_intervention_win_counts}
\end{table*}

\paragraph{Additional Routing Control.}
We additionally compare the probe-informed policy with blanket generic CoT on $6{,}561$ outputs from four model--dataset settings. Relative to generic CoT, probe-informed routing improves pooled accuracy by $3.54$ percentage points (95\% CI $[2.82,\,4.25]$) and reduces break rate by $5.12$ points (95\% CI $[4.12,\,6.09]$). The policy improves accuracy relative to generic CoT in all four settings, although it is not universally better than no intervention; VizWiz is a clear boundary case. This control therefore supports only the narrow claim that probe-informed correction can be safer than applying the same generic reasoning prompt everywhere.

\section{Human Audit Protocol}
\label{app:human_audit}

We use a small human audit as a semantic plausibility check for the Perturbation-Response interpretation. The audit does not validate hidden causal mechanisms inside a Multimodal Large Language Model (MLLM). It only asks whether a human reader, given the image, question, reference answer, model answer, and correctness flag, can assign a plausible dominant failure source that agrees with the normalized Perturbation-Response Signature. The audit is therefore aligned with the main paper's diagnosis-first claim: Perturbation-Response Signatures are behavioral diagnostic objects, not ground-truth causes.

\subsection{Audit Objective}
\label{app:human_audit_objective}

The audit checks whether the source labels used by \HalluPrism{} are interpretable at the example level. In particular, we ask whether model outputs with a visually dominant, $L$-dominant, or alignment-dominant normalized signature correspond to human-plausible explanations of the observed prediction behavior. We do not use the audit to train the diagnostic classifier, tune source scores, select routing prompts, or choose scalarization weights. All main quantitative results are computed without human audit feedback.

\subsection{Sampling Strategy}
\label{app:human_audit_sampling}

We sample examples from the same evaluated model-dataset outputs used in the main experiments. The audit sample contains $200$ examples, with $50$ examples from each benchmark: HallusionBench, POPE, VizWiz-VQA, and Visual Spatial Reasoning (VSR). Within each dataset, we stratify by normalized dominant source $s^{*}_{\mathrm{norm}}(x,M)$ whenever possible. This choice is necessary because raw dominance is strongly affected by image-removal confidence-retention saturation, as shown in Appendix~\ref{app:additional_results}. When a source stratum has fewer available examples, we include all available examples from that stratum and fill the remaining slots with examples sampled from the other normalized-source strata. We also maintain a mixture of correct and incorrect predictions, since the audit checks source plausibility rather than only error explanation.

\begin{table*}[t]
\centering
\small
\begin{tabular}{lll}
\toprule
\textbf{Dataset} & \textbf{Sample Size} & \textbf{Primary Stress Factor} \\
\midrule
HallusionBench & $50$ & Visual illusion, attribute error, and grounded hallucination \\
POPE & $50$ & Object-presence hallucination and object priors \\
VizWiz-VQA & $50$ & Visual insufficiency and answerability \\
VSR & $50$ & Spatial and relation binding \\
\midrule
Total & $200$ & Mixed multimodal failure sources \\
\bottomrule
\end{tabular}
\caption{Human-audit sampling plan. The audit uses normalized dominant source for stratification because raw dominance is heavily affected by image-removal confidence-retention saturation.}
\label{tab:human_audit_sampling_plan}
\end{table*}

\subsection{Annotator Recruitment And Payment}
\label{app:human_audit_recruitment_payment}

The human audit was conducted by student annotators recruited locally from undergraduate programs. The annotators were in the 18-22 age group and were studying Computer Science or Economics. We selected annotators with sufficient technical background to understand image-question answering tasks, model outputs, correctness flags, and the distinction between visual evidence, language-prior reliance, and alignment-related errors. No annotator was asked to provide personal, demographic, medical, behavioral, or otherwise sensitive information beyond the minimal background needed to report the audit setting.
Annotators were not compensated in cash. The compensation arrangement was communicated before participation. The audit involved semantic labeling of benchmark examples and model outputs rather than user-facing data collection or evaluation of the annotators themselves. Annotators were not shown scalar confidence, source scores $(V,L,A)$, raw or normalized dominant sources, routing decisions, or model-derived explanations. This separation reduces score-induced bias and keeps the audit focused on visible example-level evidence.

\subsection{Annotator View}
\label{app:human_audit_view}

Annotators see only information that would be available in a semantic error inspection setting. They are not shown scalar confidence, source scores $(V,L,A)$, raw dominant source, normalized dominant source, routing decision, or any model-derived explanation. This prevents the audit from simply reproducing the model-derived signature.

\begin{table*}[t]
\centering
\small
\begin{tabular}{lll}
\toprule
\textbf{Category} & \textbf{Shown To Annotators} & \textbf{Hidden From Annotators} \\
\midrule
Input & Image and question or caption & Perturbed images and probe outputs \\
Output & Model answer & Scalar confidence $c$ \\
Reference & Gold answer or answerability label & Source scores $(V,L,A)$ \\
Outcome & Correctness flag & Raw and normalized dominant source \\
Metadata & Dataset name & Routing decision and intervention result \\
\bottomrule
\end{tabular}
\caption{Information shown to annotators and information hidden during the audit.}
\label{tab:human_audit_visible_hidden}
\end{table*}

\subsection{Annotation Labels}
\label{app:human_audit_labels}

Each annotator assigns one label from six categories. The first three labels correspond to the Perturbation-Response coordinates. Two labels preserve ambiguity rather than forcing a clean source assignment, while a final label handles cases with no visible source-sensitive risk. The correctness flag is shown to annotators only as context, while the annotation target remains the plausible source of risk or source-sensitive behavior in the example.

\begin{table*}[t]
\centering
\setlength{\tabcolsep}{3.5pt} 
\scalebox{1}{
\small
\begin{tabular}{p{0.1\linewidth}p{0.4\linewidth}p{0.4\linewidth}}
\toprule
\textbf{Label} & \textbf{When To Use} & \textbf{Example Pattern} \\
\midrule
Visual & The image evidence appears weak, degraded, incomplete, too small, occluded, blurry, dark, or otherwise insufficient for the answer. & The model answers an object question, although the relevant object is too blurred or partially outside the frame. \\
\hline
Language-prior & The answer appears plausible from the question alone, common object priors, or default world knowledge, even when the image evidence is weak or absent. & The model answers ``yes'' to a common object-presence question where the image does not clearly support the object. \\
\hline
Alignment & The relevant entities are visible, while the answer appears to fail because of relation, attribute, grounding, or image-text binding. & The model detects a cone and a car, although it reverses the left-right relation between them. \\
\hline
Mixed & More than one source is plausibly active, and no single source is clearly dominant. & The image is blurry and the question also invites a common prior; either source could plausibly explain the answer. \\
\hline
Unclear & The example does not provide enough information for a reliable judgment, or the provided gold reference answer is factually incorrect. & The image is ambiguous, or the gold answer states there is a TV when the image clearly only contains a monitor. \\
\hline
No Failure & The model prediction accurately matches the correct gold answer, indicating no failure or hallucination occurred. & The model correctly answers the question based on clear visual evidence without error. \\

\bottomrule
\end{tabular}
}
\caption{Human-audit label rubric. The labels are diagnostic descriptions of the observed example, not causal claims about model internals.}
\label{tab:human_audit_rubric}
\end{table*}

\subsection{Annotation Instructions}
\label{app:human_audit_instructions}

Annotators are instructed to label the most plausible source of risk in the specific example. They should not infer a model's hidden reasoning process. They should also avoid using demographic, cultural, or identity-based assumptions not required by the task. For correct predictions, annotators assign a source label only when the prediction appears fragile or source-dependent from the visible evidence. If no source-sensitive risk is visible, they use \textit{Unclear}. This convention keeps the audit aligned with \HalluPrism{}: the Perturbation-Response Signature is defined for both correct and incorrect predictions, while the human label remains an example-level diagnostic judgment.
The instruction shown to annotators is:

\begin{quote}
\small
You are given an image, a question or caption, a reference answer, a model answer, and whether the model answer is correct. Your task is to judge the most plausible source of risk or source-sensitive behavior in this example. Use \textit{Visual} if the relevant visual evidence appears weak or insufficient. Use \textit{Language-prior} if the answer seems plausible from the question alone or from common world knowledge rather than the image. Use \textit{Alignment} if the image contains the relevant entities but the answer appears to fail due to relation, attribute, grounding, or image-text binding. Use \textit{Mixed} if more than one source is plausible. Use \textit{Unclear} if the source cannot be determined or if the gold reference answer is factually incorrect. Use \textit{No Failure} if the model answer is completely correct. Do not guess the model's internal reasoning. Label only the example-level evidence visible to you.
\end{quote}

\begin{table*}[t]
\centering
\small
\begin{tabular}{p{0.24\linewidth}p{0.68\linewidth}}
\toprule
\textbf{Metric} & \textbf{Meaning} \\
\midrule
Raw agreement & Fraction of examples with identical independent labels. \\
Fleiss's $\kappa$ & Chance-corrected agreement over the six-label set. \\
Clear-source coverage & Fraction of examples adjudicated as \textit{Visual}, \textit{Language-prior}, or \textit{Alignment}. \\
Model-human agreement & Agreement between $s^{*}_{\mathrm{norm}}(x,M)$ and mapped clear-source human labels. \\
Mixed rate & Fraction adjudicated as \textit{Mixed}. \\
Unclear rate & Fraction adjudicated as \textit{Unclear}. \\
Collapsed Fleiss's $\kappa$ & Chance-corrected agreement for source vs. non-source classification. \\
No Failure rate & Fraction adjudicated as \textit{No Failure}. \\

\bottomrule
\end{tabular}
\caption{Human-audit reporting fields. We report mixed and unclear cases separately rather than forcing them into one of the three source coordinates.}
\label{tab:human_audit_reporting_fields}
\end{table*}

\subsection{Agreement and Adjudication}
\label{app:human_audit_agreement}
Three annotators independently label each example. We compute raw agreement and Fleiss's $\kappa$ over the six-label set: \textit{Visual}, \textit{Language-prior}, \textit{Alignment}, \textit{Mixed}, \textit{Unclear}, and \textit{No Failure}. In addition to the full 6-label agreement, we compute a collapsed Fleiss's $\kappa$ that groups the three source labels (Visual, Language-prior, Alignment) against the non-source labels (Mixed, Unclear, No Failure). This isolates whether humans can reliably distinguish source-driven failures from generic ambiguity. Disagreements are adjudicated after the independent pass. During adjudication, annotators may revise a label only by referring to the visible evidence in the image, question, reference answer, model answer, and correctness flag. They do not see source scores during adjudication.
After adjudication, we compute model-human agreement between the adjudicated human label and the normalized dominant source $s^{*}_{\mathrm{norm}}(x,M)$. This comparison is restricted to clear-source cases labeled \textit{Visual}, \textit{Language-prior}, or \textit{Alignment}. Mixed and unclear cases are reported separately. 

\subsection{Reporting Metrics}
\label{app:human_audit_metrics}

We report eight audit quantities: (i) \textit{raw agreement}, the fraction of examples on which annotators assign the same label before adjudication; (ii) Fleiss's $\kappa$, the chance-corrected agreement over the six-label set; (iii) \textit{collapsed Fleiss's $\kappa$}, the chance-corrected agreement for source vs. non-source classification; (iv) \textit{clear-source coverage}, the fraction of adjudicated examples assigned to \textit{Visual}, \textit{Language-prior}, or \textit{Alignment}; (v) \textit{model-human agreement}, the agreement between the normalized dominant source and the adjudicated human label on clear-source examples; (vi) \textit{mixed rate}, the fraction of examples adjudicated as \textit{Mixed}; (vii) \textit{unclear rate}, the fraction adjudicated as \textit{Unclear}; and (viii) \textit{no failure rate}, the fraction adjudicated as \textit{No Failure}.
Let $\widehat{s}(x)=s^{*}_{\mathrm{norm}}(x,M)$. Let $v_h$, $l_h$, $a_h$, $m_h$, $u_h$, and $n_h$ denote the adjudicated labels \textit{Visual}, \textit{Language-prior}, \textit{Alignment}, \textit{Mixed}, \textit{Unclear}, and \textit{No Failure}, respectively. We define the clear-source label set, clear-source examples, label-coordinate map, and agreement indicator as follows:

\begin{equation}
\begin{split}
&\mathcal{C}_{\mathrm{src}}=\{v_h,l_h,a_h\}, \ \mathcal{X}_{c}
=
\{x\in\mathcal{X}:h(x)\in\mathcal{C}_{\mathrm{src}}\}, \\
&\psi(v_h)=\mathrm{V},\quad
\psi(l_h)=\mathrm{L},\quad
\psi(a_h)=\mathrm{A}\\
&\delta_x
=
\mathbf{1}\!\left[\widehat{s}(x)=\psi(h(x))\right].
\end{split}
\label{eq:human_audit_aux}
\end{equation}

The reported rates are then:
\begin{equation}
\begin{split}
&\mathrm{ClearCov}
=
\frac{|\mathcal{X}_{c}|}{|\mathcal{X}|}, \\
&\mathrm{MHAgree}
=
\frac{1}{|\mathcal{X}_{c}|}
\sum_{x\in\mathcal{X}_{c}}\delta_x, \\
&\mathrm{MixedRate}
=
\frac{1}{|\mathcal{X}|}
\sum_{x\in\mathcal{X}}
\mathbf{1}\!\left[h(x)=m_h\right], \\
&\mathrm{UnclearRate}
=
\frac{1}{|\mathcal{X}|}
\sum_{x\in\mathcal{X}}
\mathbf{1}\!\left[h(x)=u_h\right], \\
&\mathrm{NoFailureRate}
=
\frac{1}{|\mathcal{X}|}
\sum_{x\in\mathcal{X}}
\mathbf{1}\!\left[h(x)=n_h\right].
\end{split}
\label{eq:human_audit_metrics}
\end{equation}

Here $h(x)$ denotes the adjudicated human label. The map $\psi$ links each clear-source human label to the corresponding Perturbation-Response coordinate. We report the confusion matrix over clear-source examples and separately report the number of mixed and unclear cases.

\subsection{Interpretation}
\label{app:human_audit_interpretation}

The audit is intentionally conservative. High model-human agreement would support the semantic plausibility of the Perturbation-Response labels. Low agreement would not invalidate the main quantitative results, because the main claims concern behavioral probe responses, failure-family prediction, scalarization boundary checks, and routing breakage. Instead, low agreement would indicate that the normalized dominant source is not always human-interpretable as a single source. This outcome would still be consistent with the paper's entanglement claim. Therefore, we treat the audit as an interpretability check on the source-label interface, not as evidence for causal source recovery.

\section{Computational Cost And Resource Use}
\label{app:compute}

We report the computational cost of \HalluPrism{} because the method is a diagnosis layer rather than a single-pass confidence score. Each evaluated prediction requires the clean answer, four visual perturbation answers, one blank-image answer, and one or two alignment-probe answers. Thus, the core Perturbation-Response computation uses $7$ forward passes for datasets without relation-swap probing and $8$ forward passes for Visual Spatial Reasoning (VSR), where explicit relation perturbation is available. This cost is the main practical tradeoff of the method: \HalluPrism{} gives source-aware diagnostic information, but it is not intended as a zero-overhead scalar confidence replacement.

\subsection{Inference Budget}
\label{app:inference_budget}

The evaluation contains $58{,}440$ model-output instances across four datasets and four model families. The Perturbation-Response computation uses the clean prediction, four visual-perturbation predictions, one blank-image prediction, and one grounding re-check prediction for every image-paired example. Therefore, HallusionBench, POPE, and VizWiz-VQA use $7$ forward passes per model-example. Visual Spatial Reasoning (VSR) additionally uses the relation-swap alignment probe, yielding $8$ forward passes per model-example.

Let $N_d$ denote the number of examples in dataset $d$, and let $p_d$ denote the number of forward passes used for Perturbation-Response measurement in that dataset. We compute the core Perturbation-Response budget as:
\begin{equation}
\begin{split}
&N_{\mathrm{core}}
=
\sum_{m\in\mathcal{M}}
\sum_{d\in\mathcal{D}}
N_d p_d,
\\
&p_d
=
\begin{cases}
7, & \text{without relation-swap probing},\\
8, & \text{with relation-swap probing}.
\end{cases}
\end{split}
\label{eq:core_forward_budget}
\end{equation}

In our evaluation, relation-swap probing is available only for VSR. Hence, the core Perturbation-Response budget is:
\begin{equation}
\begin{split}
N_{\mathrm{core}}
&=
4\bigl(951\cdot 7
+
9000\cdot 7
+
4319\cdot 7
+
340\cdot 8\bigr)
\\
&=
410{,}440.
\end{split}
\label{eq:core_forward_budget_instantiated}
\end{equation}

Intervention routing adds two further branches for each model-output instance: one generic intervention and one matched intervention. Since the pooled evaluation contains $58{,}440$ model-output instances, routing adds:
\begin{equation}
N_{\mathrm{route}}
=
2\cdot 58{,}440
=
116{,}880.
\label{eq:routing_forward_budget}
\end{equation}

The full diagnosis-plus-routing experiment therefore uses:
\begin{equation}
N_{\mathrm{total}}
=
N_{\mathrm{core}}+N_{\mathrm{route}}
=
527{,}320
\label{eq:total_forward_budget}
\end{equation}
forward passes, excluding optional self-reported confidence baselines and implementation-level retries caused by model-serving failures.

\begin{table*}[t]
\centering
\small
\begin{tabular}{lrrrr}
\toprule
\textbf{Dataset / Branch} & \textbf{Examples Per Model} & \textbf{Models} & \textbf{Passes Per Example} & \textbf{Forward Passes} \\
\midrule
HallusionBench & 951 & 4 & 7 & 26,628 \\
POPE & 9,000 & 4 & 7 & 252,000 \\
VizWiz-VQA & 4,319 & 4 & 7 & 120,932 \\
VSR & 340 & 4 & 8 & 10,880 \\
\midrule
Core Perturbation-Response measurement & 14,610 & 4 & -- & 410,440 \\
Generic and matched routing branches & 14,610 & 4 & 2 & 116,880 \\
\midrule
Full diagnosis-plus-routing experiment & 14,610 & 4 & -- & 527,320 \\
\bottomrule
\end{tabular}
\caption{Corrected inference budget for HalluPrism. HallusionBench, POPE, and VizWiz-VQA use $7$ passes per model-example. VSR uses $8$ passes because relation-swap probing is available. Routing adds two additional branches per model-output instance.}
\label{tab:inference_budget}
\end{table*}

\subsection{Training Cost}
\label{app:compute_training_cost}
We do not fine-tune any Multimodal Large Language Model (MLLM). All MLLMs are evaluated with deterministic decoding. The only learned models in the paper are lightweight diagnostic classifiers trained on cached scalar confidence and Perturbation-Response features. These classifiers include logistic regression, XGBoost, and LightGBM. Their inputs are four-dimensional at most, namely $[V,L,A,c]$, and their cost is negligible relative to MLLM inference. The diagnostic classifiers are used only for analysis. They are not used to update the evaluated MLLMs.

\subsection{Decoding And Runtime Controls}
\label{app:compute_runtime_controls}

We use greedy decoding with temperature $0$, disabled sampling, and a maximum response length of $128$ tokens. These choices reduce stochastic variation in the perturbation-response measurements. They also make the compute budget easier to interpret, since repeated generations under the same input are not used as uncertainty samples. This differs from sampling-based uncertainty methods, where multiple generations are required for each clean input before any perturbation is applied.

The Perturbation-Response cost scales linearly in the number of probes. If a deployment setting needs only a subset of the diagnostic sources, the cost can be reduced by disabling unused probes. For example, a system concerned only with image-removal confidence retention can run the clean and blank-image branches. A system concerned with spatial binding can retain the relation-swap branch. In the full paper, we keep all probes active because the central empirical question is whether the joint $(V,L,A)$ signature diagnoses failure families and guides correction.

\subsection{Hardware And Wall-Clock Reporting}
\label{app:compute_hardware}

We report the hardware and wall-clock runtime in Table~\ref{tab:hardware_reporting}. Exact runtime depends on the serving backend, batch size, model implementation, image resolution, and GPU type. The main reproducibility-critical quantity is therefore the number of forward passes, decoding configuration, model checkpoint, and prompt regime. These are fixed by the protocol and reported in Appendix~\ref{app:prompt_catalog} and Table~\ref{tab:inference_budget}.

\begin{table*}

\centering
\small
\begin{tabular}{ll}
\toprule
\textbf{Field} & \textbf{Value} \\
\midrule
GPU type &  NVIDIA RTX A6000 (48\,GB) \\
Number of GPUs & 3 \\
CPU and memory & AMD EPYC 7543 32-Core Processor, 324GB \\
Inference framework & HuggingFace Transformers ($\geq$4.40.0) \\
Batch size & 1 (sequential per-sample inference) \\
Elapsed calendar time & $\approx$15 days \\
Total sequential GPU compute hours & 369.5 \\
\bottomrule
\end{tabular}
\caption{Hardware and runtime used for the submitted experiments. Runtime depends on serving backend, batching, model implementation, and image preprocessing, so forward-pass counts in Table~\ref{tab:inference_budget} remain the primary reproducibility quantity.}
\label{tab:hardware_reporting}
\end{table*}

\subsection{Resource Use And Environmental Reporting}
\label{app:compute_environmental}

The experiments are inference-only for the evaluated MLLMs and do not require model fine-tuning. This substantially reduces the resource cost compared with training or adapting large vision-language models. However, \HalluPrism{} still increases inference cost because each prediction is evaluated under multiple perturbation branches. We therefore report forward-pass counts explicitly rather than only reporting the number of original examples.
If the final execution environment logs GPU-hours or energy use, we will report those values with the hardware description in Table~\ref{tab:hardware_reporting}. If such logs are unavailable, we report the fixed forward-pass budget and deterministic decoding settings as the reproducible compute description. This is less informative than a full energy audit, but it avoids inventing carbon or power estimates that were not directly measured.

\subsection{Practical Deployment Tradeoff}
\label{app:compute_practical_tradeoff}

The compute cost should be interpreted together with the intended role of the method. \HalluPrism{} is not designed to run as a cheap scalar abstention score for every low-risk query. It is more suitable as an audit-time diagnostic, a high-risk inference-time check, or a routing layer for cases where generic correction may damage correct answers. In latency-critical deployments, the full protocol can be reserved for uncertain, high-impact, or audit-selected examples. This preserves the main contribution: source-aware uncertainty should expose the kind of risk present before a system decides whether to answer or correct.

\begin{table*}[t]
\centering
\setlength{\tabcolsep}{3pt}
\small
\begin{tabular}{p{0.14\linewidth}p{0.14\linewidth}p{0.25\linewidth}p{0.4\linewidth}}
\toprule
\textbf{Artifact} & \textbf{Type} & \textbf{Use In Paper} & \textbf{License / Access Notes} \\
\midrule
HallusionBench
&
Benchmark dataset
&
Visual illusion, attribute error, grounded hallucination, source-profile diagnosis, and routing analysis
&
The official \href{https://github.com/tianyi-lab/HallusionBench}{HallusionBench repository} lists the artifact under the \href{https://github.com/tianyi-lab/HallusionBench?tab=BSD-3-Clause-1-ov-file\#readme}{BSD-3-Clause license}. We use the image-paired evaluation subset and do not redistribute raw images.
\\
\midrule
POPE
&
Benchmark dataset and evaluation code
&
Object-presence hallucination, blank-image pressure, abstention boundary check, and object-hallucination diagnosis
&
The official \href{https://github.com/RUCAIBox/POPE}{POPE repository} provides the benchmark code and data. Its \href{https://github.com/RUCAIBox/POPE/blob/main/LICENSE}{license file} lists the artifact under the MIT License. We use released prompts and annotations for evaluation.
\\
\midrule
VSR
&
Benchmark dataset
&
Spatial relation verification and relation-swap alignment probing
&
The official \href{https://github.com/cambridgeltl/visual-spatial-reasoning}{Visual Spatial Reasoning repository} lists the project under the \href{https://github.com/cambridgeltl/visual-spatial-reasoning?tab=Apache-2.0-1-ov-file\#readme}{Apache-2.0 license}. We use zero-shot development examples with explicit spatial relations.
\\
\midrule
VizWiz-VQA
&
Benchmark dataset
&
Visual insufficiency, answerability diagnosis, and visual-fragility analysis
&
The \href{https://vizwiz.org/tasks-and-datasets/vqa/}{VizWiz-VQA dataset page} states that the work is licensed under a \href{https://creativecommons.org/licenses/by/4.0/}{Creative Commons Attribution 4.0 International License}. Since the images can contain user-provided visual contexts, we do not redistribute raw images.
\\
\bottomrule
\end{tabular}
\caption{Dataset artifacts used in \HalluPrism{}. We use public benchmark datasets and follow the corresponding license and access terms.}
\label{tab:dataset_artifact_statement}
\end{table*}

\begin{table*}[t]
\centering
\small
\begin{tabular}{p{0.10\linewidth}p{0.16\linewidth}p{0.25\linewidth}p{0.36\linewidth}}
\toprule
\textbf{Artifact} & \textbf{Type} & \textbf{Use In Paper} & \textbf{License / Access Notes} \\
\midrule
Gemma 4
&
Open-weight MLLM family
&
One evaluated model family for Perturbation-Response measurement and routing analysis
&
Google AI documentation describes Gemma 4 as open-weight with responsible commercial use.
Use remains subject to the applicable \href{https://ai.google.dev/gemma/terms}{Gemma terms}, \href{https://ai.google.dev/gemma/prohibited_use_policy}{Gemma Prohibited Use Policy}, and the \href{https://ai.google.dev/gemma/docs/core}{Gemma 4 model documentation}.
\\
\midrule
Qwen3-VL
&
Open-weight MLLM
&
One evaluated model family for Perturbation-Response measurement, abstention scalarization, and routing analysis
&
The \href{https://huggingface.co/Qwen/Qwen3-VL-8B-Instruct}{Qwen3-VL-8B-Instruct model card} is used as the checkpoint reference.
The model is listed with an \href{https://www.apache.org/licenses/LICENSE-2.0}{Apache-2.0 license}; the Azure/Hugging Face catalog also lists Qwen3-VL-8B-Instruct as Apache-2.0.
\\
\midrule
LLaVA-Mistral
&
Open-weight MLLM
&
One evaluated model family for Perturbation-Response measurement and diagnostic controls
&
The \href{https://huggingface.co/llava-hf/llava-v1.6-mistral-7b-hf}{llava-v1.6-mistral-7b-hf model card} is used as the checkpoint reference.
The model card metadata lists \href{https://www.apache.org/licenses/LICENSE-2.0}{Apache-2.0} as the license.
\\
\midrule
LLaVA-Vicuna
&
Open-weight MLLM
&
One evaluated model family for Perturbation-Response measurement and diagnostic controls
&
The \href{https://huggingface.co/llava-hf/llava-v1.6-vicuna-7b-hf}{llava-v1.6-vicuna-7b model card} is used as the checkpoint reference.
The model card states that Llama 2 is licensed under the \href{https://ai.meta.com/llama/license/}{Llama 2 Community License}.
\\
\bottomrule
\end{tabular}
\caption{Model artifacts used in \HalluPrism{}. We use the listed models for deterministic inference only and do not redistribute model weights.}
\label{tab:model_artifact_statement}
\end{table*}

\section{Artifact, License, and Data Statement}
\label{app:artifact_license_data_statement}

We report the artifacts used in \HalluPrism{} to make dataset use, model access, and release boundaries explicit. All experiments use public benchmark datasets and publicly accessible model checkpoints or model families. We do not collect new user images, and we do not redistribute benchmark images or model weights. When redistribution is restricted or unclear, we release only derived metadata, evaluation scripts, prompt templates, normalized predictions, source scores, and aggregate result tables, subject to the original artifact licenses.

\subsection{Dataset Artifacts}
\label{app:dataset_artifacts}
We use four public multimodal benchmarks under their stated access and license terms. We report the dataset-level license and access summary in Table~\ref{tab:dataset_artifact_statement}. For HallusionBench, we use only the image-paired evaluation subset and do not redistribute raw images. For POPE and VSR, we use the released benchmark prompts, annotations, and examples subject to the corresponding repository terms. For VizWiz-VQA, we treat the raw images with additional care because they may contain user-provided visual contexts. Across all datasets, any released artifact from this paper should contain only derived metadata, prompt templates, normalized outputs, source scores, and aggregate results unless the original dataset license explicitly permits redistribution of raw images.

\subsection{Model Artifacts}
\label{app:model_artifacts}
We evaluate open-weight and publicly accessible Multimodal Large Language Models (MLLMs). We use these models only for inference. We do not fine-tune, redistribute, or modify their weights. All generated outputs are used for research evaluation of Perturbation-Response Signatures, abstention scalarization, and intervention routing. A summary of the model artifacts is given in Table \ref{tab:model_artifact_statement}.

\subsection{Derived Artifacts}
\label{app:derived_artifacts}

We create derived artifacts from model outputs and perturbation probes. These include clean predictions, normalized answers, scalar confidence values, source scores $(V,L,A)$, dominant-source labels, failure-family labels, abstention scalarization scores, and intervention-routing outcomes. These artifacts are produced from public benchmark examples and model inference. They do not contain new human-subject data. If released, they should be shared as derived metadata keyed by dataset identifiers rather than as redistributed benchmark images, unless the original dataset license explicitly permits redistribution.

\begin{table*}[t]
\centering
\small
\begin{tabular}{p{0.22\linewidth}p{0.38\linewidth}p{0.28\linewidth}}
\toprule
\textbf{Derived Artifact} & \textbf{Content} & \textbf{Release Boundary} \\
\midrule
Prediction files & Clean model answer, normalized answer, scalar confidence, correctness flag, and failure-family label & Can be released with dataset identifiers, subject to dataset terms. \\
Perturbation-Response files & $V$, $L$, $A$, raw dominant source, normalized dominant source, and probe metadata & Can be released as derived metadata. \\
Perturbation metadata & Names and parameters for blur, crop, brightness, noise, blank-image replacement, grounding, and relation-swap probes & Can be released as code and metadata. \\
Routing files & Generic intervention output, matched intervention output, fix indicator, break indicator, and post-intervention accuracy & Can be released as derived metadata if original dataset terms permit. \\
Prompt catalog & Base prompts, diagnostic probes, entanglement probes, generic intervention, and matched intervention templates & Can be released in full. \\
\bottomrule
\end{tabular}
\caption{Derived artifacts produced by the \HalluPrism{} pipeline. These artifacts support reproducibility without requiring redistribution of raw benchmark images or model weights.}
\label{tab:derived_artifact_statement}
\end{table*}



\subsection{License Compliance}
\label{app:license_compliance}

We use each artifact under its stated research or open-weight access terms. The paper cites the original dataset and model creators in the main bibliography. The released code will include a license file for our own implementation and will document external dependencies separately from dataset and model licenses. Any user of the released artifacts remains responsible for complying with the original licenses and terms for HallusionBench, POPE, VSR, VizWiz-VQA, Gemma 4, Qwen3-VL, LLaVA-Mistral, and LLaVA-Vicuna.

\subsection{Data Sensitivity}
\label{app:data_sensitivity_statement}

VizWiz-VQA contains images and questions originating from blind or low-vision users. Although the dataset is public and licensed for research use, the images may include private environments, personal objects, or accessibility-sensitive contexts. We therefore avoid redistributing raw images and treat VizWiz-derived outputs as evaluation metadata. More generally, multimodal datasets can contain visual details that are more sensitive than the corresponding text prompts. Downstream users should inspect dataset licenses, privacy notes, access terms, and intended-use statements before reusing any raw images or derived artifacts.

\subsection{Statement Of Use}
\label{app:statement_of_use}

We use these artifacts only to evaluate the diagnosis-first claim of \HalluPrism{}. The artifacts are not used to train a new deployed MLLM, infer private attributes, identify people, or make decisions about real users. The human audit protocol, if completed, uses benchmark examples and does not collect new sensitive personal information. Any public release of audit labels should exclude annotator-identifying information and should preserve only example identifiers, adjudicated labels, and aggregate agreement statistics.

\section{Use Of AI Writing Assistance}
\label{app:ai_writing_assistance}

AI writing assistance was used during manuscript preparation for language editing and polishing of LATEX tables. The authors remained responsible for all scientific content, including the problem formulation, experimental design, implementation, data processing, result interpretation, citations, mathematical statements, and final manuscript text. The authors manually verified the reported claims, numerical results, tables, citations, proofs, and responsible-use statements before submission. AI assistance was not used as a substitute for author judgment, was not treated as authorship, and did not determine the paper's scientific conclusions.

\begin{table*}[t]
\centering
\small
\begin{tabular}{p{0.22\linewidth}p{0.34\linewidth}p{0.34\linewidth}}
\toprule
\textbf{Use Case} & \textbf{Appropriate Use} & \textbf{Boundary Or Risk} \\
\midrule
Model auditing & Compare scalar confidence with Perturbation-Response Signatures and failure-family labels. & Does not recover hidden causal mechanisms. \\
Selective answering & Use scalar scores only under explicit validation-to-deployment assumptions. & A threshold is a deployment guardrail, not an instance-grounded guarantee. \\
Correction routing & Use $(V,L,A)$ to choose visual enhancement, anti-prior prompting, or relation checking. & Routing is conditional and may still break correct answers. \\
High-stakes deployment & Use only with domain-specific validation, monitoring, and human oversight. & Not sufficient for medical, legal, financial, or accessibility-critical certification. \\
Human audit & Check whether normalized dominant sources are semantically plausible. & Human labels do not prove causal source separability. \\
\bottomrule
\end{tabular}
\caption{Responsible-use boundaries for \HalluPrism{}. The method is intended for diagnosis and matched correction analysis, not for standalone safety certification.}
\label{tab:responsible_use_boundaries}
\end{table*}

\section{Ethical Considerations and Responsible Use}
\label{app:ethics_responsible_use}

We study \HalluPrism{} as a behavioral diagnostic layer for Multimodal Large Language Models (MLLMs). The intended use is model analysis, audit support, and safer correction-path selection. It should not be used as a standalone safety certificate, a causal explanation of model internals, or an automatic deployment approval mechanism. This distinction is central to the paper. \HalluPrism{} measures how an output changes under controlled visual, image-removal, and grounding/relation pressure. It does not prove why the model produced the answer, and it does not guarantee that an answer is safe for a new open-world instance.

\subsection{Intended Use}
\label{app:ethics_intended_use}
\HalluPrism{} is intended for model auditing, benchmark-level error analysis, and guarded correction-path selection. Appropriate uses include: (i) comparing model behavior across multimodal hallucination benchmarks, (ii) identifying broad failure-family patterns that are not visible from scalar confidence alone, and (iii) selecting a correction path only when validation shows that such routing reduces harmful breakage. Any deployment use should specify the task, user population, validation distribution, acceptable risk level, fallback policy, and monitoring plan before Perturbation-Response Signatures are used to guide correction.

\subsection{Non-Intended Use}
\label{app:ethics_non_intended_use}
\HalluPrism{} should not be used as a standalone safety certificate, a deployment approval mechanism, or a replacement for human judgment in high-stakes settings. It should not certify answers in medical, legal, financial, educational, accessibility-critical, or public-sector contexts without domain-specific validation, expert review, and post-deployment monitoring. The source scores should also not be interpreted as ground-truth causes of model behavior. They are bounded diagnostic signals produced under a fixed evaluation protocol. Treating them as causal explanations or complete reliability guarantees would overstate what the method establishes.

\subsection{Risks From Diagnostic Misuse}
\label{app:ethics_diagnostic_misuse}
The main misuse risk is over-trust. A system that treats the Perturbation-Response Signature as a complete reliability certificate recreates the same failure mode that motivates the paper's caution about scalar abstention. The signature is bounded by the evaluated datasets, models, prompts, perturbations, and routing choices. It can miss risks outside these conditions, including domain expertise gaps, cultural context, ambiguous user intent, missing background information, and harms that are not expressed through the measured source coordinates. A second risk is correction harm. A correction prompt can improve one model-dataset condition and damage another. Therefore, routing should be treated as conditional actionability rather than universal mitigation. Systems using \HalluPrism{} should log the original answer, routed action, corrected answer, fix indicator, break indicator, and any human-review outcome, since accuracy alone can hide harmful overcorrection.

\subsection{Risks For Affected Users}
\label{app:ethics_affected_users}
Several datasets used in this paper involve user-facing visual question answering. VizWiz-VQA is especially relevant because it contains images and questions from blind or low-vision users. Errors in such settings may create practical harm because users may not be able to independently verify a visual answer. A model that confidently guesses from a language prior may give a fluent answer that is not supported by the image. A generic correction prompt may also change an originally correct answer. We therefore view \HalluPrism{} as an audit aid for identifying risk structure, not as a substitute for accessibility-aware system design, user consent, domain-specific evaluation, or human fallback mechanisms.

\subsection{Data And Privacy Considerations}
\label{app:ethics_data_privacy}
We use public benchmark datasets and do not collect new user images for the main experiments. The human audit protocol, if completed, uses examples sampled from the evaluated benchmark outputs and does not require collecting personal information from annotators beyond standard annotation logistics. Annotators do not see source scores or routing decisions, which limits score-induced bias during the audit. Public image datasets can still contain sensitive visual content, private environments, or user-specific contexts. Any released artifact should therefore include dataset source information, license terms, access constraints, and a clear description of which fields are original, derived, or model-generated.

\subsection{Responsible Deployment Guidance}
\label{app:ethics_deployment_guidance}
Before using \HalluPrism{} in deployment, a system owner should document: (i) the target task and user population; (ii) the validation distribution and expected deployment shifts; (iii) the acceptable risk level and fallback policy; (iv) the conditions under which Perturbation-Response Signatures trigger correction or human review; (v) the logging fields for the original answer, Perturbation-Response Signature, routed action, corrected answer, fix indicator, and break indicator; and (vi) the monitoring plan for post-deployment drift. In high-impact or latency-critical settings, the full protocol should be reserved for audit-selected, uncertain, or high-risk examples. Perturbation-Response Signatures should guide review or guarded correction rather than automatically certifying an answer as safe.

\end{document}